\documentclass[a4paper,reqno, 10pt]{amsart}

\usepackage{amsfonts,amsmath, amssymb, amsthm}
\usepackage[foot]{amsaddr}
\usepackage{graphicx}
\usepackage{subcaption}
\usepackage{color}
\usepackage[hyperindex,breaklinks]{hyperref}
\usepackage{xspace}
\usepackage{framed}
\usepackage[ruled,noline]{algorithm2e}
\usepackage[utf8]{inputenc}
\usepackage{multirow}
\usepackage[a4paper]{geometry}
\usepackage[mathscr]{euscript}
\usepackage{tablefootnote}
\usepackage{setspace}
\usepackage{enumitem}  
\newtheorem{theorem}{Theorem}[section]
\newtheorem{assumption}{Assumption}[section]
\newtheorem{definition}{Definition}[section]
\newtheorem{example}{Example}[section]
\newtheorem{lemma}{Lemma}[section]
\newtheorem{proposition}{Proposition}[section]
\newtheorem{corollary}{Corollary}[section]
\newtheorem{remark}{Remark}[section]

\usepackage{subcaption}
\usepackage{graphicx}
\usepackage{color}
\usepackage[hyperindex,breaklinks]{hyperref}
\usepackage{xspace}
\usepackage{framed}
\usepackage[ruled,noline]{algorithm2e}
\usepackage[utf8]{inputenc}
\usepackage{multirow}
\usepackage[mathscr]{euscript}
\usepackage{tablefootnote}
\usepackage{setspace}
\usepackage{enumitem}  
\usepackage{endnotes}
\let\footnote=\endnote

\hypersetup{colorlinks = true, urlcolor = green, linkcolor = red,
  citecolor = blue }

\usepackage{natbib}
 \bibpunct[, ]{(}{)}{,}{a}{}{,}%
\usepackage{natbib}
 \bibpunct[, ]{(}{)}{,}{a}{}{,}%
\usepackage{amsfonts,amsmath, amssymb, amsthm}
\usepackage{graphicx}
\usepackage{subcaption}
\usepackage{color}
\usepackage[hyperindex,breaklinks]{hyperref}
\usepackage{xspace}
\usepackage{framed}
\usepackage[ruled,noline]{algorithm2e}
\usepackage[utf8]{inputenc}
\usepackage{multirow}
 \usepackage[a4paper]{geometry}
\usepackage[mathscr]{euscript}
\usepackage{tablefootnote}
\usepackage{enumitem}  
\usepackage{natbib}
 \bibpunct[, ]{(}{)}{,}{a}{}{,}%
\newcommand{\xmath}[1]{\ensuremath{#1}\xspace}
\newcommand{\Lev}{\xmath{\textnormal{lev}}}
\newcommand{\RV}{\xmath{\mathscr{RV}}}
\newcommand{\MRV}{\xmath{\mathscr{MRV}}}
\newcommand{\R}{\xmath{\mathbb{R}}}
\newcommand{\mv}[1]{\boldsymbol{#1}}

\newcommand{\qq}{\boldsymbol{q}}

\newcommand{\xx}{\boldsymbol{x}}
\newcommand{\XX}{\boldsymbol{X}}
\newcommand{\YY}{\boldsymbol{Y}}
\newcommand{\ZZ}{\boldsymbol{Z}}
\newcommand{\secmom}{\textsc{sec-mom}}
\newcommand{\yy}{\boldsymbol{y}}
\newcommand{\ttheta}{\boldsymbol{\theta}}
\newcommand{\zz}{\boldsymbol{z}}
\newcommand{\aalpha}{\boldsymbol{1/\alpha}}

\newcommand{\pp}{\boldsymbol{p}}
\newcommand{\e}{\mathrm{e}}
\newcommand{\Fbar}{\bar{F}}
\renewcommand{\det}[1]{\textnormal{det}({#1})} 
\newcommand{\Real}{\mathbb{R}}
\newcommand{\Prob}{P}
\newcommand{\Expc}{\mathbb{E}}

\newcommand{\fLD}{f_{\textnormal{LD}}}
\newcommand{\supp}{\textnormal{supp}}
\newcommand{\Lcr}{\xmath{L_{\textsc{cr}}}}

\newcommand{\Icr}{\xmath{I_{\textsc{cr}}}}
\newcommand{\YYu}{\xmath{\mv{Y}_{\hspace{-3pt}{u}}}}
\newcommand{\fldh}{\bar{f}_{\textnormal{LD}}}

\newcommand{\TTu}{\xmath{\mv{T}_{\hspace{-2pt}u}}}

\hypersetup{colorlinks = true, urlcolor = blue, linkcolor = blue,
  citecolor = blue }
\title[]{
   \hspace{5pt} A\MakeLowercase{chieving} E\MakeLowercase{fficiency} \MakeLowercase{in} B\MakeLowercase{lack-box} S\MakeLowercase{imulation} \MakeLowercase{of}
  D\MakeLowercase{istribution} T\MakeLowercase{ails} \MakeLowercase{with} S\MakeLowercase{elf}-s\MakeLowercase{tructuring} I\MakeLowercase{mportance}
  S\MakeLowercase{amplers}}
\author{{\large A\MakeLowercase{nand}
    D\MakeLowercase{eo}}}  \author{{\large
    K\MakeLowercase{arthyek} M\MakeLowercase{urthy}}}
  \address{Indian Institute of Management, Bannerghatta Road, Bangalore 560076} \email{anand.deo@iimb.ac.in}
  \address{Singapore University of Technology and Design, Singapore
  487372} \email{karthyek\_murthy@sutd.edu.sg}   
\begin{document}
\begin{abstract}
  This paper presents a novel Importance Sampling (IS)  scheme for estimating  distribution tails of performance measures modeled with a rich set of tools such as  linear programs,  integer linear programs, piecewise linear/quadratic objectives, feature maps specified with deep neural networks, etc. 
  {The conventional approach of explicitly identifying efficient changes of measure} suffers from feasibility and scalability concerns beyond highly stylized models, due to their need to be tailored intricately to the objective and the underlying probability distribution. {This bottleneck is overcome in the proposed scheme with an elementary transformation which is capable of implicitly inducing an effective IS distribution in a variety of models}  by  replicating the concentration properties observed in less rare samples.  This novel approach is guided by developing a  large deviations principle that brings out the phenomenon of
  self-similarity of optimal IS distributions.
  The proposed sampler is the first to attain asymptotically optimal
  variance reduction across a spectrum of multivariate distributions
  despite being oblivious to the {specifics of the underlying model.} Its applicability is illustrated  with    contextual shortest path and portfolio credit risk  models informed by neural networks.

\end{abstract}
\maketitle

\section{Introduction}
\label{sec:Introduction}
In addition to being an integral part of quantitative risk management,
the need to estimate and control tail risks is inherent in managing
operations requiring high levels of service or reliability
guarantees. The variety of contexts for which chance-constrained and
risk-averse optimization formulations are employed serve as a
testimony to the importance of tail risk management in operations
research. Naturally, this significance is retained in the numerous
operations and quantitative risk management models which are being enriched with
the use of algorithmic feature-mapping tools (such as neural networks,
kernels, etc.,) employed to facilitate a greater degree of automation
and expressivity in mapping data to
decisions. 
Recent literature on modeling mortgage risk with deep neural networks
\citep[eg.,][]{sirignano2018deep} and models incorporating contextual side information
into decision making \citep[eg.,][]{VahnRudin, elmachtoub2020smart} serve as
illustrative examples.
With the increasing adoption of these expressive models,  it is imperative that the risk management practice seeks to measure
and manage the tail risks associated with their use.  In a similar vein, considerations of certifying safety, fairness, and
robustness have led to a number of
applications seeking to measure tail risks in avenues extending beyond
operations and  risk management as well.  Assessing the
safety of automation in cyber-physical systems
get naturally cast in terms of evaluating expectations restricted to
distribution tails
\citep[eg.,][]{7534875,NamkoongAV,uesato2018rigorous},
as is the case with evaluating severity of algorithmic biases on
minority sub-populations
\citep{FairnessRiskMeasures,jeong2020robust}. 

Motivated by the importance and the challenges in measuring tail risks in a broad variety of such applications,
we consider the estimation of distribution tails of a rich class of performance functionals specified with {enabling tools, such as linear programs, mixed integer linear programs, piecewise linear and quadratic objectives, feature maps/decision rules specified in terms of  deep neural networks, etc., which serve as key modeling ingredients in these applications.}

To describe the challenges in the tail estimation tasks we consider, suppose that $L(\xx)$ denotes the loss (or cost) incurred when the uncertain variables affecting the system, modeled by a random vector $\XX,$ realize the value $\xx \in \Real^d.$ For example, $L(\XX)$ may denote the losses associated with a portfolio exposed to risk factors
$\XX,$ or may capture the {optimal value of the cumulative generation and  transmission costs which could be incurred in meeting demand $\XX$ in the consumer nodes of a power distribution network.}
In general, we take $L(\cdot)$ to be
modeling a suitable performance measure of interest.
The distribution of $L(\XX)$ is not analytically tractable even in
elementary models, and related measures such as its mean, or tail risk
measures such as Value-at-Risk, Conditional Value-at-Risk, etc., are
typically estimated from samples. Estimation and
subsequent optimization of tail risks via simulation becomes
computationally expensive however, as mere sample averaging requires
about $p_u^{-1}\varepsilon^{-2}\delta^{-1}$ samples to achieve a
relative precision of $\varepsilon$ in tasks requiring estimation
of $p_u := \Prob(L(\XX) \geq u)$ with $1-\delta$ confidence. This
prohibitively large requirement points to the need for samplers whose complexity do not grow as severely with $p_u$
decreasing to zero.


Application contexts where tail risk measurement is of central
importance, such as those arising in financial engineering, actuarial
risk, system availability, etc., have facilitated the development of
variance reduction techniques which aim to tackle this
difficulty. Prominent examples include the use of importance sampling
and splitting, potentially in combination with other variance
reduction tools such as control variates, stratification, conditional
Monte Carlo, etc; see \cite{Gbook, AGbook,rubino2009rare} for an overview. In particular, Importance
Sampling (IS) is seen as the primary method for combating rarity of
relevant samples in diverse scientific disciplines and is shown to
offer remarkable variance reduction in financial engineering models \citep[see][]{GHS2000,BJZ2006, GKS2008,Liu2015}, actuarial risk models \citep[see eg.,][]{AGbook,collamore2002}, and various queueing and reliability models \citep[see][and references therein]{Heidelberger,JUNEJA2006291,blanchet2009rare}. 
The idea behind IS is to accelerate the occurrences of the
target risk event in simulation by sampling from an alternate
distribution which places a much greater emphasis on the risk
scenarios of interest. Observed samples are then suitably reweighed to
eliminate the bias introduced. We shall refer to this alternate
sampling distribution as \textit{IS distribution} hereafter.  

\subsection{{Conventional approach towards efficient IS and key challenges} \label{sec:challenges}}  
Effective use of IS in \textit{all} the above instances rely, however, on carefully leveraging the specific model structure {to explicitly identify a suitable IS distribution}. 
{Typically this involves an initial Step (i) seeking to select a parametric distribution family $\mathcal{P}$ with the following desirable properties: the family $\mathcal{P}$ should include distributions which place significantly more probability on the target rare event, while also being rich enough to mirror the \textit{large deviations} behaviour of the theoretically optimal IS distribution. This variance minimizing optimal IS distribution  is merely the conditional law of $\XX$ given the tail event of interest, and is also referred to as the  zero-variance distribution \citep[see eg.,][Chapter 5]{AGbook}.  Identifying the best IS distribution within $\mathcal{P}$ is then accomplished in a subsequent Step (ii) devoted to solving an optimization problem {\tt{(OPT)}} 
formulated suitably in terms of the proposed family $\mathcal{P},$ the distribution of $\XX,$ and level sets of $L(\cdot).$}

Such reliance on large deviations characterizations to
{explicitly identify an effective IS distribution}, while a source of strength, also helps showcase the challenges one may face in our estimation tasks
{which seek to go significantly beyond the piecewise linear assumption on $L(\cdot)$ and independence/normality assumptions on $\XX$ featuring largely in the literature. Even under these narrow assumptions, execution of Step (i) above requires a case-by-case large deviations analysis leading to distinctly different choices for $\mathcal{P}$ based on the model-at-hand. These include $\mathcal{P}$ obtained via exponentially twisting the probability density \citep[eg.,][]{siegmund1976, GHS2000}, twisting the hazard rate \citep[]{JS2002}, mixture families based on the so-called dominating points \citep[]{SadowskyBucklew,honnappa2018dominating,bai2020rareevent}, or mixtures featuring one/many big jumps \citep{AK2006,DLW2007,chen_rhee_zwart_2018}.  Heuristic application of well-known techniques not accompanied by appropriate large deviations have been shown to be ineffective or counterproductive even in instances involving relatively simpler choices of $L(\cdot)$  \citep[see eg.,][]{glasserman_wang,juneja2007asymptotics}}. 

We next highlight that it is often impractical to invoke the specific form of the functional $L(\cdot)$ or that of the distribution of $\XX,$ in their entirety, as is typically required
{for formulating and solving the optimization problem {\tt{(OPT)}} in Step (ii) above. In special cases where $\XX$ is multivariate normal and $L(\cdot)$ is additive and explicitly known, {\tt{(OPT)}} can be written as a quadratic program  \citep[see eg.,][]{GHS2000, GLi2005}; it may additionally possess a combinatorial structure as in \cite{GKS2008}, or may be written as a mixed integer quadratic program as in  \cite{bai2020rareevent}. For some non-Gaussian $\XX,$ {\tt{(OPT)}} could get formulated as a dynamic control problem  even in instances as elementary as independent sums \citep[eg.,][]{DLW2007}. Such variedly nuanced formulation of  the Step (ii) optimization problem {\tt{(OPT)}} could be impractical for more  involved objectives $L(\cdot)$ or non-Gaussian $\XX$. 
Besides this formidable challenge, 
there is no reason to expect the resulting {\tt{(OPT)}}  to be convex or solvable in general, and the identified IS distribution to be easy to sample from.}

\subsection{{Novelty of the proposed approach and main contributions}} In order to overcome the above challenges, 
{we  recast the search for effective IS distributions instead as follows: ``Can we find a transformation $\mv{T}(\cdot)$ whose respective push-forward measure (i.e., the law of  $\mv{T}(\XX)$) readily induces an effective IS distribution when deployed across a large class of models?". This reframed pursuit seeking to induce an effective IS distribution implicitly via a map $\mv{T}(\cdot)$ is a radical departure from the existing prominent approaches (which, as described in Section \ref{sec:challenges}, seek to explicitly identify an efficient IS distribution). While this reframed enquiry is spurred by the bottlenecks highlighted in Section \ref{sec:challenges}, it is not clear apriori if such widely applicable transformations $\mv{T}(\cdot)$ should exist.} 
{In turn, a primary contribution in this paper is to exhibit a fixed family of transformations which, despite being oblivious to the loss $L(\cdot)$ and the underlying distribution, is proved to offer asymptotically optimal variance reduction for a wide variety of models requiring only a mild nonparametric structure.}

{For a high-level view of why such  transformations capable of inducing effective IS distributions could exist, we point to} the following ubiquitous yet unexploited phenomenon of the  \textit{self-similarity} of the optimal IS distributions. This notion serves as a key ingredient in our approach and is explained as follows: For any $u > 0,$ suppose that $P_u^\ast$ denotes
the theoretically optimal IS distribution for estimating
$P(L(\XX) \geq u);$ in other words, $P_u^\ast$ is just the law of $\XX$
conditioned on $\{L(\XX) \geq u\}$ \citep[see][Chapter 5]{AGbook}.  As $u \rightarrow \infty,$ we show that suitably scaled versions of distributions $P_u^\ast$ and $P_l^\ast$ share similar large deviations behaviour and concentrate their mass on identical sets even if the level $l > 0$ is only a fraction of the level $u.$ 
Figure \ref{fig:self-similarity} below offers a graphical illustration of this self-similarity holding across three different distributions for $\XX.$ 


\begin{figure}[t]
 \begin{center} 
\caption{{Illustration of the notion of self-similarity of
          optimal IS distributions: Samples from the distributions
          $P_l^\ast,P_u^\ast$ (displayed in blue and red respectively) reveal
          that they share similar concentration properties for three
          distribution choices of $\XX$ informed by a Gaussian copula
          with correlation $\rho.$ The levels $l,u$ are such that the
          probabilities of $L(\XX)$ exceeding these levels are
          approximately $10^{-3}$ and $10^{-5.5}.$ The contours (drawn in green) represent level sets of
          $L(\xx)=1^\intercal (A\xx - \mv{b})^+$ derived from a ReLU
          neural network with weights given by the matrix $A$ with
          rows $(0.3,1), (1,0.3), (0,1.1), (1.1,0)$ and vector
          $\mv{b} = \mv{0}.$
        }}\label{fig:self-similarity}
        \begin{subfigure}{0.32\textwidth}
          \includegraphics[width=1\textwidth]{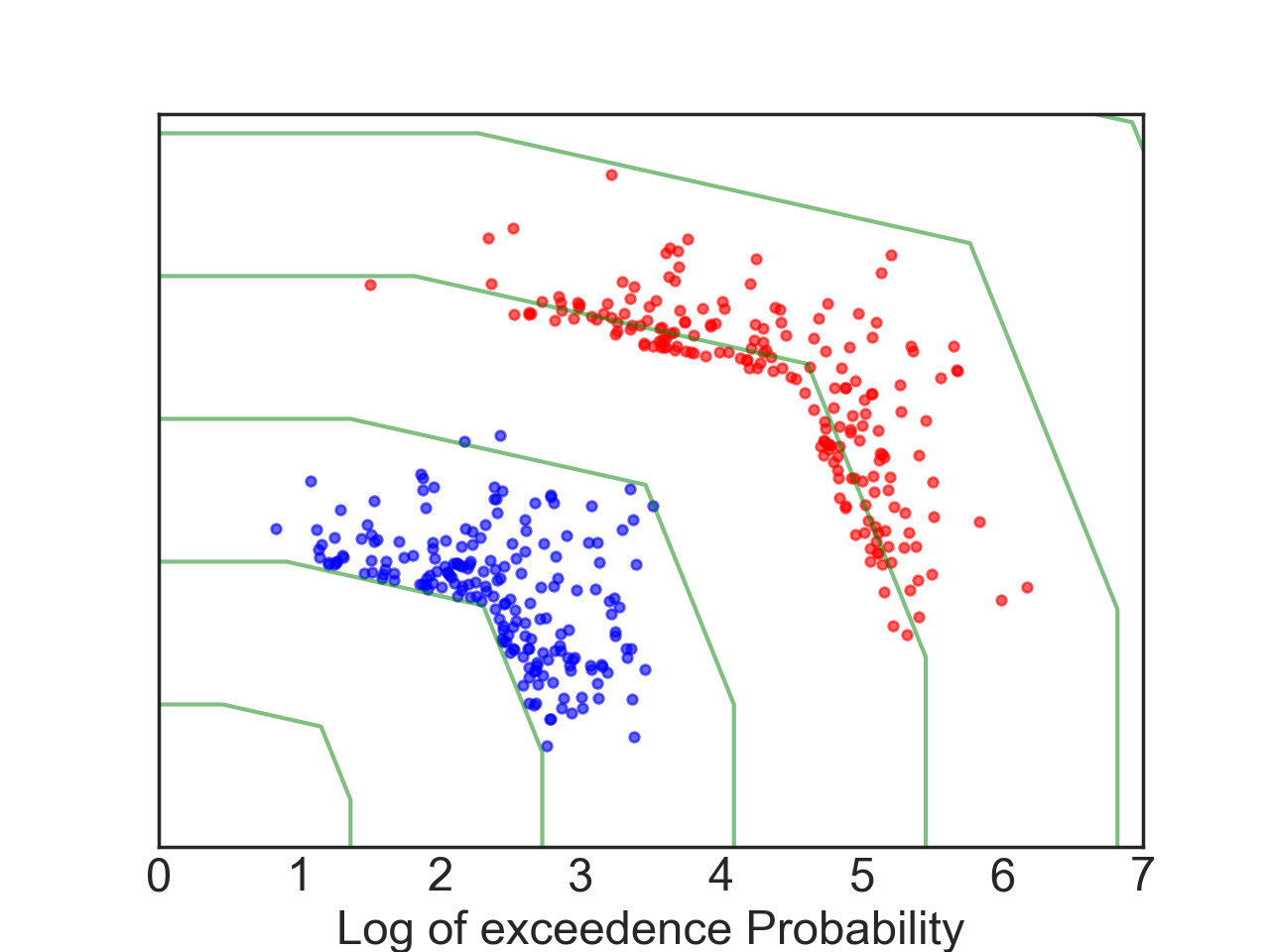}
          \caption{\small{Normal$\,$ marginals,$\,\rho = 0.5$}}
        \end{subfigure}
        \begin{subfigure}{0.32\textwidth}
          \includegraphics[width=1\textwidth]{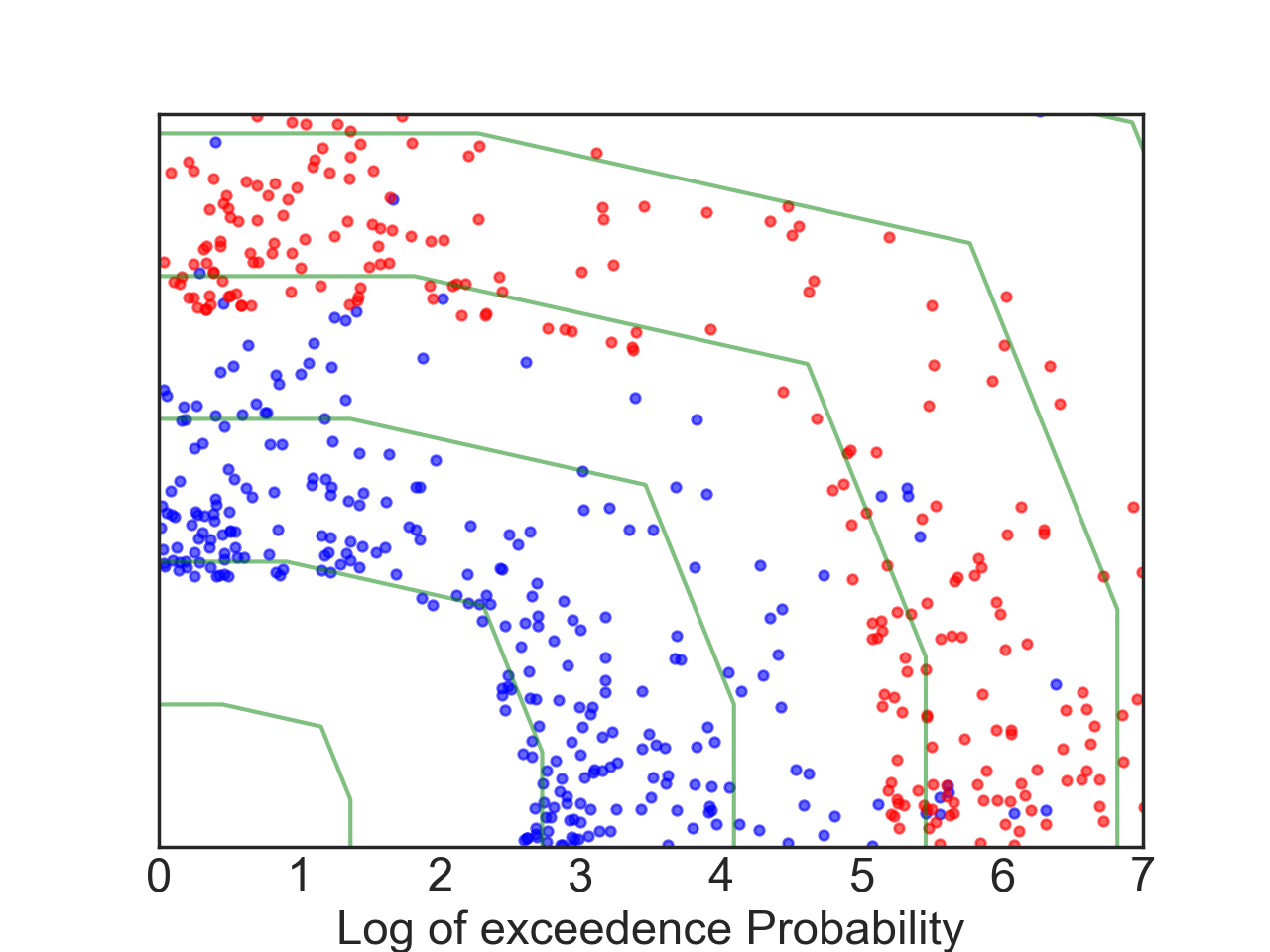}
          \caption{\small{Weibull$\,$marginals,$\,\rho = 0.3$}}
      \end{subfigure}
      \begin{subfigure}{0.33\textwidth}
          \includegraphics[width=0.97\textwidth]{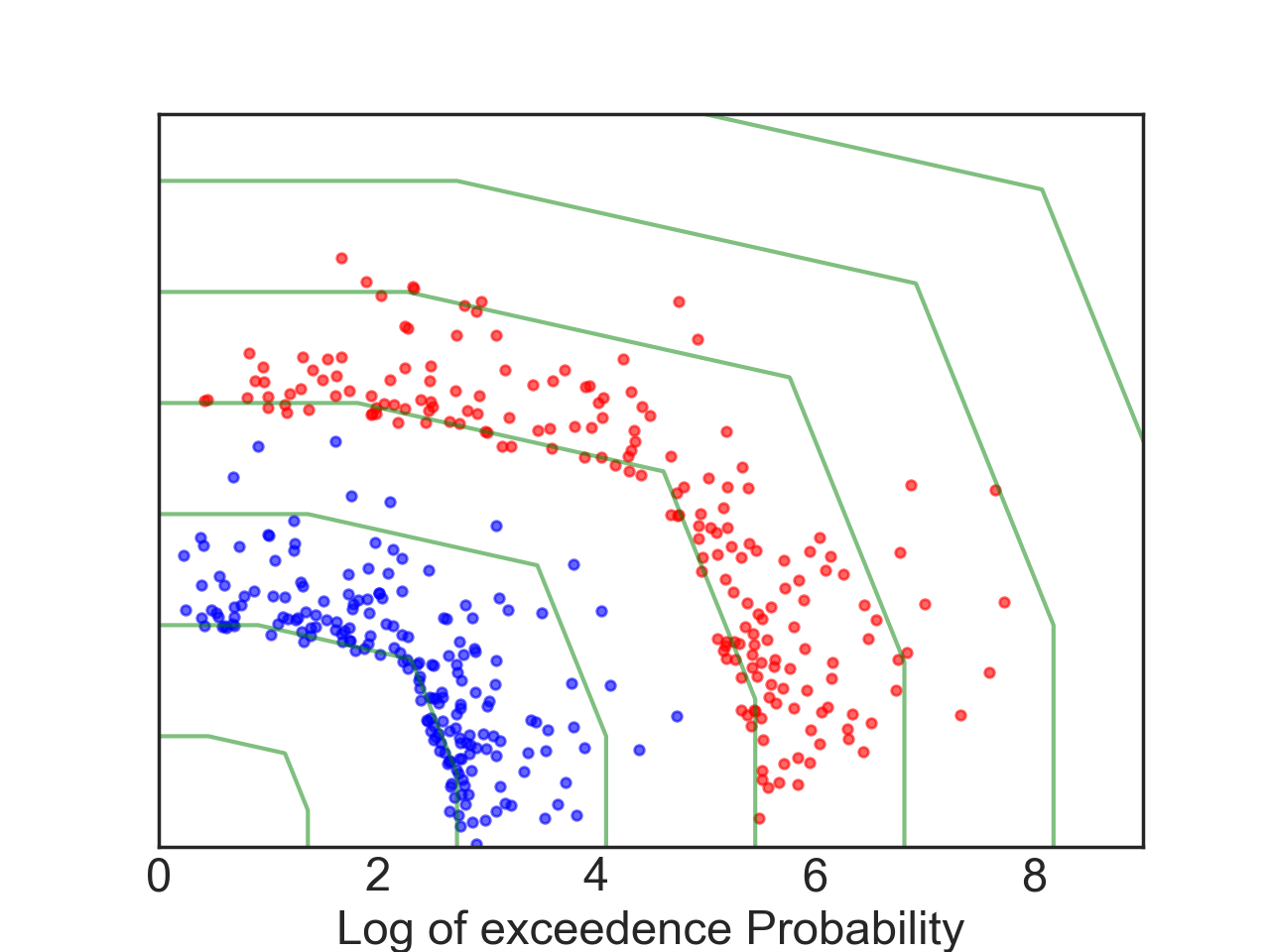}
         \caption{\small{Exponential$\,$marginals,$\,\rho = 0.5$}}
        \end{subfigure}
      \end{center}
    \end{figure}

{To leverage this remarkable similarity in how the samples of $P_u^\ast$ and $P_l^\ast$ concentrate, we seek transformations which automatically replicate the large deviations concentration properties of $P_u^\ast$ from how the much more frequently occurring samples of $P_l^\ast$ manifest.  As a product of this entirely novel approach, 
we are able to make the following main contributions in this paper.}


\noindent
\noindent
\textbf{{1) Tail modeling framework characterizing self-similarity in optimal IS distributions:}} {Building on the tail modeling approach introduced in \cite{de2016approximation}, we identify  a general class of models for which the above  self-similarity in optimal IS distributions can be made precise in terms of large deviations principles (Proposition \ref{prop:zv-ldp}). This self-similarity phenomenon, being nonparametric in nature, is not limited to objectives/distributions fitting within specific parametric assumptions. As a result, our framework becomes the first to feature a rich set of models with
\begin{itemize}
    \item[a)]  objectives $L(\cdot)$ including, but not limited to, those specified in terms of tools such as linear programs, mixed-integer linear programs, piecewise linear and quadratic objectives, feature-maps/decision-rules informed by neural networks, etc., (see Assumption \ref{assume:V}); and
    \item[b)] a wide variety of light and heavy-tailed multivariate distributions for $\XX$ (see Assumptions \ref{assump-marginals}, \ref{assump:joint-Y}, \ref{assume:marginals-HT} and the examples in Tables \ref{tab:marginals-light} - \ref{tab:multivariate}). 
\end{itemize}
}

\noindent 
{\textbf{2) Novel approach to IS:} We exhibit a fixed family of transformations (see \eqref{eqn:alt-T}) which, despite being oblivious to the loss $L(\cdot)$ and the underlying distribution, is able to induce IS distributions with desirable properties in the considered generality: In particular, the target event $\{L(\XX) \geq u\}$ is shown to occur exponentially more frequently under the induced IS distributions, while also ensuring that the resulting conditional excess loss samples mirror the  large deviations properties of the theoretically optimal $P_u^\ast.$ The need to explicitly formulate a good IS density family $\mathcal{P}$ and the optimization problem {\tt (OPT)} (as described in Section \ref{sec:challenges}) 
gets obviated with the radical discovery of these transformations, thereby rendering the selection of IS distributions entirely algorithmic.}


\noindent 
{\textbf{3) An efficient IS algorithm with wider applicability:} The use of the IS transformation in \eqref{eqn:alt-T} results in a novel IS algorithm whose execution requires only oracle access to the evaluations of loss $L(\cdot)$ and the probability density of $\XX$  (see Algorithm \ref{algo:IS}).
We derive large deviations asymptotic for the distribution tails of $L(\XX)$ and show that the proposed sampler offers asymptotically optimal variance reduction in the considered generality (Theorems \ref{thm:Tail-asymp} - \ref{thm:Var-Red-HT}).

This is to be contrasted with the efficient IS changes of measures available, largely on a case-by-case basis, for specific highly stylized objectives $L(\cdot)$ 
and typically under normal distribution assumptions, i.i.d assumptions, or specific copula assumptions in the literature. The proposed sampler joins the recent line of enquiry initiated in the last couple of years \citep[see][]{bai2020rareevent,arief2020deep}  striving to make IS amenable for more sophisticated objectives. The earlier works in this pursuit have restricted the focus to normal distributions and objectives which can be modeled or approximated by piecewise linear functions, with the complexity of the approach scaling less graciously in terms of the number of pieces involved.  A distinguishing feature of the proposed sampler is that it is the first in the literature to consider a spectrum of multivariate light and heavy-tailed distributions simultaneously  and achieve log-efficiency across this spectrum despite tackling several challenging and important objectives (such as value of linear programs, contextual optimization objectives, etc.) for which efficient IS algorithms are  unavailable even under Gaussian distributional assumptions.} 


We demonstrate the utility of the IS scheme in the evaluation of probabilities of (a) large losses in a portfolio credit risk setting, and (b) large delays in contextual routing.
The proposed sampler for portfolio credit risk also serves as an entirely novel addition that extends the scope of applicability of the line of research pursued in
\cite{GHS2000, GLi2005, BJZ2006, GKS2008, Liu2015} to credit risk models which employ diverse copula or algorithmic approaches such as neural
networks.
{ Following \cite{glynn1996importance} and \cite{HongReview}, a follow-up to this work \citep{deo2021efficient} demonstrates how the IS scheme proposed in this paper for estimation of distribution tails can be employed to gain efficient variance reduction in Value-at-Risk and conditional Value-at-Risk estimation. Its use in further  optimization tasks, such as minimizing conditional Value-at-Risk, is explored in Section \ref{sec:opt}.}

The rest of the paper is organized as follows. {Following a description of the problem, our novel IS procedure is introduced in Section \ref{sec:Desc-IS}}. The tail modeling framework introduced in Section \ref{sec:Mod-Framework} is used to establish the large deviations asymptotics and {the self-similarity property of optimal IS distributions in Section \ref{sec:LD-Tails}. Section \ref{sec:IS-VR} identifies  transformations capable of inducing efficient IS distributions} and presents the main results verifying the asymptotic optimal variance reduction properties. An application to the portfolio credit risk setting and results of numerical experiments are presented in Sections \ref{sec:CR} and \ref{sec:num-exp}. Proofs and additional useful examples are given in the accompanying supplementary material.



\section{The problem considered and the proposed IS procedure}
\label{sec:Desc-IS}
Vectors are written in boldface to enable differentiation from
scalars. For any $\mv{a} = (a_1,\ldots,a_d) \in \Real^d,$
$\mv{b} = (b_1,\ldots,b_d) \in \Real^d$ and $c \in \Real,$ we have
$\vert \mv{a} \vert = (\vert a_1 \vert, \ldots, \vert a_d \vert),$
$\mv{a}\mv{b} = (a_1b_1,\ldots,a_d b_d),$
$\mv{a}/\mv{b} = (a_1/b_1,\ldots,a_d/b_d),$
$\mv{a} \vee \mv{b} = (\max\{a_1,b_1\},\ldots,\max\{a_d,b_d\}),$
$\mv{a}^{\mv{b}} = (a_1^{b_1},\ldots,a_d^{b_d}),$
\sloppy{$\mv{a}^{-1} = (1/a_1,\ldots,1/a_d),$}
$\log\mv{a} = (\log a_1,\ldots,\log a_d),$
$c^{\mv{a}} = (c^{a_1},\ldots,c^{a_d}),$
denoting the respective component-wise operations. Let
$\Real^d_+ = \{\xx \in \Real^d: \xx \geq \mv{0}\}$ denote the positive
orthant and $\Real^d_{++}$ 
denote its interior.

\subsection{A description of the problem considered}
\label{sec:Desc}
Suppose that
$L(\xx)$ denotes the cost incurred when the uncertain variables
affecting the problem, modeled by a random vector
$\XX,$ realize the value $\xx \in
\mathbb{R}^d.$
While the loss
$L(\cdot)$ may be expressed as a linear combination of uncertain
variables in some simple settings, the inherent nature of managing operations under resource constraints often results in $L(\cdot)$ expressed suitably as the value of an optimization formulation. 
We consider the task
of estimating the probabilities or expectations associated with tail
risk events of the form $\{L(\XX) \geq u\},$ for a threshold
$u$ suitably large. The need for having a control over likelihoods of
these risk scenarios is inherently present in many operational
settings affected by uncertainty, due to the need to keep the costs
below a target risk level (or) to meet a service-level agreement which
ensures that a target quality of service is
met. 
In many applications, high losses are experienced when the random
vector $\XX$ takes undesirably high values in the positive orthant; for
example, large travel durations in vehicle routing instances leading
to large delays. Thus, without loss of generality, we take the set
specifying risk scenarios, $\{\xx \in \text{supp}(\XX): L(\xx) \geq
u\},$ to be a subset of the positive orthant
$\mathbb{R}^d_+;$ here
$\text{supp}(\XX)$ denotes the support of the distribution of $\XX.$ {Considering a rich class of loss functions $L(\cdot)$ satisfying Assumption \ref{assume:V} below, we aim to design efficient  IS schemes for estimating the distribution tails of $L(\XX).$}



\begin{assumption}
  \textnormal{The function $L : \Real^d \rightarrow \Real$ satisfies
    the following conditions:
    \begin{itemize}
    \item[a)] the set $\{\xx \in \supp(\XX): L(\xx) \geq u\}$ is
      contained in $\Real^d_+$ for all sufficiently large $u;$ and
    \item[b)] for any sequence $\{\xx_n\}_{n \geq 1}$ of
      $\mathbb{R}^d_+$ satisfying $\xx_n \rightarrow \xx,$ we have 
  \begin{align*}
    \lim_{n \rightarrow \infty} \frac{L(n\xx_n)}{n^\rho}= L^\ast(\xx),
  \end{align*}
  where $\rho$ is a positive constant and the limiting function
  $L^\ast: \mathbb{R}^d_+ \rightarrow \mathbb{R}$ is such that the
  cone $\{\xx \in \R^d_+: L^\ast(\xx) > 0\}$ is non-empty.
 \end{itemize} 
}
  \label{assume:V}
\end{assumption}

 Assumption \ref{assume:V}b merely specifies asymptotic homogeneity,
which implies that larger the value of $u,$ farther is the target rare
set from the origin.  The set $\{\xx \in \R^d_+: L^\ast(\xx) > 0\}$ is
necessarily a cone because
$L^\ast(c\xx) = \lim_{n \rightarrow \infty}n^{-\rho}L(nc\xx ) =
c^{\rho}\lim_{n \rightarrow \infty} (cn)^{-\rho}L(cn\xx) = c^\rho
L^\ast(\xx),$ for any $c > 0, \xx \in \mathbb{R}^d_+.$  Examples~\ref{eg:affine-MILP} - \ref{eg:features-NN} below provide a non-exhaustive yet indicative list of objectives $L(\cdot)$ for which Assumption \ref{assume:V} is readily satisfied. {
 A discussion on verifying Assumption~\ref{assume:V}b when only oracle access to loss evaluations $L(\cdot)$ are available is presented in  Section~\ref{sec:Verify_Num}. As Assumption \ref{assume:V} does not require convexity,
the treatment in this paper is applicable even if $L(\cdot)$ is neither convex nor concave.}

\begin{example}[Piecewise affine functions, value of mixed integer
  linear programs]
  \textnormal{Suppose that $L(\cdot)$ can be written as
    \begin{align}
      L(\xx) = \sup_{\mv{\theta} \in \Theta} \{\mv{\theta}^\intercal
      \mv{x} + r(\mv{\theta}) \},
      \label{affine-MILP}
    \end{align}
    where $\Theta$ is a bounded subset of $\mathbb{R}^d$ and
    $r:\Theta \rightarrow \mathbb{R}$ is a bounded function which
    serves to capture terms, if any, which do not involve the random
    vector $\XX.$ In the case where $\Theta$ is a finite set,
    $L(\cdot)$ could represent a piecewise affine function as in,
    $L(\xx) = \max_{k = 1,\ldots,K} \{\mv{\theta}_k^\intercal \xx +
    r_k\},$ where $K$ is a positive integer, $\mv{\theta}_k \in \R^d$
    and $r_k \in \R,$ for $k =
    1,\ldots,K.$ 
    If the set $\Theta$ is described by linear and/or integer
    constraints and if the function $r(\cdot)$ is affine, we have that
    \eqref{affine-MILP} is a linear (or) a mixed integer linear
    program. With the notation used in Assumption~\ref{assume:V}, we
    have $\rho = 1$ and
    $L^\ast(\xx) = \max_{\mv{\theta} \in \Theta}
    \mv{\theta}^\intercal\xx$ for the example $L(\cdot)$ in
    (\ref{affine-MILP}); see \cite[Proposition
    7.29]{rockafellar2009variational}.  The requirements in
    Assumption \ref{assume:V} are met, for example, if
    at least one vector
    in the collection $\Theta$ lies outside the negative orthant
    $\Real^d_{-};$ or, in other words, if
    $\sup \{\theta_{k}: (\theta_1,\ldots,\theta_d) \in \Theta, k =
    1,\ldots,d \} > 0$.  Objectives in many planning problems, such as
    project evaluation and review networks, linear assignment or
    matching, traveling salesman problem, vehicle routing problem,
    max-flow, minimum cost flow, etc., satisfy this requirement either
    in the native formulation or in the respective dual formulation. \hfill$\Box$}
  \label{eg:affine-MILP}
\end{example}

\begin{example}[Piecewise quadratic functions]
  \textnormal{As an extension to Example \ref{eg:affine-MILP}, one may also consider
    piecewise quadratic functions of the form
      $L(\xx) = \max_{k=1,\ldots,K} \left\{\xx^T Q_k \xx +
      \mv{c}_k^T\xx \right\},$
    where $K$ is a positive integer, $\{Q_k: k=1,\ldots,K\}$ are
    $(d \times d)$-symmetric matrices, and $\mv{c}_k \in \Real^d,$ for
    $k=1,\ldots,K.$ As long as the matrices $Q_k$ are not all
    identically zero, we have $\rho = 2$ and
    $L^\ast(\xx) = \max \{\xx^T Q_k \xx: k=1,\ldots,K\}$ in this
    example.  When the support of $\XX$ is bounded from below, the
    requirements in Assumption \ref{assume:V} are automatically met
    if, for example, at least one of the eigen values of the matrices
    in the collection $\{Q_k: k =1,\ldots,K\}$ is positive. If $L(\cdot)$ is instead a piecewise-minimum as in $L(\xx) = \min_{k=1,\ldots,K} \left\{\xx^T Q_k \xx +
      \mv{c}_k^T\xx \right\},$ the requirements are readily checked to be
    satisfied with $\rho = 2$ if the collection $\{Q_k: k =1,\ldots,K\}$ is positive
    semidefinite. \hfill$\Box$}
  \label{eg:piecewise-quadratic}
\end{example}

\begin{example}[Models using contextual information via feature maps/decision rules]
  \textnormal{Suppose that $L(\cdot)$ is written as a composition of
    functions as in
  $L(\xx) = c \left(\mv{\theta}^\intercal \mv{\Phi}(\xx) + \theta_0 \right),$
    where $c: \mathbb{R} \rightarrow \mathbb{R}$ is an objective
    measuring the cost incurred by plugging in a specific decision
    rule or a function approximation based on the feature map
    $\mv{\Phi}:\mathbb{R}^d \rightarrow \mathbb{R}^m$ (see
    \cite{VahnRudin} for an example use of feature-based decision
    rules in newsvendor models). In simple settings, one may take the
    feature map to be merely $\mv{\Phi}(\xx) = \xx,$ or, may include
    cross-terms in the feature vector as in
    $\xx = (x_1,\ldots,x_d) \mapsto (x_i,x_ix_j: i,j = 1,\ldots,d).$
    Motivated by the proliferation of deep neural networks in learning
    expressive feature maps, one may
    consider the feature map $\mv{\Phi}$ to be specified in terms of
    several function compositions defined recursively as in,
    \begin{align}
      \mv{\Phi}(\xx) =  \mv{L}_K(\xx),  \quad 
      \mv{L}_k(\xx) = (A_k\mv{L}_{k-1}(\xx) - \mv{b}_k)^+, k =
      1,\ldots,K, \quad \text{ and } \quad
      L_0(\xx) = (A_0\xx - \mv{b}_0)^+.
      \label{ReLU-NN}
    \end{align}
    In the above, the operation $(\mv{a})^+ = \mv{a} \vee \mv{0},$ $K$
    is a positive integer and for each $k \leq K,$
    $A_k \in \Real^{n_k \times n_{k-1}},$ $\mv{b}_k \in \Real^{n_k}$
    are weight parameters in a neural network with $n_k \geq 1$
    rectified linear activation units (ReLU) in the $k$-th layer. We
    have $n_{K} =: m$ as the dimension of the resulting feature
    map. Refer \cite{sirignano2018deep} for a treatment of
    their utility in identifying relevant features in the context of
    modeling mortgage default risk. For the map $\mv{\Phi}(\cdot)$
    considered in (\ref{ReLU-NN}), we have
    \begin{equation*}
      n^{-1}\mv{\Phi}(n\xx_n)  \rightarrow \big( A_K
      \big( \cdots A_1(A_0\xx)^+ \big)^+\big)^+,
      \text{ for every sequence $\{\xx_n\}_{n \geq 1}$
        of $\Real^d$ satisfying
    $\xx_n \rightarrow \xx.$ }  
    \end{equation*}
    In general, suppose the feature map $\mv{\Phi}$ is such that
    $n^{-p}\mv{\Phi}(nx_n) \rightarrow \mv{\Phi}^\ast(\xx),$ for some
    $p > 0$ and every sequence $\{\xx_n\}_{n \geq 1}$ of $\Real^d_+$
    satisfying $\xx_n \rightarrow \xx.$ Then for the desired
    convergence in Assumption \ref{assume:V}b, we have,
    $L^\ast(\xx) =
    c_+\big[\big(\mv{\theta}^\intercal\mv{\Phi}^\ast(\xx)\big)^+\big]^q
    \ + \
    c_{-}\big[\big(\mv{\theta}^\intercal\mv{\Phi}^\ast(\xx)\big)^{-}\big]^q$
    and $\rho = pq$, if, for example, $c(\cdot)$ is such that
    $c(u)/u^q \rightarrow c_+$ as $u \rightarrow \infty,$
    $c(u)/\vert u \vert^q \rightarrow c_{-}$ as
    $u \rightarrow -\infty,$ with constants $q,c_+,c_{-}$ satisfying
    $q > 0, \min\{c_+, c_{-}\} > 0.$ One may include an additional
    composition to consider models of the form,
    \begin{align}
      L(\mv{s},\mv{\varepsilon}) = \min_{\mv{\theta} \in \mv{\Theta}}
      \ \mv{\theta}^\intercal \mv{c}\big(\mv{\Phi}(\mv{s}),
      \mv{\varepsilon}\big),
      \label{spo-eg}
    \end{align}
    where $\mv{s}$ is seen as contextual side information,
    $\mv{\Phi}(\cdot)$ is a feature map that models the dependence of
    cost vector $\mv{c}$ in terms of the side information $\mv{s}$ and
    additional uncertainty $\mv{\varepsilon},$ and $\mv{\Theta}$
    describes the constraints; see, for example,
    \cite{elmachtoub2020smart} for details and Section \ref{num:SPP}
    for a contextual shortest-path example. Here suppose that the
    feature map $\mv{\Phi}(\cdot)$ is as above and the cost mapping
    $\mv{c}$ is positive and satisfies
    $n^{-\rho}\mv{c}(n^{p}\mv{s}_n, n\mv{\varepsilon}_n) \to
    \mv{c}^\ast(\mv{s},\mv{\varepsilon}),$ for
    $\mv{s}_n \rightarrow \mv{s},$
    $\mv{\varepsilon}_n\to\mv{\varepsilon}$ and some $\rho,p > 0.$ If
    we let $\xx = (\mv{s}, \mv{\varepsilon}),$ we have
    Assumption~\ref{assume:V}(b) satisfied with
    $L^\ast(\mv{s},\mv{\varepsilon}) = \min_{\mv{\theta} \in
      \mv{\Theta}} \mv{\theta}^\intercal \mv{c}^\ast
    \big(\mv{\Phi}^*(\mv{s}),\mv{\varepsilon}\big).$ \hfill$\Box$
  }
\label{eg:features-NN}
\end{example}

One can identify more functionals $L(\cdot)$ which
satisfy Assumption \ref{assume:V}b) by taking linear combinations  {(that is, if $L_1$ and $L_2$ satisfy Assumption~\ref{assume:V}b), so does $L_1+L_2$)} or
compositions suitably {(as in Example \ref{eg:features-NN})} based on  modeling needs. Further, the requirements in Assumption \ref{assume:V}b) can be recast
naturally if a particular application requires the set quantifying
risky scenarios, $\{\xx: L(\xx) \geq u\},$ to be a subset in a different orthant.  


\subsection{The proposed importance sampling method}
\label{sec:IS-prop}
The proposed importance sampling (IS) procedure for  fast evaluation of $p_u := P(L(\XX) > u),$
where $L(\cdot)$ is taken to satisfy Assumption \ref{assume:V}, is
presented in Algorithm \ref{algo:IS} below. {A key ingredient of Algorithm \ref{algo:IS} is a multiplicative transformation of the form,}
\begin{align}
  \mv{T}(\xx) = \xx \times \big( u/l \big)^{\mv{\kappa}(\xx)},
  \label{IS-transf}
\end{align}
{where $l \in [0,u]$ is a hyper-parameter choice and  $\mv{\kappa}: \mathbb{R}^d \rightarrow \mathbb{R}^d_+$ is a suitably defined vector-valued map. 
One may view the components of $\mv{T}(\xx) = (T_1(\xx),\ldots,T_d(\xx))$ as a multiplicative stretching of the components of $\xx = (x_1,\ldots,x_d)$ as in ${T}_i(\xx) = x_i \times (u/l)^{\kappa_i(\xx)} \geq x_i,$ with the extent of stretch of each component determined by the respective exponent $\kappa_i(\xx) \geq 0.$ 

\begin{algorithm}[h!]
  \caption{Self-structuring IS procedure for estimating
    $P(L(\XX) \geq u)$}
  \ \vspace{-2pt}\\
  \KwIn{Threshold $u,$ independent samples
    $\XX_1,\ldots,\XX_N$ of $\XX,$  hyperparameter $l,$ 
    {specify if the exponent $\mv{\kappa}(\cdot)$ is chosen from \eqref{defn:kappa} to be either $\mv{\kappa} = \mv{\kappa}^{(1)}$ (or)  $\mv{\kappa} = \mv{\kappa}^{(2)}$ }}
  \ \vspace{-16pt}\\
  \textbf{Procedure:}\\ 
  \textbf{1. Transform the samples:} For each sample
  $i=1,\ldots,N,$ compute the transformation,
  \begin{align}
    \ZZ_i = \mv{T}(\XX_i) := \XX_i (u/l)^{\mv{\kappa}(\XX_i)},
    \label{imp-transf}
   \end{align}
   \ \vspace{-10pt}\\
   \textbf{2. Compute the associated likelihood:} For each transformed
   sample $\ZZ_i,$ compute the respective likelihood ratio as,
  \begin{align}
    \mathcal{L}_i := \frac{f_{\XX}(\ZZ_i)}{f_{\XX}(\XX_i)}  J(\XX_i), 
    \qquad i = 1,\ldots,N, 
    \label{LLR}
  \end{align}
  where $f_{\XX}(\cdot)$ is the probability density of $\XX$ and $J:\mathbb{R}^d \rightarrow \R_+$ is the Jacobian of the transformation $T$ {(see Table \ref{tab:Jacobians} for expressions of $J(\cdot)$ along with prescribed choices of $\mv{\kappa}(\cdot)$).}

     \ \vspace{-10pt}\\
  \textbf{3. Return the output estimator:} Return the IS average
  computed as in,
  \begin{equation*}
    \bar{\zeta}_{_N}(u) = \frac{1}{N}\sum_{i=1}^{N} \mathcal{L}_i
    \mathbf{I}\left( L(\ZZ_i) \geq u\right).
      \end{equation*}
      \label{algo:IS}
    \end{algorithm}

We provide efficient variance reduction guarantees for Algorithm \ref{algo:IS} when the transformation $\mv{T}(\cdot)$  is used in \eqref{imp-transf} with the exponent $\mv{\kappa}(\cdot)$ fixed to any one of the following two choices:
\begin{align}
  \mv{\kappa}^{(1)}(\xx) := \frac{1}{\rho}  \frac{\log(1 + \vert \xx \vert)}
  {\Vert\log (1 + \vert \xx \vert) \Vert_\infty}, \qquad
    \mv{\kappa}^{(2)}(\xx) &:=  \frac{\log(1 + \vert \xx \vert)}
  {\log l}.
   \label{defn:kappa}
\end{align}
While the former relies on knowledge of the growth parameter $\rho$ in Assumption \ref{assume:V}, the latter is model-agnostic in the sense that it is free of any dependence on $L(\cdot)$ or the distribution of $\XX.$ As a result, $\mv{\kappa}^{(2)}(\xx)$ is applicable more broadly in distribution tail estimation tasks including those in which evaluations of $L(\cdot)$ is  available only via oracle queries. A discussion on how the choice $\mv{\kappa}^{(1)}(\xx)$ can be advantageous in preserving convexity of the objective if one is engaged in further optimization tasks (such as minimizing Conditional Value-at-Risk) is presented in Section \ref{sec:opt}.} 

In 
{Algorithm \ref{algo:IS},} the samples for the proposed IS procedure are taken
as
${\ZZ}_i := \mv{T}(\XX_i), \quad i = 1,\ldots,N,$
{where $\XX_1,\ldots,\XX_N$ are independent and identically distributed as $\XX.$}
The bias resulting from counting the fraction of samples $\ZZ_i$ lying
in the target rare set, instead of that of $\XX_i,$ is adjusted by
multiplying with the respective likelihood ratio term $\mathcal{L}_i$
as in,
\begin{equation}
  \label{eqn:Prob-I.S.}
    \bar{\zeta}_{_N}(u) = \frac{1}{N}\sum_{i=1}^{N} \mathcal{L}_i
        \mathbf{I}\left( L(\ZZ_i) \geq u\right),
\end{equation}
to obtain the estimator $\bar{\zeta}_{_N}(u)$ returned by Algorithm \ref{algo:IS}.
As with any IS procedure, the
likelihood ratio term $\mathcal{L}_i$ is taken to be the ratio between the
probability densities of $\XX$ and $\ZZ$ evaluated at $\ZZ_i$ \citep[see eg.,][Chapter 5]{AGbook}.  With a standard change of variables formula involving the Jacobian of the transformation, $\mathcal{L}_i$ can be written conveniently {as in Table \ref{tab:Jacobians} below. This involves plugging in the choice of $\mv{\kappa}$ in the Jacobian determinant 
\begin{align}\label{eqn:Jac_T_general}
      J(\xx)  :=  \text{det}\left( \frac{\partial \mv{T}(\xx)}{\partial \xx}\right) = (u/l)^{\mv{1}^\intercal \mv{\kappa}(\xx)} \text{det} \left( \mathbb{I}_d + \ln \left(u/l\right) \text{Diag}(\xx)  
      \frac{\partial \mv{\kappa}(\xx)}{\partial \xx} \right),
\end{align}
almost everywhere, to obtain the associated 
 likelihood ratio $\mathcal{L}_i = J(\XX_i) f_{\XX}(\ZZ_i)/f_{\XX}(\XX_i).$} Consequently, as verified in  Proposition \ref{prop:unbiased} below, the resulting estimator $\bar{\zeta}_{_N}(u)$ has no bias.
\begin{table}[h!]
\caption{Choice of the exponent $\mv{\kappa}(\cdot)$ in \eqref{IS-transf} and the respective Jacobian determinant}\label{tab:Jacobians}
  {
  \begin{tabular}{|c|l|l|}
  \hline
      \vspace{-5pt} & & \\
  Choice of  exponent $\mv{\kappa}$ in  & Jacobian determinant  $J(\XX_i)$ in the respective & Additional remarks \\
    \vspace{-5pt} & & \\
   $\mv{T}(\XX_i) =  \XX_i \left( u/l\right)^{\mv{\kappa}(\XX_i)}$ &  
    likelihood ratio $\mathcal{L}_i =  J(\XX_i) \frac{f_{\XX}(Z_i)}{f_{\XX}(X_i)}$ & \\ 
  \vspace{-8pt} & & \\
  \hline 
  & & \\
    $\mv{\kappa(\xx)} = \mv{\kappa^{(1)}(\xx)}$ & 
    $J(\xx)  = \left[\prod_{k=1}^d \tilde{J}_k(\xx) \right]\times \frac{(u/l)^{\mv{1}^\intercal \mv{\kappa}(\xx)}}{\max_{k=1,\ldots,d} \tilde{J}_k(\xx)},$  & $\circ$ Relies on knowing $\rho$ \\
    & where & $\circ$ Advantageous in \\
    & $\quad \tilde{J}_k(\xx)
      := 1+\frac{\rho^{-1}\log(u/l)}{\Vert\log(1+|\xx|)
        \Vert_\infty}\frac{|x_k|}{1+|x_k|}, \ \ k = 1,\ldots,d$   & \ \ retaining convexity \\
     &   & \ \ in optimization tasks \\
\hline & &\\
$\mv{\kappa(\xx)} = \mv{\kappa^{(2)}(\xx)}$ & 
    $J(\xx)  = \left[\prod_{i=1}^d \bar{J}_{k}(\xx) \right]\times (u/l)^{\mv{1}^\intercal \mv{\kappa}(\xx)}$ & $\circ\ $resulting $\mv{T}$ in \eqref{IS-transf}\\
     & where & \ \ does not depend on\\
     & $\quad \bar{J}_{k}(\xx) =  1 +  \frac{\log (u/l)}{\log l}\frac{|x_k|}{1+|x_k|}, \quad k = 1,\ldots,d$ & \ \ $L(\cdot)$ or pdf of $\XX$\\
            &     & \\
\hline
  \end{tabular}
  }
\end{table}

\begin{proposition}
  {Suppose that the transformation $\mv{T}$ in  \eqref{imp-transf} is employed with either $\mv{\kappa}(\XX_i) = \mv{\kappa}^{(1)}(\XX_i)$ (or) $\mv{\kappa}(\XX_i) = \mv{\kappa}^{(2)}(\XX_i),$ and the  likelihood ratio $\mathcal{L}_i$ in \eqref{LLR} is computed with the respective Jacobian determinant in  Table \ref{tab:Jacobians}}. Then for any $u > 0,$ the estimator
  $\bar{\zeta}_{_N}(u)$ is unbiased. In other words,
  $ E[\bar{\zeta}_{_N}(u)] = p_u.$ Moreover there exists a map
  $\mv{T}^{-1}:\Real^d \rightarrow \Real^d$ such that
  $\mv{T}\circ\mv{T}^{-1}(\xx) = \xx$ for almost every
  $\xx \in \Real^d.$

\label{prop:unbiased}
\end{proposition}

The choice of the transformation $\mv{T}(\cdot)$ in Algorithm \ref{algo:IS}, which implicitly specifies the IS density, is guided by the
self-similarity properties of the  distribution of $\XX$ to be 
made concrete with the large deviations framework in following 
Sections \ref{sec:Mod-Framework} - \ref{sec:LD-Tails}.
Building on this framework, an account on the rationale behind the
choice of the IS transformation $\mv{T}$ and its variance
reduction properties is offered in Section \ref{sec:IS-VR}.  Roughly
speaking, the transformation $\mv{T}(\cdot)$ seeks to suitably replicate the
concentration properties of the theoretically optimal IS distribution
 from observations which are not as rare.  This is
facilitated by taking the parameter $l$ such that $l \ll u$ and the
event $\{L(\XX) \geq l\},$ though also a tail risk event, is much more
frequently observed in the initial samples when compared to the target
event $\{L(\XX) \geq u\}.$ 
Even if the parameter $l$ is relatively negligible when compared to
the level $u,$ we show the variance of the resulting IS estimator is
small as in,
\begin{align}
  \text{var}[\,\bar{\zeta}_{_N}(u)\,] = o\left(p_u^{2-\varepsilon}N^{-1}\right),
  \label{log-opt-VR}
\end{align}
as the estimation task is made more challenging by taking
$p_u \rightarrow 0.$ The relationship (\ref{log-opt-VR}) holds for any
arbitrary $\varepsilon > 0$ and any choice of $l = l(u)$
which is taken to be slowly varying in $u$ and satisfying
$\lim_{u \rightarrow \infty} l(u) = +\infty;$ see Theorem
\ref{thm:Var-Red} in Section \ref{sec:IS-VR} for a precise statement
of the variance reduction result and Section \ref{sec:preliminaries}
for the definition and examples of slowly varying functions.
%

{Recall from Section \ref{sec:challenges} that traditional IS approaches typically require solving a non-trivial optimization problem {\tt (OPT)} to identify the best distribution within a chosen family IS distribution $\mathcal{P}.$ Unlike these procedures,  the selection of IS density in our approach is simplified to that of selecting the single parameter $l$ minimizing the sample variance.} The robust variance reduction
guarantee for Algorithm \ref{algo:IS}, obtained for any $l$ which is slowly varying in $u,$ enables to confine the search for a good choice of the hyper-parameter $l$ to be within a relatively narrow collection. One may execute the selection of $l$ by means of cross-validation 
{over candidate choices of $l,$ (or) with a retrospective approximation based search procedure we provide in Section \ref{sec:num-exp} together with numerical examples.} 

Contrast the reduced variance of the IS estimator in
(\ref{log-opt-VR}) with that of the naive sample average which merely
counts the fraction of samples $\{\XX_i: i = 1,\ldots,N\}$ in the
target rare set. In the case of naive sample average, the variance is
$p_u(1-p_u) N^{-1}$ and the coefficient of variation grows as in
$p_u^{-1}N^{-1/2},$ as $p_u \rightarrow 0.$ Thanks to
(\ref{log-opt-VR}), the coefficient of variation of the proposed IS
estimator grows only as $o(p_u^{-\varepsilon} N^{-1/2})$ where
$\varepsilon$ can be arbitrarily small, thus requiring only a
negligible fraction of samples compared to that required by the naive
sample average. 
Any estimator which meets the relative error guarantee in
(\ref{log-opt-VR}) is said to offer asymptotically optimal variance
reduction and is referred to as \textit{logarithmically
  efficient}. Please refer \citet[Chapter 6]{AGbook} for a discussion
on the significance on logarithmic efficiency and why it is a natural
and pragmatic efficiency criterion for estimation tasks pertaining to
rare events.

\section{A nonparametric tail modeling description and associated
  LDP}
\label{sec:Mod-Framework}

\subsection{Preliminaries: Regularly varying functions (Class
  $\mv{\RV}$)}
\label{sec:preliminaries}
A function $f:\R_+ \rightarrow \R_+$ is
said to be \textit{regularly varying} with index $\rho \in \mathbb{R}$
if for every $x > 0,$
\begin{align}
  \lim_{n \rightarrow \infty} \frac{f(nx)}{f(n)} = x^\rho.
  \label{uni-RV}
\end{align}
When referring to \eqref{uni-RV}, we write $f \in \RV,$ or,
$f \in \RV(\rho)$ if there is a need to explicitly specify the
exponent $\rho$. The function $f(x) = x^\rho$ is a canonical example
of the class $\RV(\rho).$ If $\rho = 0,$ then $f$ is specifically
referred as \textit{slowly varying}. Some examples of slowly varying
functions include $\log(1+x),$ $\log \log (e+x), (1 + \log(1+x))^{a}$
where $a \in \mathbb{R},$ $\exp(\log(x)^a)$ where $a \in (0,1),$ or
any function $f$ satisfying
$\lim_{x \rightarrow \infty}f(x) = c \in (0,\infty).$ A function
$f \in \RV(\rho)$ can be written as $f(x) = \ell(x)x^\rho,$ for some
slowly varying $\ell(\cdot)$ and $\rho \in \mathbb{R}.$ Evidently,
\eqref{uni-RV} is a characteristic of all homogeneous functions and
univariate polynomials.  By allowing $\rho$ to be an arbitrary real
number and $\ell(\cdot)$ to be any slowly varying function, the class
$\RV$ possesses substantially improved modeling power.
See, for example, \cite{BorovkovBorovkov} for a detailed
treatment of the properties of the class $\RV.$



%
%

\subsection{Assumptions on the probability distribution of $\XX$}
Let $\bar{F}_i(x_i) := P(X_i > x_i)$ and
$\Lambda_i(x_i) := -\log \bar{F}_i(x_i)$ respectively denote the
complementary c.d.f (also known as survival function) and the
cumulative hazard function of the component $X_i$ in
$\XX = (X_1,\ldots,X_d).$ 
The marginal components $X_i$ are required to satisfy Assumption
\ref{assump-marginals} below.
\begin{assumption}
  \label{assump-marginals}
  \textnormal{ For $i \in \{1,\ldots,d\},$ the marginal components
    $X_i$ are such that $\Lambda_i$ is continuous, strictly increasing
    in an interval of the form $(x_0,\infty),$ and
    $\Lambda_i \in \RV(\alpha_i)$ for some $\alpha_i \in (0,\infty).$}
\end{assumption}
Common examples of distributions which satisfy Assumption
\ref{assump-marginals} are as follows: standard exponential
distribution where $\Lambda_i(x) = x$ satisfies
$\Lambda_i \in \RV(1);$ standard normal distribution where
$\Lambda_i(x) = x^2/2 - \log x(1+o(1)),$ as $x \rightarrow \infty,$
satisfies $\Lambda_i \in \RV(2);$ Weibull distribution with shape
parameter $k \in (0,\infty),$ where $\Lambda_i(x) = x^k,$ satisfies
$\Lambda_i \in \RV(k).$ Other examples of parametric families, along
with respective tail parameters $\alpha_i,$ are given in Table
\ref{tab:marginals-light} in Appendix
\ref{sec:app-eg:tails}.  
This large class includes distributions which are light-tailed and as
well as heavy-tailed distributions of the Weibull type. The case where
the marginal distributions possess even heavier tails, such as
log-normal, pareto, regularly varying distributions, etc., are treated
later in Section \ref{sec:extensions}. 

To describe the joint distribution of $\XX$, we first consider the {standardizing} transformation,
\[\mv{Y} = (Y_1,\ldots,Y_d) := \mv{\Lambda} (\mv{X}), \qquad 
  \text{ where } \mv{\Lambda}(\mv{x}) :=
  (\Lambda_1(x_1),\ldots,\Lambda_d(x_d)),
\]
which ``standardize'' the marginal distributions to that of standard
exponential.
\begin{lemma}
  The marginal distributions of the components $Y_i,$ for
  $i=1,\ldots,d,$ are identical and is given by,
  $P(Y_i > y_i) = \exp\left(-y_i\right),$ for $y_i > 0.$
  \label{lem:Y-exp-marginals}
\end{lemma}
As with the wide-spread practice of modeling joint distributions in
terms of copulas \citep[eg.,][]{copulaEmbrechts}, this
standardization restricts the focus to the dependence structure
without getting distracted by the potential non-identical nature of
marginal distributions of $\XX.$

\begin{assumption}
  \textnormal{The probability density of 
    $\mv{Y} := \mv{\Lambda}(\XX)$ admits the form
      \begin{align}
        f_{\mv{Y}}(\mv{y}) = p(\mv{y})\exp(-\varphi(\mv{y})),
        \label{pdf-Y}
       \end{align}
       where $\varphi(\cdot),p(\cdot)$ satisfy the following: There
       exists a limiting function
       $I:\mathbb{R}^d_+ \rightarrow \mathbb{R}_+$ such that,
       \begin{align}
         n^{-1}\varphi(n\mv{y}_n) \rightarrow I(\yy) \quad \text{ and }  \quad
         n^{-\varepsilon}\log p(n\mv{y}_n) \rightarrow 0,
         \label{limiting-I}
       \end{align}
       for any sequence $\{\yy_n\}_{n \geq 1}$ of $\Real^d_+$
       satisfying $\yy_n \rightarrow \yy \neq \mv{0},$ and
       $\varepsilon > 0.$
      \label{assump:joint-Y}
}
\end{assumption}
{A sufficient condition  for $\XX$ to satisfy Assumption \ref{assump:joint-Y} is that its pdf  is of the form $f_{\XX}(\xx) = \exp(-\psi(\xx)),$ where $\psi$ is multivariate regularly varying (see Section \ref{sec:app-eg:tails} for a definition of multivariate regularly varying functions and a precise statement of the sufficient condition).} The nonparametric nature of the assumption 
suggests that a wide variety of dependence models satisfy Assumption
\ref{assump:joint-Y}. Indeed, most commonly used distribution families
such as multivariate normal, multivariate $t$, elliptical densities,
archimedean copula models, exponential family with any regularly
varying sufficient statistic, {extreme value distributions}, suitable
members of generalized linear models, log-concave densities, etc. can
be verified to satisfy Assumption \ref{assump:joint-Y}. Table
\ref{tab:multivariate} in Appendix \ref{sec:app-eg:tails} is intended
to offer a sample of distribution families which satisfy the marginal
and joint distribution conditions in Assumption \ref{assump-marginals}
- \ref{assump:joint-Y} and to serve as a quick reference for the
limiting function $I(\cdot)$ in Assumption
\ref{assump:joint-Y}.  

\begin{example}[Gaussian copula]
  \textnormal{Suppose that $\mv{Y}$ has a joint distribution given by
    a Gaussian copula with correlation matrix $R.$ Given a copula with
    density $c:[0,1]^d \rightarrow [0,1],$ the respective probability
    density $f_{\mv{Y}}(\cdot)$ can be expressly computed as,
    $ f_{\mv{Y}}(\mv{y}) =
    c\big(\mv{1}-\exp(-\mv{y})\big)\exp(-\mv{1}^\intercal\mv{y})$.
    Therefore 
    \begin{align*}
      f_{\YY}(\yy) =  [\det{R}]^{-1/2} \exp\left(-\mv{1}^\intercal \yy -
      2^{-1}\mv{g}(\yy)^\intercal(R^{-1}-\mathbf{I})\mv{g}(\yy)\right), 
    \end{align*}
    where
    $\mv{g}(\mv{y}) := (\bar{\Phi}^{-1}(\e^{-y_1}), \ldots,
    \bar{\Phi}^{-1}(\e^{-y_d}))$ and
    $\bar{\Phi}(\cdot) : = 1 - \Phi(\cdot)$ is the complementary
    c.d.f. of the standard normal variable.  Thus, in this example, we
    have from the notation in \eqref{pdf-Y} that
    $p(\yy) = [\det{R}]^{-1/2}$ and
    $\varphi(\yy) = -\mv{1}^\intercal \yy -
    2^{-1}\mv{g}(\yy)^\intercal(R^{-1}-\mathbf{I})\mv{g}(\yy).$
    Since $\bar{\Phi}^{-1}(p) = -2\log p \ (1 + o(1)),$
    as $p \rightarrow 0,$ we have
    $\mv{g}(n\mv{y})/(n) \rightarrow (y_1^{1/2},\ldots,y_d^{1/2}),$
    and subsequently,
    $ \varphi(n\yy)/\varphi(n) \rightarrow (\yy^{1/2})^\intercal
    R^{-1} \yy^{1/2},$
    compactly, as $n \rightarrow \infty.$ We therefore have the
    limiting $I(\cdot)$ in Assumption \ref{assump:joint-Y} as
    $I(\yy) := (\yy^{1/2})^\intercal R^{-1} \yy^{1/2}.$ \hfill$\Box$
  }
\label{eg:gauss-copula}
\end{example}
{
Properties of the limiting function $I(\cdot)$ and a continued account of the distributions satisfying Assumption \ref{assump:joint-Y} are presented after introducing the tail large deviations principle in Section \ref{sec:tail-LDP}.}


\subsection{Tail large deviations principle with $I(\cdot)$ as the
  rate function}
\label{sec:tail-LDP}
   A sequence of random vectors $\mv{\xi}_n$ is said to satisfy a
   Large Deviations Principle (LDP) with \textit{rate function} $J(\cdot)$ { and speed $r_n \rightarrow \infty$}
   if,
\begin{align*}
 \limsup_{n \rightarrow \infty}  \frac{1}{r_n}\log P\left( \mv{\xi}_n \in F\right)
  \leq -\inf_{\xx \in F} J(\xx) \quad \text{ and } \quad
   \liminf_{n \rightarrow \infty} \frac{1}{r_n}\log P\left( \mv{\xi}_n \in G\right)
  \geq -\inf_{\xx \in G} J(\xx),
\end{align*}
for every closed subset $F$ and open subset $G.$ 
Theorem \ref{thm:LDP} below establishes the LDP which is useful in the
context considered.
\begin{theorem}[Tail LDP]
  Suppose that $\YY$ is a random vector whose probability density
  admits the form (\ref{pdf-Y}), where the functions
  $\varphi(\cdot),p(\cdot)$ satisfy the convergences in
  (\ref{limiting-I}) for any sequence $\{\yy_n\}_{n \geq 1}$ of
  $\Real^d_+$ satisfying $\yy_n \rightarrow \yy \neq \mv{0},$ and
  $\varepsilon > 0.$ Then the sequence $\{n^{-1}\YY : n \geq 1\}$
  satisfies the large deviations principle with rate function
  $I(\cdot)$ {{and speed $n$}.} 
  %
  \label{thm:LDP}
\end{theorem}


The following useful properties of the limiting function
$I:\Real^d_{+} \rightarrow \Real$ in (\ref{limiting-I}) are deduced
from the conditions in Assumption \ref{assump:joint-Y} and the
conclusion in Theorem \ref{thm:LDP}.

\begin{lemma}
  Suppose that Assumption \ref{assump:joint-Y} holds. Then
  \begin{itemize}
  \item[a)] $I(\cdot)$ is continuous, $I(\mv{0}) = 0$ and $I(\xx) > 0$
    for all $\xx \in \mathbb{R}_d^+ \setminus \{\mv{0}\};$
  \item[b)] $I(\cdot)$ is homogeneous: that is,
    $I(\lambda \xx) = \lambda I(\xx),$ for any
    $\lambda > 0, \xx \in \mathbb{R}_d^+;$
  \item[c)] $I(\cdot)$ has compact level sets; specifically,
    $\inf_{\xx \in \R^d_+: x_i > c} I(\xx)= c,$ for all $c \geq 0$ and
    $i=1,\ldots,d.$
  \end{itemize}
  \label{lem:properties-I}
\end{lemma}

Conversely, any function $I:\R^d_{+} \rightarrow \R_+$ which satisfies
the above conditions can be used to readily specify a joint
distribution for $\YY$ which has standard exponential marginals and
satisfies the tail LDP. This is verified, for instance, by considering $\YY$ for which
$P(\YY > \yy) = \exp(-\inf_{\zz > \yy} I(\zz)).$ 

\begin{figure}[ht]
  \caption{{Illustration of the level sets of
      $I(\cdot)$ capturing different strengths of the positive
      (indicated (+)) or negative (indicated (-)) tail correlations
      between the components of $\YY = (Y_1,Y_2).$ Range of axes
      =[0,5]}}\label{fig:diff-I}
\begin{center}
  \begin{subfigure}{0.24\textwidth}
  \centering
    \includegraphics[width=\textwidth]{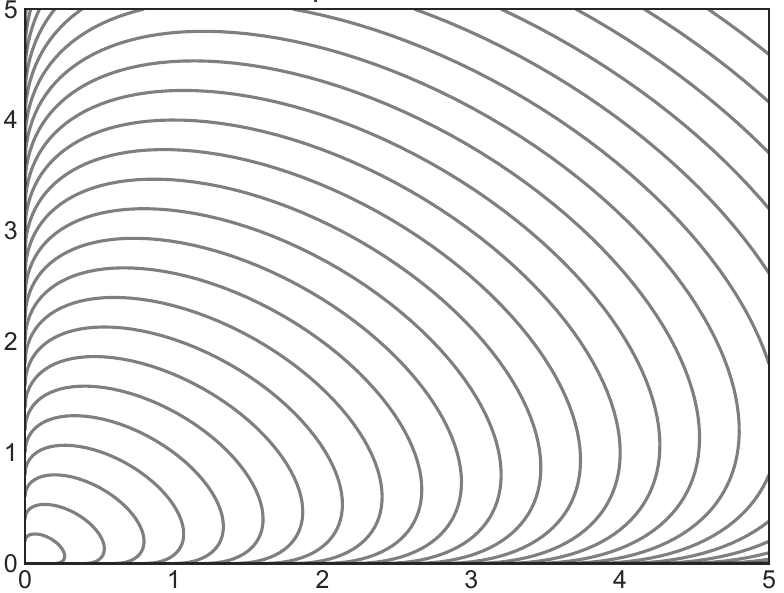}
    \caption{{Gaussian copula(+)}}
\end{subfigure}
\begin{subfigure}{0.24\textwidth}
 \centering
  \includegraphics[width=\textwidth]{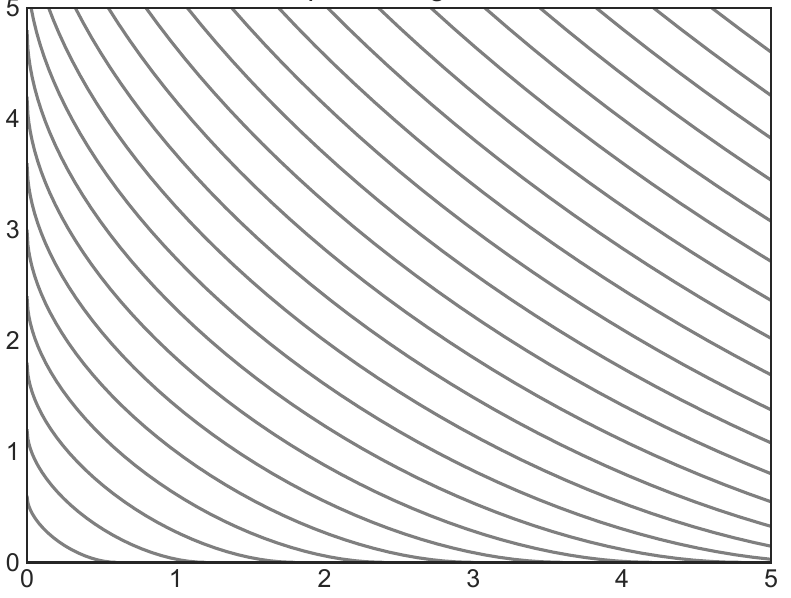}
  \caption{{Gaussian copula (-)}}
  \end{subfigure}
    \begin{subfigure}{0.24\textwidth}
     \centering
      \includegraphics[width=\textwidth]{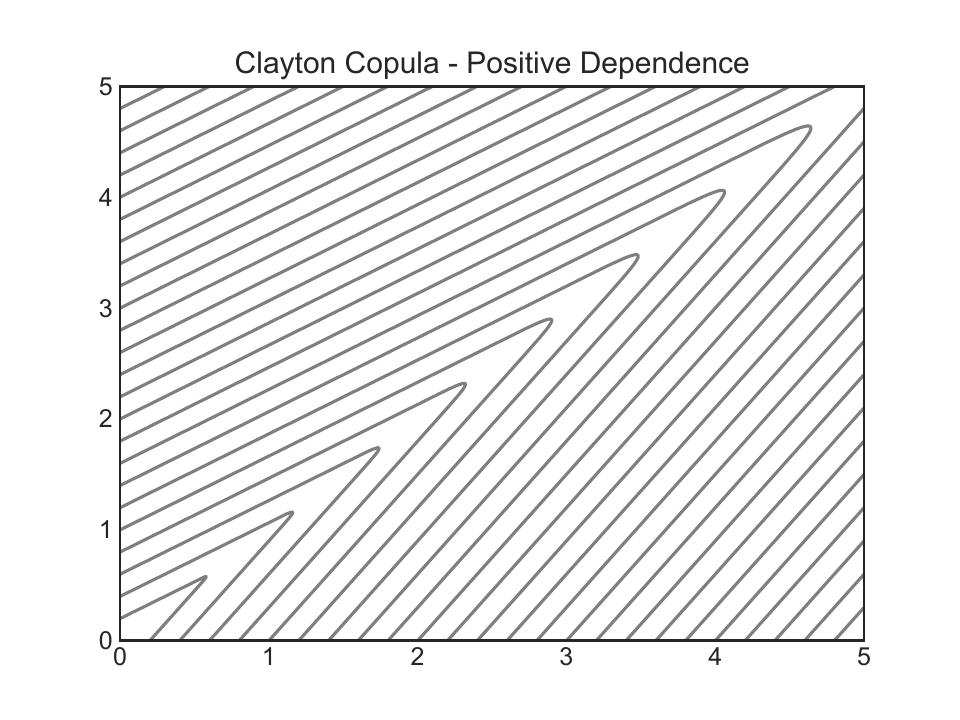}
      \caption{{Clayton copula (+)}}
\end{subfigure}
\begin{subfigure}{0.24\textwidth}
 \centering
  \includegraphics[width=\textwidth]{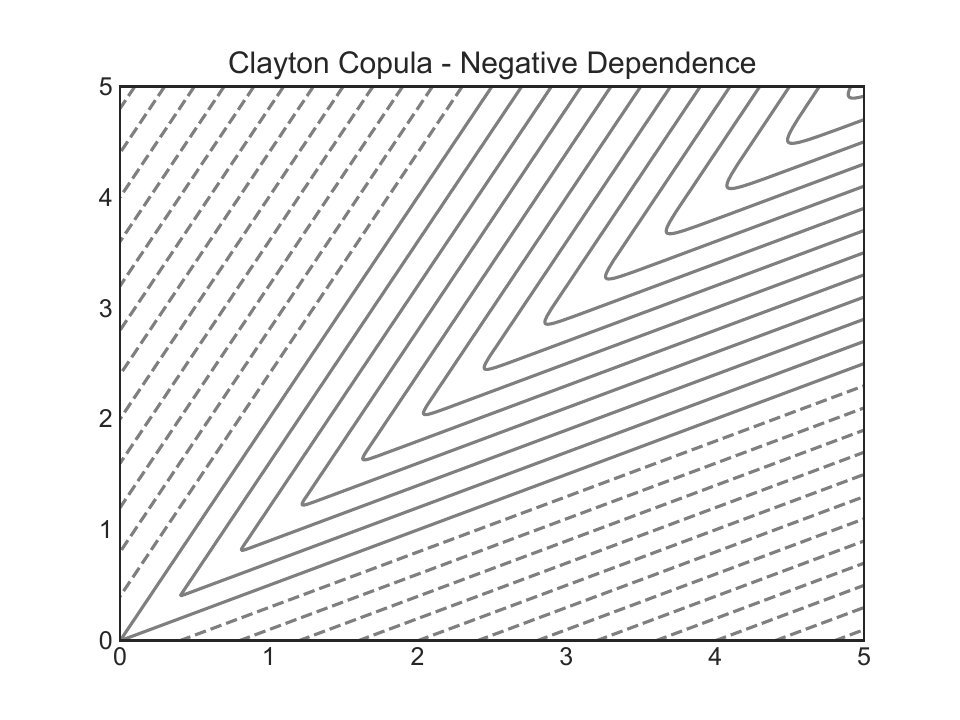}
  \caption{{Clayton copula (-)}}
  \end{subfigure}
\end{center}
\end{figure}

 While the rate
function $I(\cdot)$ is unique for a given distribution for which the
tail LDP holds, it is instructive to note that there may be multiple
distributions which give rise to the same limit $I(\cdot).$ Indeed,
the relation ``has the same rate function $I(\cdot)$ in the tail LDP''
is an equivalence relation, and every function $I(\cdot)$ satisfying
properties (a) - (c) in Lemma \ref{lem:properties-I} specifies a
equivalence class of distributions for the random vector
$\YY.$ 
Thus, the nonparametric nature of the limiting function $I(\cdot)$
offers a great amount of expressive power in capturing the joint
dependence features observed in the the tail regions. The hazard
functions $\Lambda_1(\cdot),\ldots,\Lambda_d(\cdot)$ in Assumption
\ref{assump-marginals}, on the other hand, offer flexibility in terms
of specifying marginal distributions with various tail strengths.




\section{Large deviations {characterizations and zero-variance IS distributions}}
\label{sec:LD-Tails}
{We begin by characterizing} the exponential rate at which
$P(L(\XX) \geq u)$ decays as in,
\begin{align*}
  P(L(\XX) \geq u) = \exp\{- t(u)[I^\ast + o(1)] \},
  \quad \text{ as } u \rightarrow \infty,
\end{align*}
where the function $t(u),$ which grows to infinity as
$u \rightarrow \infty,$ is identified in terms of the marginal hazard
functions $\Lambda_1,\ldots,\Lambda_d$ described in Assumption
\ref{assump-marginals}, and the constant $I^\ast$ is identified in
terms of the marginal tail parameters
$\mv{\alpha} = (\alpha_1,\ldots,\alpha_d)$ and the limiting functions
$L^\ast(\cdot)$ and $I(\cdot)$ in Assumptions \ref{assume:V} and
\ref{assump:joint-Y}.  In order to state the result,
  let us define
  $\Lambda_{\min}: \Real_+ \rightarrow \Real_+$ and
  $\hat{\mv{q}}: \mathbb{R}_+ \rightarrow \mathbb{R}^d_+$ as,
\begin{align}
  \Lambda_{\min}(u) := \min_{i=1,\ldots,d}\Lambda_i(u)
  \quad \text{ and } \quad                      
  \hat{\mv{q}}(t) :=
  \frac{\mv{q}(t\mv{1})}{\ \  \Vert \mv{q}(t\mv{1}) \Vert_\infty},
  \label{defn:qhat}                                                
\end{align}
where $\mv{q}:\Real_+^d \rightarrow \Real^d$ denotes the
component-wise inverse $\mv{q}(\yy) = (q_1(y_1),\ldots,q_d(y_d)),$
with $q_i(y_i) = {\Lambda}_i^{\leftarrow}(y_i)$ specifying the
left-continuous inverse of the hazard function $\Lambda_i(\cdot).$
  
\begin{theorem}[Tail asymptotic]
  \label{thm:Tail-asymp}
  Suppose that the marginal distributions of the components
  $X_1,\ldots,X_d$ of the random vector $\XX = (X_1,\ldots,X_d)$
  satisfy Assumption \ref{assump-marginals} and the standardized
  vector $\YY = \mv{\Lambda}(\XX)$ is such that the tail LDP in
  Theorem~\ref{thm:LDP} holds. Further suppose that the limit
  $\mv{q}^\ast := \lim_{t \rightarrow \infty} \hat{\mv{q}}(t)$ exists.
  Then, for any $L(\cdot)$ satisfying Assumption \ref{assume:V},
  \begin{align}
    \label{tail-asymp}
    {\log P\left( L(\XX) \geq u \right)}
    =  -{\Lambda_{\min}\big( u^{1/\rho}\big)} \big[I^\ast + o(1)\big], 
  \end{align}
  as $u \rightarrow \infty;$ here the non-negative constant $I^\ast$
  is given by,
  \begin{align}
    I^\ast := \inf \big\{ I(\yy): L^\ast\big(\mv{q}^\ast
    \mv{y}^{\mv{1}/\mv{\alpha}}\big)
    \geq 1, \ \yy \geq \mv{0} \big\}.
    \label{I-star}
  \end{align}
\end{theorem}

As a consequence of Theorem \ref{thm:LDP}, the tail asymptotic in
(\ref{tail-asymp}) holds automatically for any joint distribution
specified via Assumptions \ref{assump-marginals} -
\ref{assump:joint-Y}.  While there is a rich literature on tail risk
probabilities of the form $P(X_1 + \cdots + X_d \geq u)$ where
$X_1,\ldots,X_d$ are independent, the treatment for more general
objectives which arise in modeling operations exist only for specific
instances. As examples, we have \cite{GHS2000} deriving asymptotics of
the form (\ref{tail-asymp}) for a specific quadratic objective
$L(\cdot)$ motivated from the delta-gamma approximation of portfolio
losses. Likewise,
\cite{juneja2007asymptotics,blanchet2019rare,ahn2018efficient,bai2020rareevent}
derive asymptotics of the form (\ref{tail-asymp}) considering
piecewise linear $L(\cdot)$ motivated from settings requiring
evaluation of the likelihood of excessive project delays, cascading
failures in product distribution and banking networks, safety in
intelligent physical systems, etc. These results rely, however, on
exploiting the specific structure of $L(\cdot)$ and the distributional
assumptions such as $\XX$ being multivariate normal, or elliptical, or
possessing independent components. On the other hand, Theorem
\ref{thm:Tail-asymp} is applicable for $L(\cdot)$ as general as the
instances considered in Examples~\ref{affine-MILP}-\ref{eg:features-NN} and across a broad spectrum of distributions.

\begin{remark}[Sufficient conditions on existence of
  $\mv{q}^\ast$]\label{rem:q-suff}
  \textnormal{Let $r_i(x) := \Lambda_{\min}(x)/\Lambda_i(x),$
    $i = 1,\ldots,d,$ and
    $\alpha_{\ast} := \min_{i=1,\ldots,d} \alpha_i.$ With
    $\mv{q} :=
    \mv{\Lambda}^{\leftarrow},$ 
    the limit $\mv{q}^\ast = (q^\ast_1,\ldots,q_d^\ast),$ when exists,
    satisfies,
    \begin{align}
      {q}^\ast_i = \lim_{x \rightarrow \infty} r_i(x)^{1/\alpha_{\ast}}, 
      \label{ratio-Lambda}
    \end{align}
    for $i$ in $\{1,\ldots,d\};$ see the discussion at the end of
    Appendix \ref{sec:tech-proofs} for additional explanation. Then
    for any $i$ such that $\alpha_i > \alpha_{\ast},$ the limit in
    (\ref{ratio-Lambda}) expressly evaluates to ${q}_i^\ast = 0.$ For
    all $i$ such that $\alpha_i = \alpha_{\ast},$ we have that the
    limit in (\ref{ratio-Lambda}) exists if
    $\Lambda_i(x) = x^{\alpha_{\ast}} (c_i + o(1)),$ for some positive
    constant $c_i;$ or more generally if
    $ \vert \frac{d}{dx} r_i(x) \vert = O(x^{-(1+\varepsilon)}),$ for
    some $\varepsilon > 0.$ As the latter condition merely restricts
    the magnitude of oscillations of the ratio
    $\Lambda_{\min}(x)/\Lambda_i(x),$ we have that the limit
    $\mv{q}^\ast$ exists for commonly used parametric distribution
    families.} 
\end{remark}

To interpret the tail asymptotic (\ref{tail-asymp}), first note that
the occurrence of
$\Lambda_{\min}(\cdot) := \min_{i=1,\ldots,d} \Lambda_i(\cdot)$ in the
denominator in \eqref{tail-asymp} is aligned with the phenomenon that
the ``heaviest tail wins''. This observation is well-known within the
specialized context of sums of random variables 
\citep[see, for eg.,][Example 8.17]{Hult_2012}. 
Thus, as is expected, the presence of an heavier tail results in
larger probability for $P(L(\XX) \geq u).$
With $\mv{q}^\ast$ characterized as in (\ref{ratio-Lambda}) in terms
of the ratio $r_i(x):=\Lambda_{\min}(x)/\Lambda_i(x),$ the appearance
of $\mv{q}^\ast$ in (\ref{tail-asymp}) captures the differences in tail
heaviness of the marginal distributions of $X_1,\ldots,X_d.$ In the
simpler case where all the components are identically distributed, we
have $\mv{q}^\ast = 1.$ If, for example, $X_1$ is the component with
the heaviest tail in $\XX = (X_1,X_2)$ and if
$P(X_1 > x)/P(X_2 > x) = O(1)$ as $x \rightarrow \infty,$ then
$\mv{q}^\ast = (1,c)$ for some constant $c \in (0,1);$ if on the other
hand, $\Lambda_1(x)/\Lambda_2(x) \rightarrow \infty,$ then
$\mv{q}^\ast = (1,0).$ The same description is applicable in higher
dimensions where $d > 2.$ 

 {
Recall from Section \ref{sec:challenges} that the theoretically optimal IS distribution, which possesses zero variance in the estimation of $P(L(\XX) \geq u),$ 
is merely the  conditional distribution
\begin{align}
P_u^\ast(d\xx) \ :=  \  P \left(\XX \in d\xx \,\vert \, L(\XX) \geq u \right)
  =  \frac{f_{\XX}(\xx)}{P (L(\XX) \geq u)} d\xx.
  \label{ZV-measure}
\end{align}
For brevity, let $\ZZ_u^\ast$ be such that the
$\text{law of }\ZZ_u^\ast = P_u^\ast$ for  $u > 0.$ 
Proposition~\ref{prop:zv-ldp} below gives an LDP for the collection $\{\ZZ_u^\ast: u > 0\}.$  For $l < u,$ the LDP reveals how a suitably scaled version of $\ZZ_l^\ast \sim P_l^\ast$ can be seen to concentrate in regions similar to that of $\ZZ_u^\ast.$ In what follows,  let $I^\ast$ be as defined in \eqref{I-star} and functions $\chi_1: \R^d \rightarrow \{0,+\infty\},$ $t:\R_+ \rightarrow \R_+$ be defined as below:
\begin{align*}
  \chi_{1}(\yy) :=
  \begin{cases}
    0, \quad& \text{ if } L^*(\qq^*\yy^{1/\mv{\alpha}})\geq 1,\\
    +\infty, & \text{ otherwise}, 
  \end{cases}
  \quad \text{ and } \quad
     t(v) := \Lambda_{\min}(v^{1/\rho}).
\end{align*}

\begin{proposition}[Self-similarity in zero variance distributions]
  \label{prop:zv-ldp}
  \sloppy{Under the assumptions in Theorem \ref{thm:Tail-asymp},} the collection  $\{\mv{\Lambda}(\ZZ_u^\ast)/t(u)\}_{u > 0}$ satisfies LDP with rate function $I^{\prime}(\cdot) = I(\cdot) - I^\ast + \chi_1(\cdot)$ and speed $t(u),$ as $u \rightarrow \infty.$ Further, if $l$ is such that $l\to\infty$ and $u/l \to s \in (1,\infty]$ as $u \rightarrow \infty,$ then 
  \begin{align*}
      P\left( \ZZ_u^\ast/\mv{q}(t(u)) \in A\right) &= \exp \left( -t(u) \left[ I^{\prime}(A) + o(1) \right] \right), \quad \text{ and }\\
      P\left( \ZZ_l^\ast/\mv{q}(t(l)) \in A\right) &= \exp \left( -t(l) \left[ I^{\prime}(A) + o(1) \right] \right),
  \end{align*}
  where $A$ is any closed subset of $\mathbb{R}^d$ and the respective $I^{\prime}(A) := \inf_{\pp \in A} I^{\prime}(\pp^{\mv{\alpha}})$.
\end{proposition}
}

\section{Variance reduction properties of Algorithm \ref{algo:IS}}
\label{sec:IS-VR}
{Building on the large deviations characterizations in Sections \ref{sec:Mod-Framework} - \ref{sec:LD-Tails}, this section aims to (i) bring out the rationale behind the choice of $\mv{T}(\cdot)$ employed in Algorithm \ref{algo:IS}; and (ii)} establish the optimal variance reduction properties of the IS estimator $\bar{\zeta}_{_N}(u)$ returned by Algorithm \ref{algo:IS}. 

\subsection{Deducing effective IS transformations from large deviations}
\label{sec:rf-preserving}
{The zero variance distribution in \eqref{ZV-measure}} is not practical as a choice of IS distribution as its specification relies on the unknown quantity $P(L(\XX) \geq u).$  Despite this limitation, the zero-variance distribution is often utilized as a guide for identifying a good choice of IS distribution. Indeed, most verifiably effective IS schemes seek to identify a proposal IS distribution which possess relevant aspects of
the respective zero-variance IS distribution and approximate it in a suitable manner \citep[see, for eg.,][Chapter 6]{JS2002, AGbook}.
{ 
In Definition \ref{defn:Conc-Preserving-Dist} below, we  instead seek transformations $\{\TTu : u > 0\}$ such that (i) the distribution of $\TTu(\XX) \mid L(\mv{T}_u(\XX))\geq u$ and the zero-variance distribution $P_u^\ast$ in  \eqref{ZV-measure} share similar large deviation properties; and (ii) the risky scenarios in the target rare set $\{\xx: L(\xx)\geq u\}$ occur exponentially more frequently under the distribution of $\TTu(\XX).$}
{Throughout this section, let $\TTu$ denote a mapping with domain and co-domain to be $\R^d$ and $\ZZ_u$ denote the conditional realization satisfying
\begin{align}
\label{zz-zz-star-u}
\text{Law of }\ZZ_u = \text{Law of }\TTu(\XX) \mid L(\TTu(\XX)) \geq u.
\end{align}


\begin{definition}
\label{defn:Conc-Preserving-Dist}
  For a given loss $L(\cdot),$ density $f_{\XX}(\cdot),$ and $s \in [1,\infty),$ a family of bijective maps $\{ \TTu: u > 0\}$ is said to be \textit{rate-function preserving} for $(L,f_{\XX})$ with speed-up $s$ if the following hold as $u \rightarrow \infty:$
\begin{itemize}
    \item[(i)] the collection $\{\mv{\Lambda}(\ZZ_u)/t(u) : u > 0\}$ 
    satisfies LDP with the rate function $I^{\prime}(\cdot) = I(\cdot) - I^\ast + \chi_1(\cdot),$ thereby coinciding with the rate function of LDP satisfied by the zero-variance based counterpart $\{\mv{\Lambda}(\ZZ_u^\ast)/t(u): u > 0\};$ and  
\item[(ii)] 
$P(L(\mv{T}_u(\XX)) \geq u)= [P(L(\XX) \geq u/s)]^{ 1 +o(1)},$ 
\end{itemize}
\end{definition}
Observe that the dependence on $\TTu$ in requirement (i) of Definition \ref{defn:Conc-Preserving-Dist} is via  \eqref{zz-zz-star-u}. The requirement (ii) in Definition \ref{defn:Conc-Preserving-Dist} stipulates that, under the change of measure, the target event occurs roughly as frequently as the less rare event $\{L(\XX) \geq u/s\}.$ In particular, due to Thm. \ref{thm:Tail-asymp} and  $\Lambda_{\min}(\cdot) \in \RV(\alpha_\ast)$,
\begin{align*}P(L(\XX) \geq u/s) =  [P(L(\XX) \geq u)]^{ \frac{\Lambda_{\min}((u/s)^{1/\rho})}{\Lambda_{\min}(u^{1/\rho})} + o(1)} = [P(L(\XX) \geq u)]^{ s^{-\alpha_\ast/\rho} + o(1)},
\end{align*}
thus rendering $P(L(\mv{T}_u(\XX)) \geq u)$ to be exponentially larger than $P(L(\XX) \geq u)$ when $s > 1.$ Here, recall that $\alpha_\ast := \min\{ \alpha_1,\ldots,\alpha_d\}.$ Proposition~\ref{prop:Self-Structuring} below gives a  sufficient condition for $\{\TTu\}_{u > 0}$  to be rate-function preserving. 
\begin{proposition}
\label{prop:Self-Structuring} 
Given a loss $L$ and density $f_{\XX}$ satisfying the conditions in Theorem \ref{thm:Tail-asymp},  a family of maps $\{\TTu\}_{u > 0}$ is rate-function preserving for $(L,f_{\XX})$ with speed-up  $s$ if as $u\to\infty$ 
\begin{equation}
\label{eqn:Conc-Preserving-Ts}
  \frac{\mv{T}_u(\mv{q}(t(u)\pp))}{\mv{q}(t(u))} \to  
   (s^{\alpha_\ast/\rho}\pp)^{1/\mv{\alpha}}  \text{ uniformly over $\pp$ in compact subsets of $\R^d_{+}\setminus\{0\}.$}
\end{equation}
\end{proposition}
Since $\mv{q} := \mv{\Lambda}^{\leftarrow} \in \RV(1/\mv{\alpha}),$ the characterization in \eqref{eqn:Conc-Preserving-Ts} can be viewed as $\TTu(\xx) = \xx s^{\frac{\alpha_\ast}{\rho \mv{\alpha}}} + o(\Vert \xx \Vert),$ uniformly over $\xx \in \{\mv{q}(t(u)\pp): \pp \in A\}$ for any compact $A \subseteq \R^d_+ \setminus \{0\}.$ Thus, to obtain a speed-up $s,$ Proposition \ref{prop:Self-Structuring} suggests that it is sufficient if the realizations of $\XX$ in relevant regions are stretched multiplicatively by a suitable factor. The multiplicative factor may need to be different component-wise based on the indices $\alpha_1,\ldots,\alpha_d$ determining the heaviness of  distribution tail of the components $X_1,\ldots, X_d$. Corollaries \ref{cor:Self-struc} - \ref{cor:Self-struc-2} below assert that the transformations
\begin{align}
    \TTu^{(1)} (\xx) &:= \xx \times \left( u/l \right)^{\mv{\kappa}^{(1)}(\xx)},  \quad \text{ where } \mv{\kappa}^{(1)}(\xx) =  \frac{\log(1 + \vert \xx \vert)}
  { \rho \Vert\log (1 + \vert \xx \vert) \Vert_\infty}, 
   \label{eqn:T}
\end{align} 
and 
\begin{align}
      \mv{T}_{u}^{(2)}(\xx) &:= \xx \times \left( u/l \right)^{\mv{\kappa}^{(2)}(\xx)}, \quad  \text{ where } \mv{\kappa}^{(2)}(\xx) = \frac{\log(1+|\xx|)}{\log \, l},
    \label{eqn:alt-T}
\end{align}
introduced in Section \ref{sec:Desc}, in turn, seek to accomplish this component-wise multiplicative stretching and are rate-function preserving with speed-up $s$ for suitable classes of loss $L$ and density $f_{\XX}.$ 

\begin{corollary}
\label{cor:Self-struc}
Consider any $\rho > 0$  and $l = u(1/s + o(1))$ as $u \rightarrow \infty.$ Suppose that the loss $L$ satisfies Assumption \ref{assume:V} with the given $\rho$ and the density $f_{\XX}$ satisfies  the requirements in Theorem \ref{thm:Tail-asymp}. Then 
\begin{align*}
    \mv{\kappa}^{(1)} \left(\mv{q}\big(t(u)\pp\big) \right) \rightarrow \frac{\alpha_\ast}{\rho \mv{\alpha}} 
\end{align*}
uniformly over $\pp$ in compact subsets of $\R^d_{++}.$ The sufficient condition in \eqref{eqn:Conc-Preserving-Ts} is satisfied as a consequence and therefore, the collection of transformations $\{\TTu^{(1)}: u > 0\}$ in  \eqref{eqn:T} is rate-function preserving for $(L,f_{\XX})$ with speed-up $s.$
\end{corollary}

\begin{corollary}
\label{cor:Self-struc-2}
Consider any $l = u(1/s + o(1))$ as $u \rightarrow \infty.$ Suppose that the loss $L$ satisfies Assumption \ref{assume:V} for any $\rho > 0$ and the density $f_{\XX}$ satisfies the requirements in Theorem \ref{thm:Tail-asymp}. Then we have 
\begin{align*}
    \mv{\kappa}^{(2)} \left(\mv{q}\big(t(u)\pp\big) \right) \rightarrow \frac{\alpha_\ast}{\rho \mv{\alpha}} 
\end{align*}
uniformly over $\pp$ in compact subsets of $\R^d_{++}.$ The sufficient condition in \eqref{eqn:Conc-Preserving-Ts} is satisfied as a consequence and therefore, the collection of transformations $\{\TTu^{(2)}: u > 0\}$ in  \eqref{eqn:alt-T} is rate-function preserving for $(L,f_{\XX})$ with speed-up $s.$
\end{corollary}
}

\noindent 
\textbf{A numerical illustration.} Figure \ref{fig:Extrp} below
provides a pictorial illustration of the {rate-function preserving} property by plotting samples from the conditional
distributions observed for the loss $L(\xx) = 0.5(x_1 + x_2).$ A fixed
number of samples from the zero-variance IS distribution are plotted
in red colour in Figures \ref{fig:Extrp}(a) - \ref{fig:Extrp}(c) below
considering different distribution choices for $\XX = (X_1,X_2).$ In
particular, $(X_1,X_2)$ are taken to have identical normal marginal
distributions in Figure \ref{fig:Extrp}(a), exponential marginal
distributions  in Figure \ref{fig:Extrp}(b), and heavier-tailed Weibull
marginal distributions for which $\mv{\alpha} = (0.5,0.5)$ in Figure
\ref{fig:Extrp}(c). To illustrate cases where the concentration of
conditional distributions happen in different regions, the joint
distributions are taken to be given by i)  Gaussian copula with correlation coefficient = 0.5 in Figure \ref{fig:Extrp}(A) ii) { a $t$-copula with degrees of freedom = 1 in} Figure \ref{fig:Extrp}(B), and iii) independent copula in Figure
\ref{fig:Extrp}(c).  In order to facilitate an easy comparison across these different choices of joint distributions, the numbers $u$ and $l$ are taken to be such that $P(L(\XX) \geq u) = 10^{-5}$ and $P(L(\XX) \geq l) = 10^{-2}$ in each of these cases. 
An identical number of samples of the IS random vector $\mv{T}_u^{(1)}(\XX) \,\vert\, L(\mv{T}_u^{(1)}(\XX)) > u,$ computed from the chosen values of $(u,l)$ for each of the above
distributions, are plotted in blue colour in the respective sub-figures in Figure \ref{fig:Extrp}.  In this setup, the following observations are readily inferred from Figure \ref{fig:Extrp}.


\begin{figure}[h!]
  \begin{center}
    \begin{subfigure}{0.32\textwidth}
      \includegraphics[width=0.9\textwidth,height=0.7\textwidth]{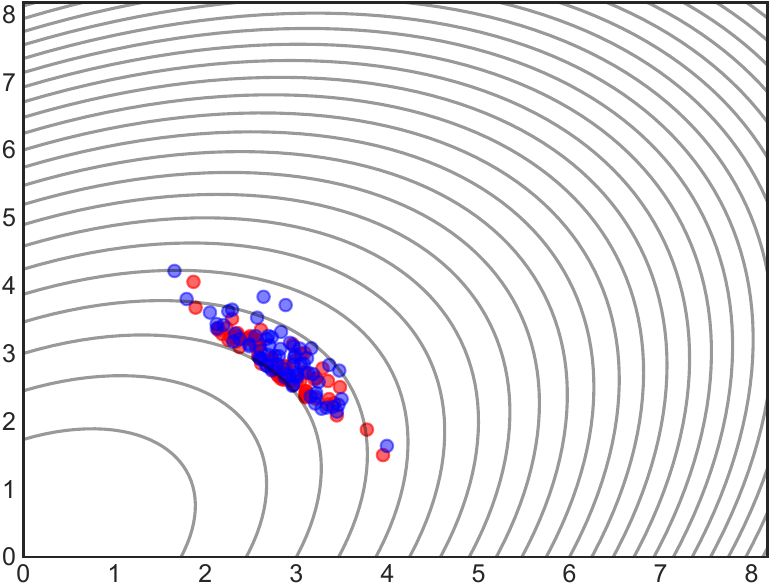}
      \caption{\small  \centering Gaussian marginals,
        correlation$ = 0.5$}
    \end{subfigure}
    \begin{subfigure}{0.32\textwidth}
      \includegraphics[width=0.9\textwidth,height=0.7\textwidth]{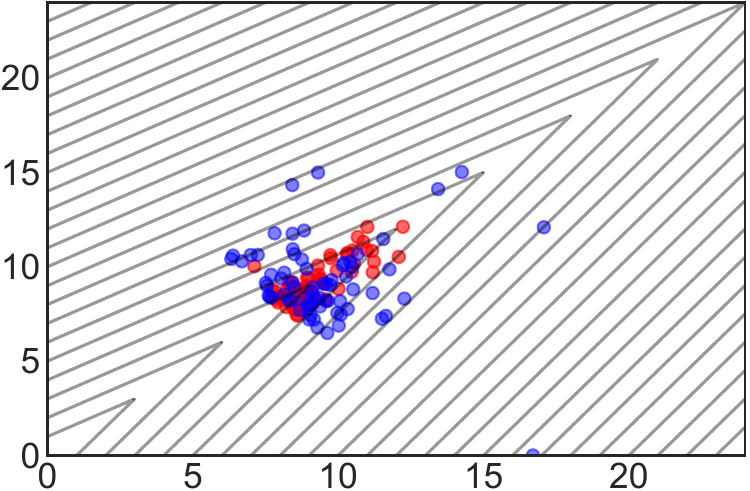}
      \caption{\small \centering Exponential marginals,  \newline \centering        with $t-$copula}
    \end{subfigure}
    \begin{subfigure}{0.32\textwidth}
      \includegraphics[width=0.9\textwidth,height=0.7\textwidth]{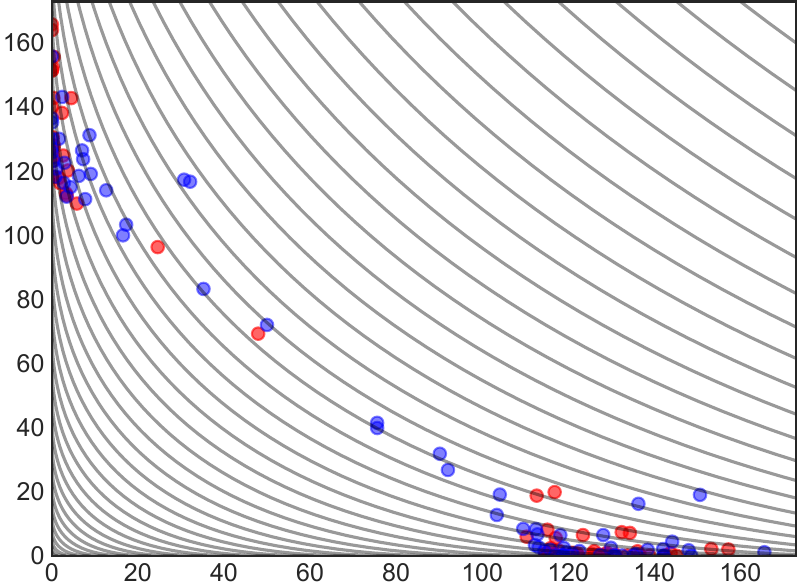}
      \caption {\small \centering Weibull marginals with
        $\mv{\alpha} = (0.5,0.5),$ uncorrelated}
    \end{subfigure}
        \begin{subfigure}{0.32\textwidth}
      \includegraphics[width=0.9\textwidth,height=0.7\textwidth]{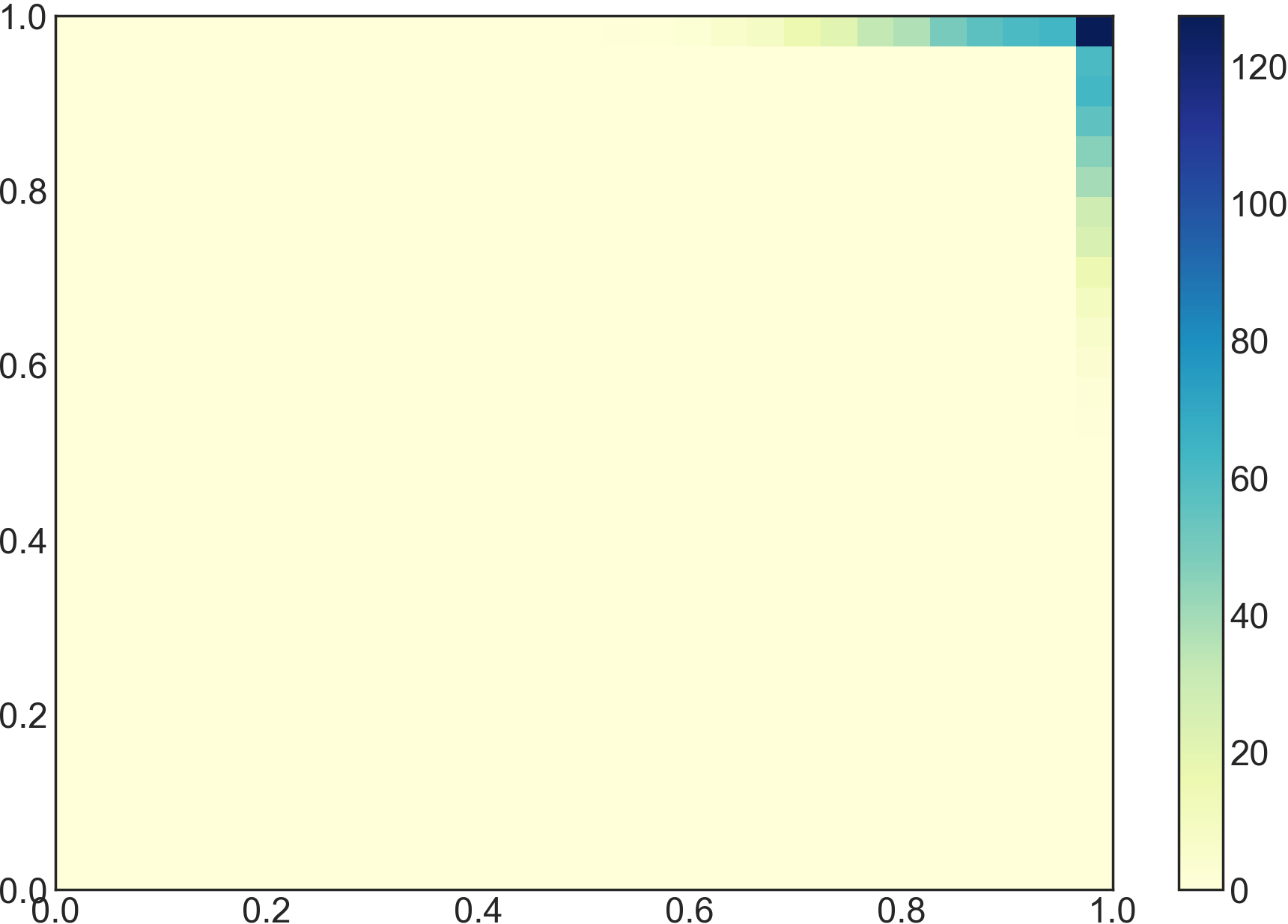}
      \caption{\small \centering}
    \end{subfigure}
    \begin{subfigure}{0.32\textwidth}
      \includegraphics[width=0.9\textwidth,height=0.7\textwidth]{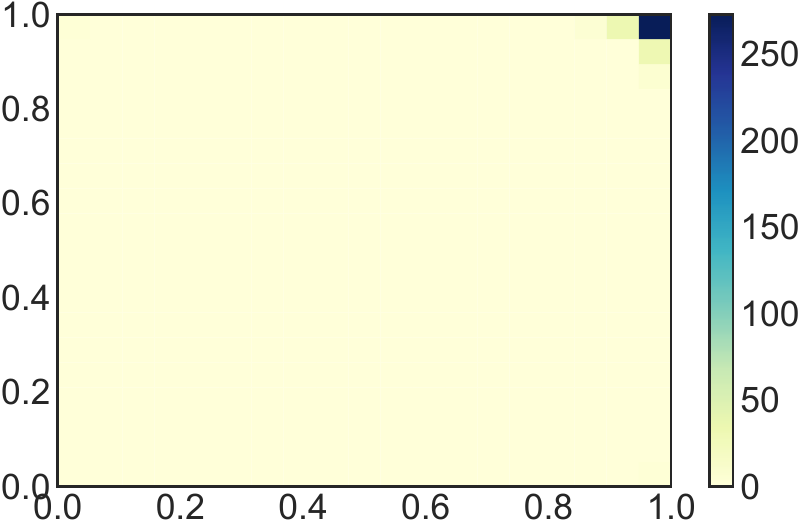}
            \caption{\small \centering }
    \end{subfigure}
    \begin{subfigure}{0.32\textwidth}
      \includegraphics[width=0.9\textwidth,height=0.7\textwidth]{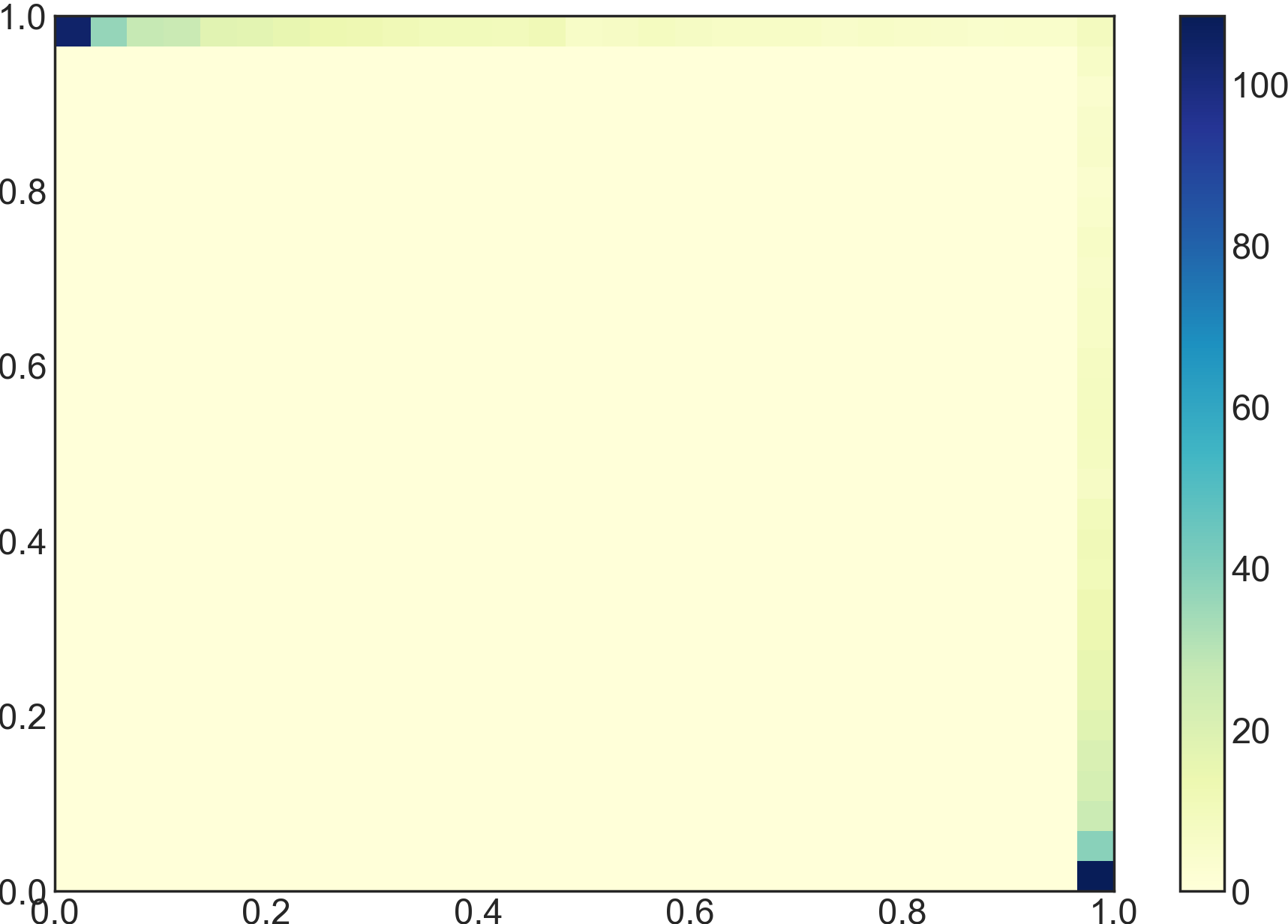}
            \caption{\small \centering }
    \end{subfigure}
  \end{center}
      \caption{\small{Figures (a) - (c) plot independent samples from the zero-variance distribution (in red) and that of the IS vector $\mv{Z} \, \vert \, L(\ZZ) \geq u$ (in blue) to illustrate their identical concentration behaviour. Contours indicate the level sets of the respective joint distributions. Figures (d) - (f) show the respective histograms for $\mv{\kappa}^{(1)}(\XX) \ \vert \ L(\ZZ) \geq u $ involved in the transformation $\ZZ = \mv{T}(\XX)$.} 
             \label{fig:Extrp}}
\end{figure}


 The zero variance and the IS samples tend to concentrate in the same
 neighbourhoods in all the three cases considered in Figures
 \ref{fig:Extrp}(a) - \ref{fig:Extrp}(c) as asserted by Proposition
 \ref{prop:Self-Structuring}.  Regardless of the distinctions in the
 regions where the zero variance distribution
 concentrates, 
 the blue conditional IS samples replicate the concentration in the same
 neighbourhood.
 
To gain intuition behind this phenomenon, we first see that the
 multiplicative factor $(u/l)^{\mv{\kappa}^{(1)}(\xx)} \gg \mv{1}$ in the
 transformation $\mv{T}_u^{(1)}(\xx) = (u/l)^{\mv{\kappa}^{(1)}(\xx)} \xx$ ensures
 that the IS vector $\mv{T}_u^{(1)}(\XX)$ is more likely to take more extreme
 values than $\XX.$ Here the exponent $\mv{\kappa}^{(1)}(\xx)$ ensures that
 the components are relatively magnified only to the extent
 necessary. Indeed, a quick examination by applying the definition of
 $\mv{\kappa}^{(1)}(\xx)$ in  \eqref{defn:kappa}
to the red points in the respective cases in Figure \ref{fig:Extrp}
reveals the following observation: The distribution of
$\mv{\kappa}^{(1)}(\XX) \,\vert\, L(\ZZ) \geq u$ concentrates in the neighbourhood of
the points $\{(1,0),(0,1)\}$ in Figure \ref{fig:Extrp}(f), unlike
those in Figures \ref{fig:Extrp}(d) - \ref{fig:Extrp}(e) where its
concentration is in the vicinity of $(1,1).$ While a naive
multiplication by the factor $(u/l)$ will result in both components
$(X_1,X_2)$ being magnified, the introduction of $(u/l)^{\mv{\kappa}^{(1)}(\XX)}$
lets the conditional distribution of $\ZZ$ concentrate appropriately
near the axes in the heavier-tailed case in Figure
\ref{fig:Extrp}(c). Thus the transformation $\mv{T}_u^{(1)}(\cdot)$ is crucial here in
adjusting the magnification of different components of $\XX$ such that
the transformed vector $\mv{Z} = \mv{T}_u^{(1)}(\XX)$ concentrates measure in
the regions deemed suitable by the zero-variance IS
distribution. 

\subsection{Logarithmic efficiency { of Algorithm \ref{algo:IS}}}
Recall the IS estimator $\bar{\zeta}_{_N}(u)$  
{returned by Algorithm \ref{algo:IS}} is  the sample mean computed from $N$ independent replications of the random variable,
\begin{align*}
   \zeta(u) :=  \mathcal{L}(\ZZ)\mathbb{I}\left( L(\ZZ) \geq u\right),
\end{align*}
where $\ZZ := \mv{T}(\XX)$ and $\mathcal{L}(\ZZ) :=  J(\XX) 
{f_{\XX}(\ZZ)}/{f_{\XX}(\XX)},$
with $ J(\cdot) $ as the Jacobian of the map $\mv{T}(\cdot)$. {Here the map $\mv{T}$ employed in Algorithm \ref{algo:IS} can be seen to coincide with the rate-function preserving transformations \eqref{eqn:T} - \eqref{eqn:alt-T} deduced in Section \ref{sec:rf-preserving}.} Let $M_{2,u} := E\big[ \, \zeta(u)^2 \, \big]$ denote the second moment of $\zeta(u).$ With $\bar{\zeta}_{_N}(u)$ being the average of $N$ independent samples of $\zeta(u),$ the variance of $\bar{\zeta}_{_N}(u)$ is given by $\left(M_{2,u} - p_u^2\right)N^{-1}.$ If one were to take the naive estimator $\mathbb{I}(L(\XX) > u),$ then as explained in Section \ref{sec:IS-prop}, the resulting second moment is $p_u(1-p_u)$ and the relative error scales as $p_u^{-1}.$ Theorem \ref{thm:Var-Red} below 
{establishes that the second moment $M_{2,u} = o(p_u^{2-\varepsilon})$ for any $\varepsilon>0$, thereby offering nearly optimal variance reduction when considering the lower bound $M_{2,u} \geq p_u^2.$} 
In order to state Theorem \ref{thm:Var-Red}, let us introduce a regularity
condition on the marginal distributions of $\XX.$
\begin{assumption}
  There exists $x_0 > 0$ such that for every $i = 1,\ldots,d,$ the
  cumulative hazard function, $\Lambda_i(x) = -\log P(X_i > x),$ is
  either a convex or concave function over the interval
  $x \in [x_0,\infty).$
  \label{assume:mhr}
\end{assumption}
The condition in Assumption \ref{assume:mhr} is readily satisfied for
the examples in Tables \ref{tab:marginals-light} -
\ref{tab:marginals-heavy} in Appendix \ref{sec:app-eg:tails} and for
other commonly used probability distributions.

\begin{theorem} [Logarithmic efficiency]
  \label{thm:Var-Red}
  Suppose $\XX$ satisfies Assumptions \ref{assump-marginals} -
  \ref{assume:mhr} and the limit $\mv{q}^\ast$ exists. For the loss
  $L(\cdot),$ suppose that $L(\cdot)$ satisfies Assumption
  \ref{assume:V} and the limiting function $L^\ast(\cdot)$ is such
  that $L^\ast(\mv{q}^\ast\xx)$ is not identically zero for
  $\xx \in \Real_d^+.$ Then for any choice of parameter $l$ in the IS
  transformation~(\ref{IS-transf}) which is taken to be slowly varying
  in $u,$ the family of estimators $\{\zeta(u): u > 0\}$ is
  logarithmically efficient in estimating $p_u := P(L(\XX) \geq u):$
  that is,
\begin{align}
  \lim_{u\to\infty} \frac{\log M_{2,u}}{\log p_u^2} = 1.
  \label{log-eff-thm}
\end{align}
\end{theorem}

{\color{blue}

{\color{blue}
}
}

\subsection{Log-efficiency in the presence of
    heavier-tailed distributions}
\label{sec:extensions}
Here we present the counterpart for Theorems \ref{thm:Tail-asymp} and
\ref{thm:Var-Red} when one or more of the components of
$X_1,\ldots,X_d$ are heavier-tailed than considered in Assumption
\ref{assump-marginals}. Interestingly, the same Algorithm
\ref{algo:IS} is shown to offer asymptotically optimal variance
reduction  in the presence of heavier-tailed distributions.  As in
Section \ref{sec:Mod-Framework}, we write
\sloppy{$\Lambda_i(x) = -\log P(X_i > x),$} for $x \in \Real.$ In
addition,  let 
\begin{align*}
  \bar{\Lambda}_i(x) &:= -\log P(\log X_i > x) = \Lambda_i \circ \exp
                       (x), \ x \in \Real. 
\end{align*}

\begin{assumption}
  \textnormal{For any $i \in \{1,\ldots,d\}$ for which $\Lambda_i$
    does not satisfy Assumption \ref{assump-marginals},
    $\bar{\Lambda}_i$ is continuous  and strictly increasing in the interval $(x_0,\infty)$, and
    $\bar{\Lambda}_i \in \RV(\alpha_i),$ for some
    $\alpha_i \in [1,\infty).$ }
  \label{assume:marginals-HT}
\end{assumption}

Assumption \ref{assume:marginals-HT} enriches Assumption
\ref{assump-marginals} by including the possibility that, if the
hazard function for $X_i$ is not regularly varying, then the hazard
function for $\log X_i$ is instead regularly varying. This immediately
brings commonly used heavier-tailed distributions such as log-normal,
pareto and regularly varying distributions under the framework
considered. Indeed, if $X_i$ is log-normally distributed, we have
$\bar{\Lambda}_i(x) = x^2/2 - \log x (1+o(1))$ satisfying
$\bar{\Lambda}_i \in \RV(2).$ Instead, if $X_i$ is a pareto or
regularly varying random variable, we have
$\bar{\Lambda}_i (x) = \alpha x - \log L(e^x),$ for some $\alpha > 0$
and a slowly varying function $L(\cdot);$ in this case,
$\bar{\Lambda}_i \in \RV(1)$ (see Table~\ref{tab:marginals-heavy}).
Since the case where all the components $X_1,\ldots,X_d$ satisfy
Assumption \ref{assump-marginals} is treated in the sections before,
we proceed without any loss of generality by assuming here that there
exists at least one component $X_i$ for which Assumption
\ref{assump-marginals} is not satisfied. Let us assign
\begin{align}
  \hat{\mv{q}}(t)  :=  \frac{\log \mv{q}(t\mv{1})}   {\ \Vert \log \mv{q}(t\mv{1})\Vert_\infty },
  \quad \text{ and } \quad
  \mv{q}^\ast := \lim_{t \rightarrow \infty} \hat{\mv{q}}(t),
  \label{defn:qhat-HT}                                                
\end{align}
if the limit exists. Here, 
$\mv{q}:\Real_+^d \rightarrow \Real^d$ denotes the component-wise
\sloppy{inverse} $\mv{q}(\yy) = (q_1(y_1),\ldots,q_d(y_d)),$ with
  $q_i(y_i) = {\Lambda}_i^{\leftarrow}(y_i)$ specifying the
  left-continuous inverse of the hazard function.
  We proceed assuming that the loss
  $L(\cdot)$ satisfies the following variation of Assumption
  \ref{assume:V}.

  \begin{assumption}
  \textnormal{Suppose that the function
    $L : \Real^d \rightarrow \Real_+$ satisfies Assumption
    \ref{assume:V} and the limiting function $L^\ast(\cdot)$ is such
    that
    $\lim_{n \rightarrow \infty} n^{-1}\log L^\ast\left( \exp(n\xx) \right)
    = \bar{L}^\ast(\xx),$ 
  for all $\xx \in \mathbb{R}^d_+$ and some limiting function
  $\bar{L}^\ast: \mathbb{R}^d_+ \rightarrow \mathbb{R}_+.$ }
  \label{assume:V-HT}
\end{assumption}
For instance, in the examples considered earlier in Section
\ref{sec:Desc-IS}, we have the resulting
$\bar{L}^\ast(\xx) = \max\{x_i: i = 1,\ldots,d\}$ and
$\bar{L}^\ast(\xx) = 2\max \{ x_i: i = 1,\ldots,d\}$ { for the linear and quadratic losses in Examples~\ref{affine-MILP}-\ref{eg:features-NN} respectively}. We have the following counterparts to
Theorems \ref{thm:Tail-asymp} \& \ref{thm:Var-Red-HT} in the presence
of heavier tailed distributions.

\begin{theorem}[Tail asymptotic]
  \label{thm:Tail-asymp-HT}
  Suppose that the marginal distributions of the components
  $X_1,\ldots,X_d$ of the random vector $\XX = (X_1,\ldots,X_d)$
  satisfy Assumption \ref{assume:marginals-HT} and the standardized
  vector $\YY = \mv{\Lambda}(\XX)$ is such that the tail LDP in
  Theorem \ref{thm:LDP} holds. Further suppose that the limit
  $\mv{q}^\ast$ in (\ref{defn:qhat-HT}) exists.  Then for any
  $L(\cdot)$ satisfying Assumption \ref{assume:V-HT},
  \begin{align}
    \label{tail-asymp-HT}
    {\log P\left( L(\XX) \geq u \right)}
    =  -{\Lambda_{\min}( u)} \big(I^\ast + o(1)\big), \quad
    u \rightarrow \infty
  \end{align}
  where the non-negative constant
  $I^\ast := \inf \big\{ I(\yy):
  \bar{L}^\ast\big({\mv{q}^\ast}\yy^{\mv{1}/\mv{\alpha}}\big) \geq 1,
  \ \yy \geq \mv{0} \big\}.$
\end{theorem}

\begin{theorem}[Logarithmic efficiency of Algorithm \ref{algo:IS} in
  the presence of heavier tails]
  Suppose that the random vector $\XX$ satisfies Assumptions
  \ref{assump:joint-Y} - \ref{assume:mhr}, \ref{assume:marginals-HT}
  and the limit $\mv{q}^\ast$ in (\ref{defn:qhat-HT}) exists. For the
  loss $L(\cdot),$ suppose that $L(\cdot)$ satisfies Assumption
  \ref{assume:V-HT} and the resulting limiting function
  $\bar{L}^\ast(\cdot)$ is such that $\bar{L}^\ast(\mv{q}^\ast\xx)$ is
  not identically zero for $\xx \in \Real_d^+.$ Taking $\rho = 1$ and the choice $l$ in the IS
  transformation~(\ref{IS-transf}) to be slowly varying in $u,$ the
  second moment $M_{2,u}$ satisfies \eqref{log-eff-thm} and therefore
  the family of estimators $\{\zeta(u): u > 0\}$ is logarithmically
  efficient. 
\label{thm:Var-Red-HT}
\end{theorem}

{While Theorems \ref{thm:Var-Red} \& \ref{thm:Var-Red-HT} prove asymptotically optimal variance reduction at the entire generality considered, we also point out that the proposed IS procedure is not suitable in its current form for example in certain tasks where Assumptions 1-4 are not satisfied: these include for example where $\XX$ is a bounded random variable, or in tail estimation for steady-state simulation.}

\section{Application to Portfolio Credit Risk}
\label{sec:CR}
Efficient IS schemes for estimating excess loss probabilities of a
portfolio of loans have been considered in
\cite{GLi2005,BJZ2006,GKS2008}. A salient feature of these approaches
is the flexibility to have correlated loan defaults informed suitably
via Gaussian or extremal copula models. The repertoire of loan default
probability models considered in the literature since then have 
expanded to include machine learning based approaches aiming to
capture more intricate interactions between the underlying covariates;
see, for example, \cite{RiskLoanPool,sirignano2018deep} and references
therein.  The treatment in this section capitalizes on the generic
applicability of the proposed IS scheme to demonstrate how efficient
samplers can be similarly devised in this
setting.  

To introduce the default model studied here, consider a portfolio of
$m$ loans indexed by $\{1,\ldots,m\}$ belonging to $J \geq 1$
types. 
For any $i \in \{1,\ldots,m\}$, let $t(i) \in \{1,\ldots,J\}$ denote
the type of loan $i$, $Y_i$ denote the indicator random variable that
loan $i$ defaults over a fixed horizon of interest, $e_i$ denote the
exposure upon its default, and $\mv{v}_i \in \mathbb{R}^k$ denote
loan-specific factors (such as original interest rate, original
loan-to-value, original debt-to-income ratios, FICO score, pre-payment
penalty, etc.) which are fixed for a given loan. The average loss
incurred by the portfolio is $L_m := m^{-1}\sum_{i=1}^{m} e_iY_i.$  If we let $\bar{e}_m := m^{-1}\sum_{i=1}^m e_i$ denote the average of the
exposures, then it is clear that $L_m \in (0,\bar{e}_m).$ For a given $q \in (0,1),$ our objective is to estimate the probability of
the excess loss event,
\begin{align}
  \mathcal{E}_m := \{L_m \geq q\bar{e}_m\}, 
  \label{defn:Em}
\end{align}
which is the event that the incurred loss exceeds a given fraction of
the maximum loss. 
To restrict the focus to main ideas, we take the exposure $e_i$ to be
fixed for every $i \in \{1,\ldots,m\}$ and satisfy $e_i \in (0,e_0],$
where $e_0 < \infty$ is the maximum exposure level.

The joint distribution of the default variables $Y_1,\ldots,Y_m$ is
taken to be determined by the loan-specific variables
$\mv{v}_1,\ldots,\mv{v}_m$ and some common stochastic factors
$\XX \in \Real^d_+$ which affect all loans. The common factors $\XX$
may capture region-level economic effects, such as those given by
unemployment level, median income, etc., whose evolution is uncertain
over the time horizon of interest. 
Conditioned on $\XX,$ the default indicators
$Y_1,\ldots,Y_m$ are taken to be independent and the respective
conditional default probabilities are specified by the family of
functions $\{W_j\}_{j = 1,\ldots,J}$ as in,
\begin{align}
  P\left( Y_i = 1 \, \vert\, \XX \right) =
  \frac{\exp\left(W_{t(i)}(\XX, \mv{v}_i) - \gamma\right) }
  {1 + \exp\left(W_{t(i)}(\XX, \mv{v}_i) - \gamma\right)}
  \quad i = 1,\ldots,m,
  \label{loan-def-prob}
\end{align}
almost surely; in the above expression,
$W_{j}:\mathbb{R}^d \times \mathbb{R}^k \rightarrow \mathbb{R}$ and
$\gamma$ is a parameter modeling the rarity of loan defaults.  While
the functions $W_j(\cdot)$ are typically modeled as members of
parametric families \citep[see eg.,][]{RiskLoanPool}, we only require Assumption \ref{assume:Wj} below which is merely a
restatement of Assumption \ref{assume:V}b suitably adapted to this
portfolio credit risk setting.

\begin{assumption}
  \textnormal{For $j \in \{1,\ldots,J\},$ the function
    $W_j: \R^d \times \mathbb{R}^k \rightarrow \Real$ is such that for any
    sequence $\{\xx_n,\mv{v}_n\}_{n \geq 1}$ of $\mathbb{R}^d_+$
    satisfying $(\xx_n, \mv{v}_n) \rightarrow (\xx,\mv{v}),$ we have
  \begin{align*}
    \lim_{n \rightarrow \infty} \frac{W_j(n\xx_n,\mv{v}_n)}{n^\rho}= W_j^\ast(\xx),
  \end{align*}
  where $\rho$ is a positive constant and
  $W_j^\ast: \mathbb{R}^d_+ \rightarrow \mathbb{R}$ is the limiting
  function such that the cone $\{\xx \in \R^d_+: W_j^\ast(\xx) > 0\}$
  is not empty.}
  \label{assume:Wj}
\end{assumption}

The task of estimating $\Prob(\mathcal{E}_m)$ is particularly
challenging in portfolios composed of high quality loans with small
default probabilities. This rarity is specified, for example, in the
default probabilities by letting the parameter $\gamma$ be large in
(\ref{loan-def-prob}). In order to study how an IS scheme fares when
the target event becomes increasingly rare, we embed the given problem
in the sequence of estimation problems indexed by $m,$ with
$m \rightarrow \infty$ and the respective $\gamma$ in
(\ref{loan-def-prob}) and the exposures satisfying
\begin{align}
  \gamma m^{-\eta}   \rightarrow c, \quad
  \bar{e}_m \rightarrow \bar{e}, \quad
  \frac{1}{m} \sum_{i : t(i) = j} e_i \rightarrow \bar{c}_j,
  \label{cr-asymp-1}
\end{align}
for some positive constants $c,\eta,\bar{e},\bar{c}_j,$
$j = 1,\ldots,J.$ Similar asymptotic frameworks form the basis of
analysis in \cite{GLi2005,BJZ2006,GKS2008}; see \cite{DeoJuneja} for a
detailed exposition on the appropriateness of this regime in the
context of logit default models. 
We impose a mild technical requirement that the parameter $q$ in
\eqref{defn:Em} 
does not lie in the set
$\{\bar{e}^{-1}\sum_{j\in I} \bar{c}_j:I \subseteq \{1,\ldots,
J\}\}$. This unrestritive condition is common in the literature for
ease of analysis; see \cite{glasserman2007large}.

We first present a tail asymptotic for the excess loss probability
$\Prob(\mathcal{E}_m)$ in Theorem \ref{thm:credit-risk-asymp} below as
an application of Theorem \ref{thm:Tail-asymp}. In order to state the
result, let $\mathcal{J}_m$ denote the collection of subsets of the
set $\{1,\ldots,J\}$ whose collective exposure exceeds the specified
threshold; in other words,
$\mathcal{J}_m := \{ I \subseteq \{1,\ldots,J\}: \sum_{i:\, t(i) \in I}
e_i \geq mq\bar{e}_m \}.$
Let
\begin{align}
  \Lcr(\xx)
  := \max_{I \in \mathcal{J}_m} \min_{t(i) \in I} W_{t(i)}(\xx,\mv{v}_i),
    \quad \text{ and } \quad   \mathcal{A}_{m}
  := \left\{ \Lcr(\XX) > c(1-\varepsilon_m) m^\eta \right\},
\label{loss-cr}
\end{align}
where $(\varepsilon_m: m \geq 1)$ is a sequence decreasing to zero 
as $m \rightarrow \infty.$
With the event $\mathcal{A}_{m}$ amenable to be treated via Theorem
\ref{thm:Tail-asymp}, Theorem \ref{thm:credit-risk-asymp} below
establishes that the events $\mathcal{E}_m$ and $\mathcal{A}_{m}$
coincide as $m \rightarrow \infty$ and uses this observation to
establish the asymptotic for $\Prob(\mathcal{A}_{m})$ and
$\Prob(\mathcal{E}_m).$



\begin{theorem}[Tail asymptotic for $\Prob(\mathcal{E}_m)$]
  Suppose that the conditional default probabilities are specified as
  in (\ref{loan-def-prob}), with the functions
  $W_1(\cdot),\ldots,W_J(\cdot)$ satisfying Assumption
  \ref{assume:Wj}. In addition, suppose that the convergences in
  (\ref{cr-asymp-1}) hold and $\XX$ satisfies the conditions in
  Theorem \ref{thm:Tail-asymp} with
  $\alpha_\ast := \min_{i=1,\ldots,d}\alpha_i$ satisfying
  $\alpha_\ast < \rho (1+\eta^{-1}).$
  Then as $m \rightarrow \infty,$
  \begin{align}
    \Prob \big(\mathcal{E}_m \setminus \mathcal{A}_m   \big)
    &= o\left(\Prob(\mathcal{E}_m)\right) \quad \text{ and }
      \label{cr-prob-equiv}\\
    \log \Prob \left(\mathcal{E}_m \right)
    \sim \log\Prob \left(\mathcal{A}_m \right) &=
      -\Lambda_{\min}\big( m^{\eta/\rho} \big)
      \big(c^\prime \Icr + o(1)\big),     \nonumber 
  \end{align}
  for any sequence $\varepsilon_m \rightarrow 0$ and
  $\varepsilon_m m^{\eta r/\rho} \rightarrow \infty,$ where
  $r < \rho (1+\eta^{-1}) -\alpha_\ast.$ Here the constant
  $c^\prime:=c^{\alpha_\ast/\rho}$ and $\Icr$ are identified as,
  \begin{align}
 \Icr := \inf \big\{ I(\yy): \max_{I \in \mathcal{J}} \min_{k \in
      I} W_{k}^\ast(\mv{q}^\ast \yy^{\mv{1}/\mv{\alpha}}) \geq 1,
    \mv{y} \geq 0 \big\},
    \label{Icr}
  \end{align}
  in terms of
  $\mathcal{J} := \{ I \subseteq \{1,\ldots,J\}: \sum_{k \in I}
  \bar{c}_k \geq q\bar{e} \}$ and $\mv{q}^\ast$ specified as in
  Theorem \ref{thm:Tail-asymp}.
  \label{thm:credit-risk-asymp}
\end{theorem}

Since the IS procedure introduced in Section \ref{sec:IS-prop} is readily applicable for estimating
$\Prob(\mathcal{A}_m),$ thanks to \eqref{cr-prob-equiv}, one may
suitably modify { it to arrive at Algorithm \ref{algo-IS-PCR-NN} below which is efficient} in estimating the desired excess loss probability $\Prob(\mathcal{E}_m).$ {
The conditional sampling of  default variables in Step 2a of Algorithm~\ref{algo-IS-PCR-NN} involving exponential twisting is conventional \citep[see, for eg.,][Chapter 9]{Gbook}.
The selection of a suitable IS distribution for the common factors $\XX,$ on the other hand, is often non-trivial \citep[see eg.,][]{GKS2008} and is unknown if the functions $W_{j}(\cdot)$ modeling the default probabilities are specified generally, say, for example, in terms of  a ReLU neural network as in \cite{RiskLoanPool}. This non-trivial task of arriving at a suitable IS distribution for $\XX$}  is however readily handled by the IS transformation $\mv{T}(\cdot)$ in Algorithm~\ref{algo-IS-PCR-NN}.
\begin{algorithm}[h!]
  \caption{IS algorithm for estimating the excess loss probability
    $\Prob(\mathcal{E}_m)$ }
  \label{algo-IS-PCR-NN}
  \vspace{5pt} \KwIn{ Loss threshold $q \in (0,1),$ $N$ independent
    samples $\XX_1,\ldots, \XX_N$ of $\XX$, loan-specific values
    $\{\mv{v}_1,\ldots,\mv{v}_m\},$ {$u := \gamma,$}
    hyper-parameter $l.$}
  \textbf{Procedure:}\\
  \textbf{1. Transform the samples and compute associated likelihood:}
  Compute IS samples $\ZZ_1,\ldots,\ZZ_n$ and their likelihoods
  $\mathcal{L}_1,\ldots\mathcal{L}_N$ as
  described in Steps 1 - 2 of  Algorithm~\ref{algo:IS}.\\
 \textbf{2. Obtain samples of portfolio loss:} For $k = 1,\ldots,N,$
  do the following steps:
  \begin{itemize}
  \item[a)] Generate the loan default variables
    $Y_{1,k},\ldots,Y_{m,k},$ where for $i = 1,\ldots,m,$ $Y_{i,k}$
    are independent Bernoulli random variable with success probability
    $\tilde{p}_i(\ZZ_k),$  {where
   \begin{equation}
  \label{eqn:P_lambda}
  \tilde{p}_i(\zz) := \frac{p_{i}(\zz)\exp(\lambda e_i)} {1+ p_{i}(\zz)(\exp(\lambda e_i)-1)},
  \end{equation}
     $p_i(\zz) := P(Y_i=1 \ \vert \ \XX=\zz)$ is as in \eqref{loan-def-prob},} and $\lambda = \lambda_m(\ZZ_k)$ in (\ref{eqn:P_lambda}) is the unique solution of  $\min_{\theta \geq 0} \{-\theta q \bar{e}_m+ \psi_m(\theta,
    \ZZ_k)\}$ {with $\psi_m(\cdot,\mv{z})$ denoting the conditional cumulant
\begin{align*}
  \psi_m(\theta, \zz)
  = {m}^{-1}\log \Expc \big[\exp(\theta mL_m) \vert \ \XX=\zz \big] ={m}^{-1}\sum_{i=1}^{m} \log\big(1+ p_{i}(\zz)(\exp(\theta e_i) -1)).
  \end{align*}}
    \item[b)] Set $L_{m,k} := m^{-1}\sum_{i=1}^m e_i Y_{i,k}$ to be the
    portfolio loss.
  \item[c)] Compute the conditional likelihood associated with the
    collection $\{Y_{1,k},\ldots,Y_{m,k}\}$ by letting
      $\mathcal{L}^y_{k} := \exp \left(-m[\lambda_m(\ZZ_k) L_{m,k} - \psi_m(\lambda_m(\ZZ_k), \ZZ_k)]\right).$
  \end{itemize}
  \textbf{3. Return the output estimator.} Return the IS average
  computed as in,
\[
  \bar{\zeta}_{N}(m) = \frac{1}{N} \sum_{k=1}^N\mathcal{L}_k\mathcal{L}^y_k
  \ \mathbb{I}(L_{m,k} \geq q\bar{e}_m). 
\]
\end{algorithm}

\begin{proposition}[Log-efficiency of Algorithm \ref{algo-IS-PCR-NN}]
  \label{prop:IS-NN}
  Suppose that $\XX$ satisfies the conditions in Theorem
  \ref{thm:Var-Red} and the density of $\XX$ is bounded away from $0$
  on compact subsets of $\R_d^+$. For the conditional default
  probabilities specified in (\ref{loan-def-prob}), let the functions
  $W_1(\cdot),\ldots,W_J(\cdot)$ satisfy Assumption \ref{assume:Wj}
  and the resulting constant $\Icr$ in \eqref{Icr} is finite. Then
  under the asymptotic regime in (\ref{cr-asymp-1}) we have 
  \begin{align*}
    \lim_{m \rightarrow \infty} \frac{\log E[\bar{\zeta}_N(m)^2]}
    {\log P(\mathcal{E}_m)^2} = 1, \qquad  N \geq 1.
  \end{align*}
  for any choice of the parameter $l$ which is slowly varying in $m.$
  In other words, the family of unbiased estimators returned by
  Algorithm \ref{algo-IS-PCR-NN} is logarithmically efficient in
  estimating $\Prob(\mathcal{E}_m).$
\end{proposition}

\section{Numerical Experiments}\label{sec:num-exp}
{We begin with a discussion on the selection of hyperparameter $l$ in Algorithm \ref{algo:IS} before presenting the results of numerical experiments. Recall that the IS estimator  returned by Algorithm \ref{algo:IS}  is unbiased for any choice of $l$ and the number of replications needed to attain a target relative precision is directly determined by the estimator variance. Therefore we seek to select $l$ which minimizes the sample average second moment estimate of the resulting IS estimator. Algorithm \ref{algo:RA_IS} below utilizes Retrospective Approximation (RA) for accomplishing this, as RA is effective in reducing the overall computational effort by balancing the errors due to sample based approximation of the objective and optimization \citep[see][]{pasupathy2010choosing}. 
The RA procedure in Step 2 seeks to progressively increase the sample size $m_k$ employed, if required, while capitalizing on the current iterate $l_{k-1}$ to obtain an improved choice $l_k.$ An intermediate quantile of the distribution of $L(\XX)$ computed from a small pilot simulation run is used to obtain the initial choice $l_0;$ as a rule of thumb, this may be chosen to be in the last two deciles of the distribution of $L(\XX)$. An advantage of our transformation based IS approach is that it requires obtaining samples only from the distribution of $\XX,$ and therefore the  collection of samples in earlier steps get fully reused in successive steps of RA and  in the computation of final IS estimate. In all the estimation tasks in Sections \ref{num:SPP} - \ref{numexp-pcr} below, we report results obtained for target relative precision requiring $\varepsilon = \alpha = 0.05$ in Algorithm~\ref{algo:RA_IS}  and with the initialisation  set to ${\tt tol} =q=0.1,m_0 = 500$, and $c=1.2$.

\begin{algorithm}[h!]
  \caption{Combining IS with Retrospective Approximation based hyperparameter search for estimating $P(L(\XX) \geq u)$ within $\varepsilon$ relative error with $(1-\alpha) \times 100\%$-confidence}
  \label{algo:RA_IS}
  \vspace{5pt} 
  \noindent \textbf{Step 0}: Initialize sample size $m_0$ for pilot run, tolerance {\tt tol} for terminating optimization,  intermediate quantile level $q \in [0.1,30/m_0),$ $c > 1,$ $\textsc{flag}_i = 0, \ i=1,2.$
  
 \noindent  \textbf{Step 1} (Pilot simulation run):  Draw $m_0$ i.i.d. samples  $\XX_1,\ldots \XX_{m_0}$ of $\XX.$ Assign $l_0 = L_{(\lceil q m_0 \rceil)},$ where $L_{(i)}$  is the $i$-th largest value in the collection $\{L(\XX_i): i = 1,\ldots,m_0\}.$
 
 \textbf{Step 2} (Retrospective Approximation): Set $k=1$. Do \textbf{while} $\textsc{flag}_1 = 0$,
  \begin{itemize}
  \item[a) ] Set $m_k = \lceil c m_{k-1} \rceil$ and draw i.i.d. samples $\XX_{m_{k-1}+1},\ldots, \XX_{m_k}$ of $\XX.$\\
  \item[b) ] Minimise second moment estimate  $\secmom(\XX_{1},\ldots,\XX_{m_k};l)$ numerically over $l$ with the initial iterate set to $l_{k-1}.$   Assign $M_k$ and $l_k,$ respectively, to be the optimal value and the solution obtained by solving until the absolute error between successive iterates becomes smaller than
  $\varepsilon_k = m_k^{-1/2}$. The procedure $\secmom(\xx_{1},\ldots,\xx_{k};l)$ for estimating second moment of the IS estimator given $l$ and a collection $\{\xx_1,\ldots,\xx_k\},$ for any $k,$  is as follows:

  \textbf{procedure} $\secmom(\xx_{1},\ldots,\xx_{k};l)$\\
  $\quad$ For $i=1,\ldots,k,$ set $\zz_i = (u/l)^{\mv{\kappa}(\xx_i)}, \mathcal{L}_i = \frac{f_{\XX}(\zz_i)}{f_{\XX}(\xx_i)}J(\xx_i),$ where $J(\xx)$ is as in Table \ref{tab:Jacobians}.\\
  $\quad$ Return the second moment estimate $\hat{M}_2(l) = \frac{1}{k}\sum_{i=1}^k \mathcal{L}_i^2 \mathbf{I}(L(\zz_i) \geq u).$\\
  \textbf{end procedure}
  
  \item[c) ] \textbf{If} relative improvement $(M_k-M_{k-1})/M_{k-1} < {\tt tol},$ 
  terminate while loop by setting  $\textsc{flag}_1 = 1,$ choice of $l = l_k,$ $\zeta_i= \mathcal{L}_i \mathbf{I}(L(\ZZ_i) \geq u),$ with $\ZZ_i = (u/l)^{\mv{\kappa}(\XX_i)},$ $\mathcal{L}_i = \frac{f_{\XX}(\ZZ_i)}{f_{\XX}(\XX_i)} J(\XX_i),$ for $i\leq m_k$; \textbf{else} increment $k \leftarrow k+1.$\\
  \end{itemize}
  \textbf{end while}


\textbf{Step 3} (Compute IS estimate with desired precision):  Set $n=m_k+1$.  Do \textbf{while} $\textsc{flag}_2= 0$,
\begin{itemize}
    \item [a) ] Draw an independent sample $\XX_n$ from the distribution of $\XX.$ Set $\zeta_n = \mathcal{L}_n \mathbf{I}(L(\ZZ_n) \geq u)$, where  $\ZZ_n =\XX_n (u/l)^{\mv{\kappa}(\XX_n)} $ and $\mathcal{L}_n =  \frac{f_{\XX}(\ZZ_n)}{f_{\XX}(\XX_n)} J(\XX_n) $. Evaluate  $\bar{\zeta}_n(u)$ as the sample mean and $\hat \sigma$ as the sample standard deviation of the collection $\{\zeta_i : i\leq n\}$. 
    \item [b) ] \textbf{If} $\Phi^{-1}(1-\frac{\alpha}{2}) \frac{\hat{\sigma}}{\sqrt{n}} <  \bar{\zeta}_n(u) \varepsilon$, terminate  loop by setting $\textsc{flag}_2 =1$; \textbf{else} increment  $n \leftarrow n+1$.
\end{itemize}
\textbf{end while} 

\noindent \textbf{Return} $\bar{\zeta}_n(u)$ as the estimate for $P(L(\XX) \geq u).$\\
\end{algorithm}


}

\subsection{Illustration with the contextual shortest path problem}\label{num:SPP}
Here we employ the IS scheme for the contextual shortest path problem considered in \cite{elmachtoub2020smart}.  The goal is to travel from the north-west corner to the south-east corner of a $5\times 5$ grid. From any given vertex, only edges which travel south or east are available.  Associated with each edge $j$ (enumerated as
$\{1,\ldots 40\}$) is a traversing cost $C_j(\mv{S},\mv{\varepsilon})$ determined by contextual side information $\mv{S}$ and an additional random vector $\mv{\varepsilon}.$ The context $\mv{S} \in \R^5$ is
taken to have independent Weibull marginals, with
$F_i(x) = 1-\exp(-x^{0.5}),$ for $1\leq i\leq 5 $.  The cost $C_j(\mv{S},\mv{\varepsilon})$ is given by, 
\begin{equation}\label{eqn:Grigas_cost}
  \mv{C}(\mv{S},\mv{\varepsilon}) = \big[ 5^{-1/2} (\mv{B} \mv{S} +\mv{3})^{\textrm{deg}} +\mv{1}  \big]\mv{\varepsilon},
\end{equation}
where $\mv{B}$ is a fixed $40\times 5$ matrix, $\text{deg} \in (0,\infty)$ allows nonlinear dependence on $\mv{S},$ and the independent noise $\mv{\varepsilon}$ has i.i.d. components uniformly distributed on $[1-a,1+a]$. 
The loss $L(\mv{C}) = L(\mv{S},\mv{\varepsilon})$ is given as in \eqref{spo-eg}, where $\mv{\Theta}$ is the shortest path polytope on the considered  $5 \times 5$ grid. 

\textbf{Experiment 1: } We consider the estimation of
$p_u := P(L(\mv{C}) \geq u),$ for values of $u$ resulting in $p_u \in [10^{-5.5}, 10^{-2.5}].$ The probability of large travel cost is estimated as in Algorithm \ref{algo:RA_IS} by averaging i.i.d. samples of $\mathcal{L} \mv{I}(L(\tilde{\mv{C}}) \geq u),$ where  $\tilde{\mv{C}}= \mv{C} (\mv{T}(\mv{S}),\mv{\varepsilon})$  is the IS cost vector and $\mathcal{L}$ is
the respective likelihood ratio. Letting $V_u$ denote the sample variance of the {collection $\{\mathcal{L}_i \mv{I}(L(\tilde{\mv{C}_i}) \geq u) :i \leq n\}$}, we plot $\log V_u$ against $\log p_u$ in Figure \ref{fig:Context}(a)  {obtained for/ both the choices of  $\mv{\kappa}$ in \eqref{defn:kappa} and} parameter choices
$a =0.25, \text{ deg} = 1.3.$  Note that the plot for IS variance is approximately a straight line with a slope of 2, supporting the conclusion $V_u = p_u^{2 + o(1)}$ from Theorem \ref{thm:Var-Red}. { Naive sample averaging, on the other hand, can be seen to have orders of magnitude larger variance in Figure \ref{fig:Context}(a).} 
{\color{blue}}

\textbf{Experiment 2:} Here we consider the predictive setting where a
routing decision
$\hat{\mv{\theta}}(\hat{\mv{C}}) =\arg\min_{\mv{\theta}\in
  \mv{\Theta}} \mv{\theta}^\intercal\mv{\hat{C}}$ is obtained by
plugging in the cost $\hat{\mv{C}}$ predicted from the realized
contextual side information $\mv{S}.$ Note that while the realized
cost $\mv{C}$ depends on both $\mv{S}$ and $\mv{\varepsilon},$ the
realization of $\mv{\varepsilon}$ is not available at the time of
decision-making.
Suppose that similar to \cite[Section 6.1]{elmachtoub2020smart}, the
cost is predicted from a linear model
$\hat{\mv{C}} = {\mv{\hat{A}}}\mv{S};$ we take 
$\mv{\hat{A}}\in \arg\min\left\{ 10^{-3} \sum_{i=1}^{10^3}
     \left\|\mv{C}_i-\mv{A}\mv{S}_i\right\|_2^2 \right\}$,
estimated from historical data,
$(\mv{S}_1, \mv{C}_1),\ldots, (\mv{S}_{1000}, \mv{C}_{1000}).$ The
total cost realized by deploying the decision
$\hat{\mv{\theta}}(\mv{C})$ is then given by
$[\mv{\hat{\theta}}(\mv{\hat{C}})]^\intercal\mv{C},$ where the
edge-traversal costs $\mv{C}$ satisfy \eqref{eqn:Grigas_cost}. A risk manager is then naturally interested in evaluating tail risk
probabilities such as in
$p_u:=P( [\mv{w}(\hat{\mv{C}})]^\intercal \mv{C} >u),$ that have to be
borne from deploying the routing decision
$\hat{\mv{\theta}}(\hat{\mv{C}}).$ Drawing 
samples 
for this purpose (drawn
independently of those used to estimate ${\mv{\hat A}}$), we first
generate the transformed tuple
$(\ZZ_i,\mv{\hat{D}}_i,\mv{\tilde{C}}_i)_{i\geq 1};$ here $\ZZ_i=\mv{T}(\mv{S}_i)$ is the IS vector,
$\mv{\tilde{\mv{C}}}_i = \mv{C}(\ZZ_i,\mv{\varepsilon}_i)$ and
$\mv{\hat{D}}_i=\mv{\hat{A}}\ZZ_i$ serve as the realized and predicted
costs, and $\hat{\mv{w}}(\hat{\mv{D}}_i)$ denotes the respective
shortest paths.  Our IS estimator for the probability
$p_u$ is then computed by averaging over 
$\{\mathcal{L}_i
\mv{I}([ \hat{\mv{w}}(\hat{\mv{D}}_i)]^\intercal\mv{\tilde{C}}_i )
\geq u): i \leq n\}$ as in  Step 3 of Algorithm~\ref{algo:RA_IS}.  We plot $\log V_u$ (denoting the sample log variance) against $\log p_u$ in
Figure~\ref{fig:Context}(b) for the same parameter choices {in Experiment 1.} As before, we observe that the plot is approximately a straight line with a slope of $2$.

Figure \ref{fig:Context}(c) presents additional details on cross-validation by plotting the logarithm of estimator variance,
$ \log V_{u},$ observed (in red) for different choices of
hyper-parameter $l$ considered. In Figure \ref{fig:Context}(c), the
level $u = 150$ is such that $p_u\approx 2\times 10^{-6}$. The high degree of
variance reduction (exceeding 99.99\%) in the interval
$l \in (15,35)$ demonstrates that the estimator variance is not
unduly sensitive to the choice of parameter $l.$ 
{Further, we observe that Step 2 concludes after $m_k = 2152$ samples} and the estimated sample variance closely approximate the true variance.  {Thus $m_k$ can be seen to constitute only a fraction of the total sample requirement: for eg., a total of $n = {16157}$ samples were required (on an average), in the estimation of $p_u \approx 10^{-4}$ with the stated precision.}
\begin{figure}[h!]
      \begin{center}
        \begin{subfigure}{0.32\textwidth}
          \includegraphics[width=1\textwidth]{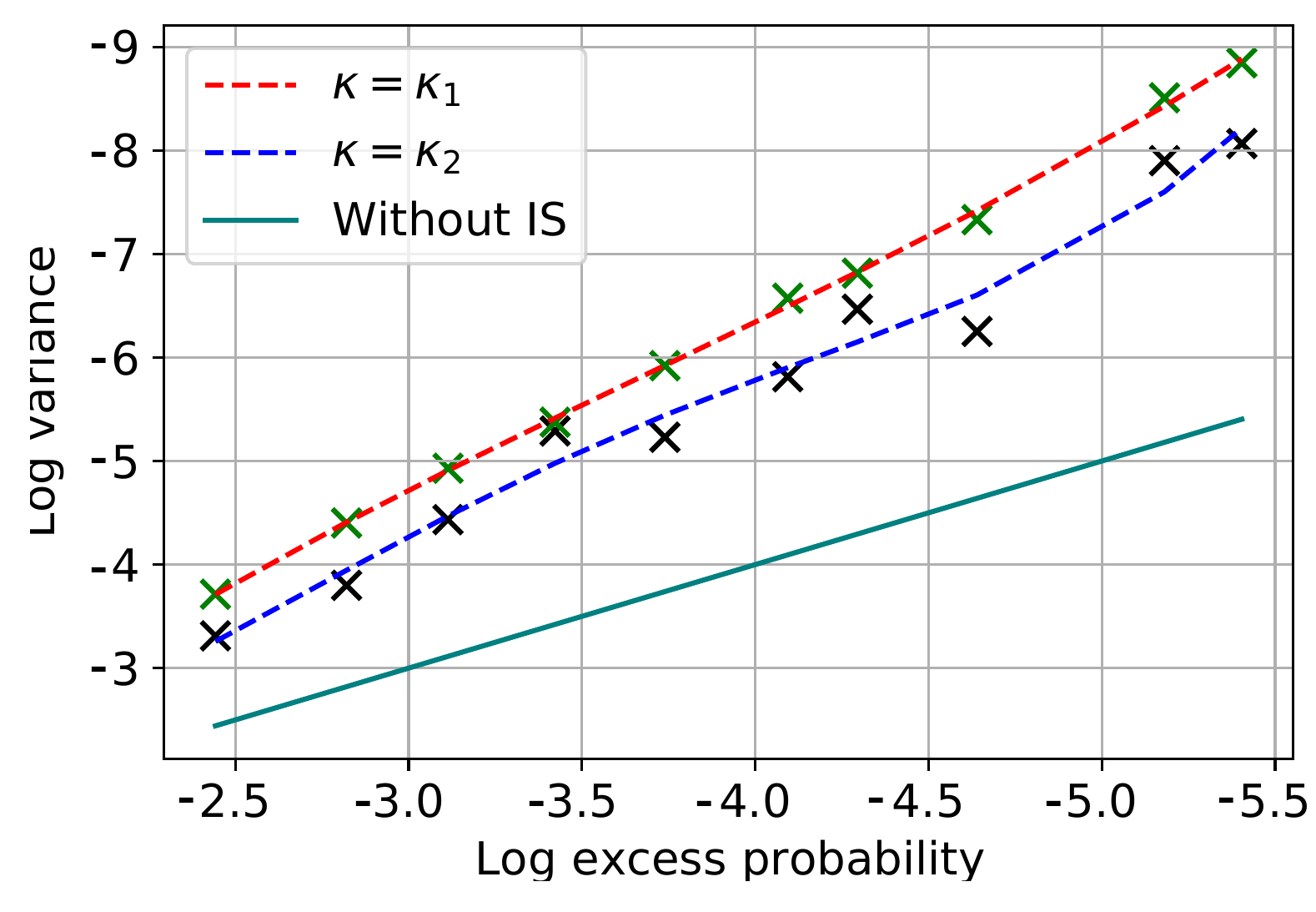}
          \caption{\small{Estimation of $P(L(\mv{C})\geq u)$}}
       \end{subfigure}
               \begin{subfigure}{0.33\textwidth}
          \includegraphics[width=0.97\textwidth]{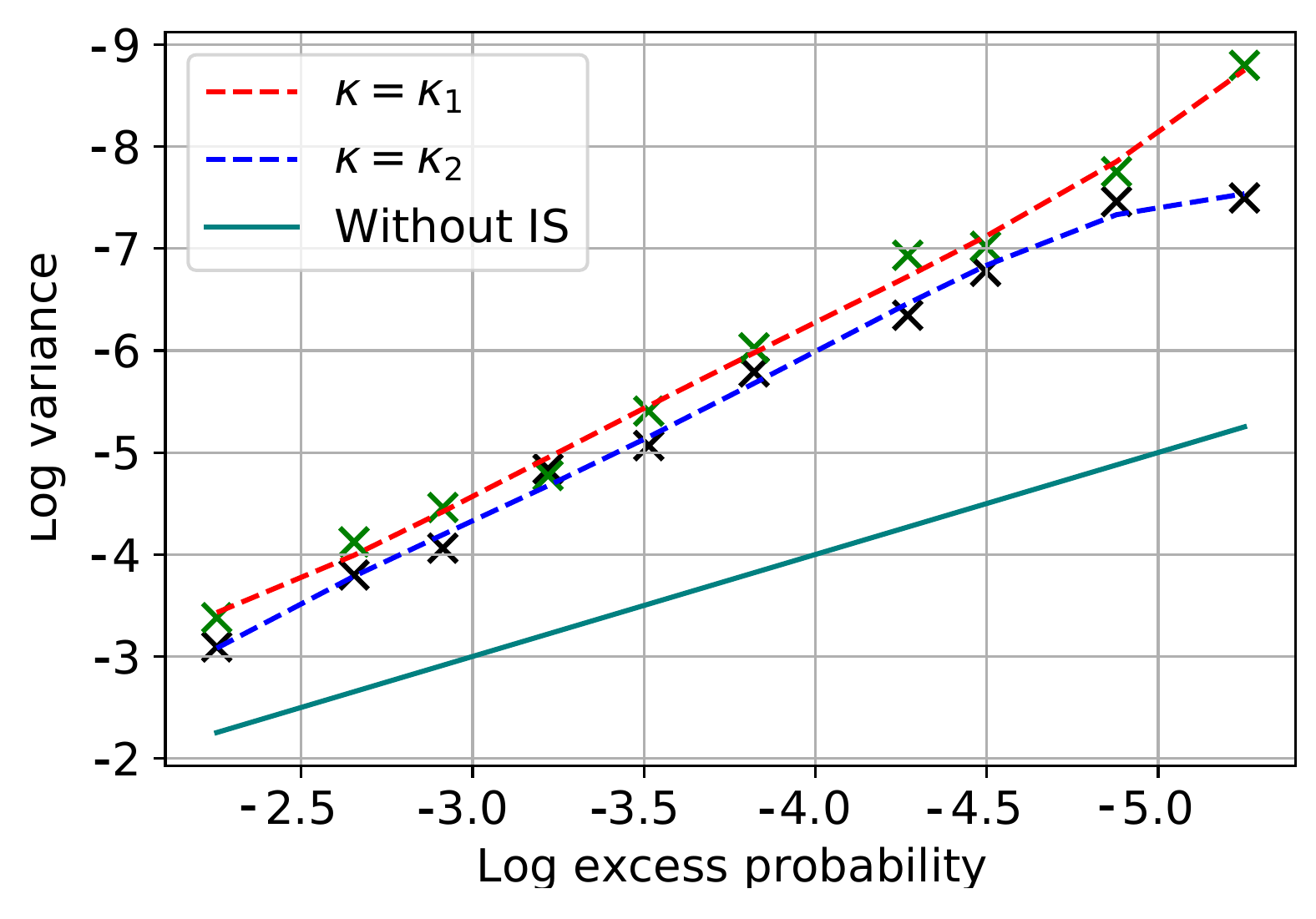}
          \caption{\small{Estimation of
              $P( [\mv{w}(\hat{\mv{C}})]^\intercal \mv{C} \geq u)$ }}
        \end{subfigure}
        \begin{subfigure}{0.31\textwidth}
          \includegraphics[width=1\textwidth]{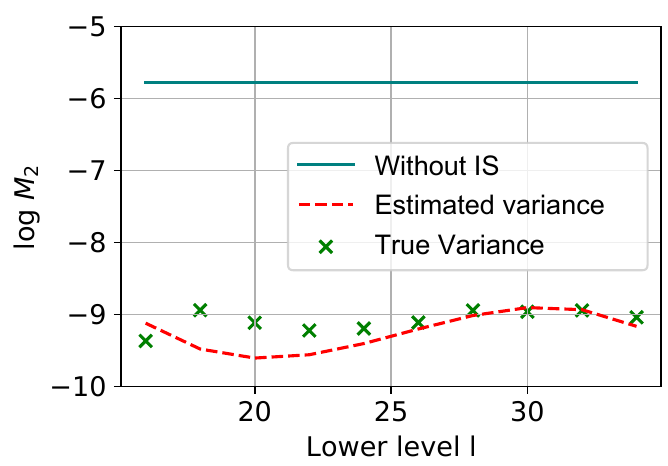}
          \caption{\small{$l$ vs second-moment $\hat{M}_2(l)$ }}
        \end{subfigure}
      \end{center}
      \caption{Variance of the proposed IS estimator, illustrated via
        a log-log plot, for the contextual shortest path problem.
        Lines (solid and dashed) represent a polynomial fit to the
        variance values marked via crosses.}
      \label{fig:Context}
    \end{figure}
\subsection{Illustration with the portfolio credit risk model}
\label{numexp-pcr}
For the portfolio credit risk model considered in Section
\ref{sec:CR}, we take a pool of $m = 3000$ loans of a single type,
each with an exposure $e_i = 1.$ As in Section \ref{num:SPP}, the
common factors $\XX \in \mathbb{R}^5$ are taken to possess standard
Weibull marginal distributions { with shape parameter $0.8$.} The dependence is informed by a
Gaussian copula whose non-zero off-diagonal entries of the correlation
matrix $R$ are taken to be $[R]_{i,j} = 0.2,$ for
$\vert j - i \vert = 1.$ We consider cases where the conditional
default probabilities are given by (a) the logit model in
\eqref{loan-def-prob} and (b) the discrete default intensity of the
form, $P(Y_i=1\, \vert \, \XX) = 1-\exp(-\e^{W(\XX) - \gamma}).$ The
function $W(\xx)= \mv{1}^{\intercal}(\mv{A}\xx -\mv{b})^{+}$ is
informed by a ReLU neural network with $1$ hidden layer with weights
${A}_{i,j} = 1/5$ for all $(i,j)$ and $\mv{b} = \mv{0}.$

Taking the loss threshold $q = 0.2$ in \eqref{defn:Em}, we aim to
estimate the probability of excess default loss
$p_\gamma := \Prob (\mathcal{E}_m);$ the parameter $\gamma$ in the
conditional default probabilities dictate the rarity of loan defaults
and is varied in the experiments so that the respective $p_\gamma$
varies from $10^{-2}$ to $10^{-5}.$ In Figure \ref{fig:PCRLDP} below,
we report $\log V_\gamma$ against $ \log p_\gamma,$ where $V_\gamma$
is the sample variance of the IS estimator.
As in Section \ref{num:SPP}, the
ratio between the logarithm of variance of IS estimator and that of
the naive estimator (without IS) ranges over the interval $(1.6,1.9)$,
which points to the IS estimator possessing negligible variance
compared to that of sample averaging without IS. 
\begin{figure}[h!]
      \begin{center}
        \begin{subfigure}{0.4\textwidth}
          \includegraphics[height = 0.7\textwidth, width=0.95\textwidth]{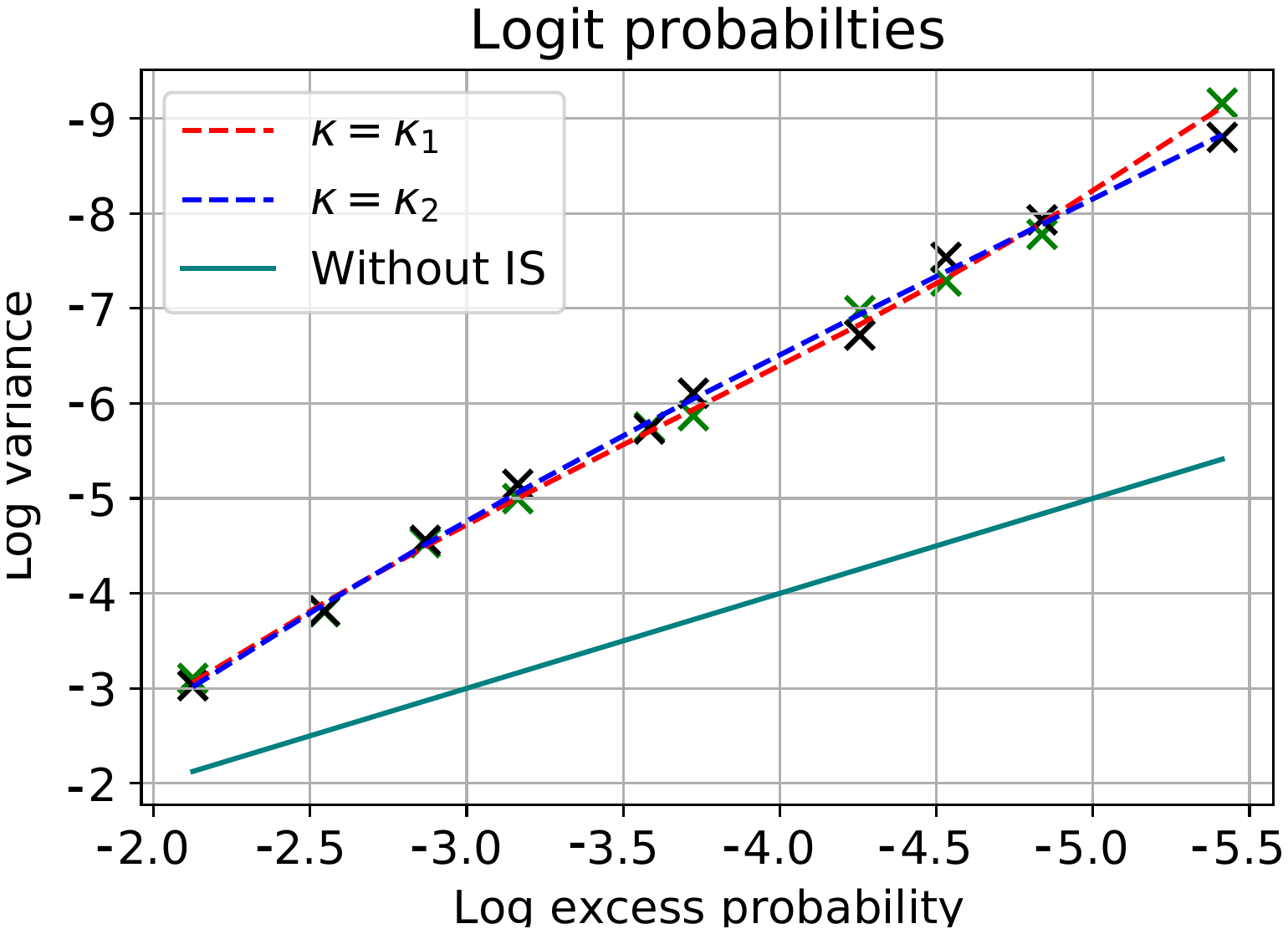}
          \caption{\small{Logit loan default probabilities}}
       \end{subfigure}
               \begin{subfigure}{0.4\textwidth}
          \includegraphics[height=0.7\textwidth,width=0.95\textwidth]{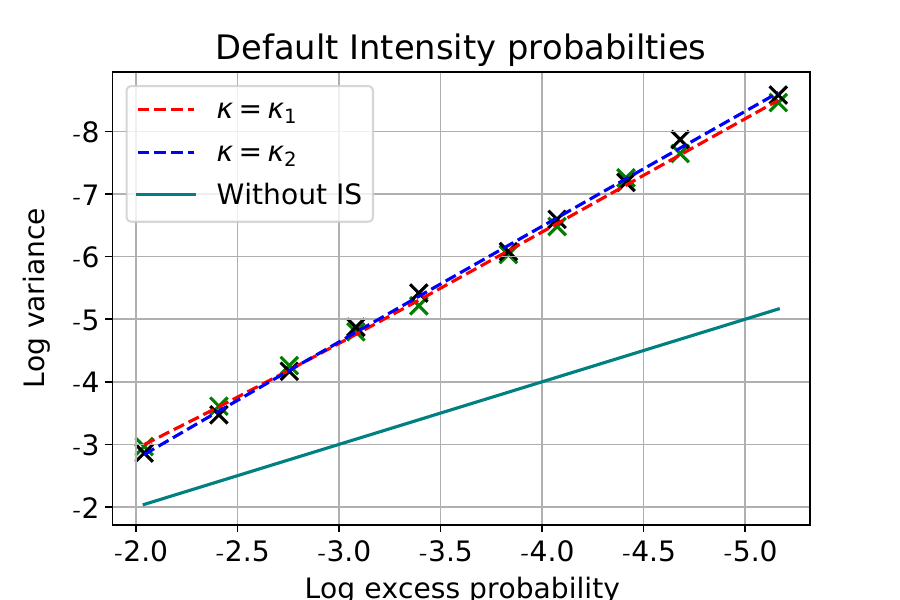}
          \caption{\small{Intensity based default probabilities}}
        \end{subfigure}
      \end{center}
      \caption{{Plots displaying the logarithm of sample
          variance, $ \log V_\gamma ,$ of the IS estimator
          versus $ \log p_\gamma,$ where $p_\gamma$ is
          excess default loss probability estimated.}}
      \label{fig:PCRLDP}
    \end{figure}
{

\subsection{Comparison with Cross-Entropy}\label{num:CE}
In order to benchmark the number of samples and loss evaluations required by Algorithm~\ref{algo:RA_IS}  against a state-of-the-art adaptive IS algorithm, we compare its performance with that of cross-entropy. To this end, we consider evaluation of exceedence probabilities of (i) the PERT network loss from \cite{rubinstein2013cross} and (ii) neural network loss $W(\cdot)$ defined in Section~\ref{numexp-pcr}. We assume that $\XX$ has independent and exponential marginals. For cross-entropy, in line with the approach in \cite{rubinstein2013cross}, we search for an IS distribution from among distributions with independent exponential marginals. For evaluation of probabilities  using the proposed method, we use Algorithm~\ref{algo:RA_IS}. To facilitate a comparison of the two methods, we plot in Figure~\ref{fig:CE_comp} the total number of samples required (across {30} independent experiments) to carry out the estimation to within a relative precision of $\varepsilon=0.05$ with confidence $\alpha = 0.05$ using the respective method. Observe that in either setting, at a probability of $10^{-5}$, the sample requirement for the proposed IS is smaller by a factor exceeding $10^3.$ 
Further, cross-entropy method takes just under $10^8$ loss evaluations to cross-validate and compute the probabilities, while for the same task IS requires only about $10^5$.


\begin{figure}[h!]
      \begin{center}
        \begin{subfigure}{0.4\textwidth}
          \includegraphics[width=\textwidth]{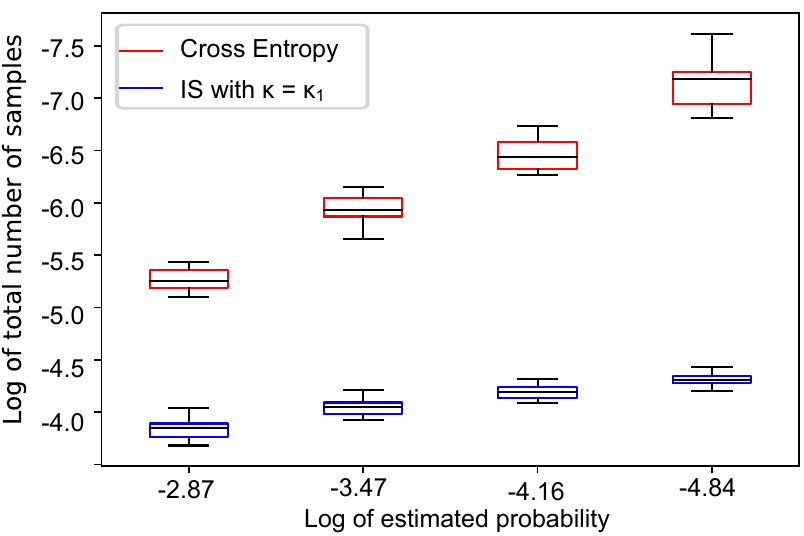}
          \caption{\centering{PERT Network}}
       \end{subfigure}
               \begin{subfigure}{0.4\textwidth}
          \includegraphics[width=\textwidth]{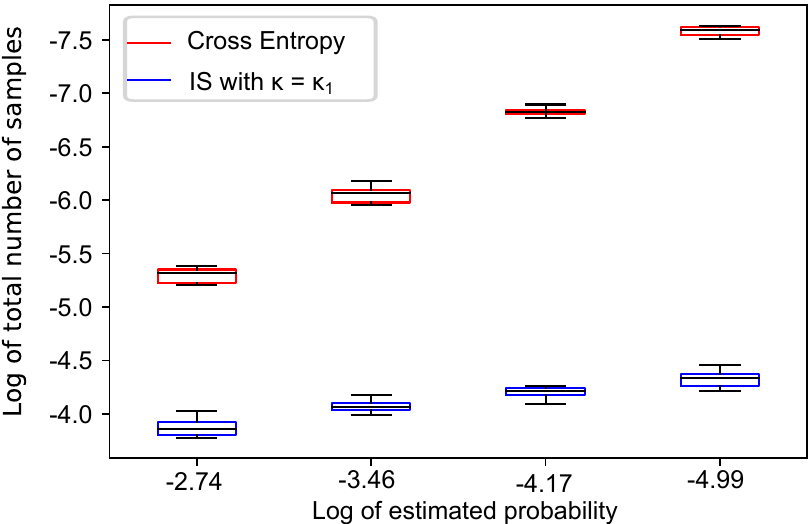}
         \caption{\centering {Neural Network}}
        \end{subfigure}
        \caption{{Box-plots comparing sample requirements of cross-entropy method (in red) with those of the proposed sampler (in blue): Cross-entropy expends about $10^3$ times more effort than the proposed sampler.}} \label{fig:CE_comp} 
      \end{center}         
    \end{figure}
\noindent
\textbf{Acknowledgements.} Support from Singapore Ministry of
Education's AcRF Tier 2 grant MOE2019-T2-2-163 is gratefully
acknowledged.{We thank the anonymous reviewers and the Associate Editor for the insightful  suggestions which have helped improve the paper in various respects.}

    
\bibliographystyle{authordate1}
\bibliography{references}




\newpage
\appendix
\noindent The appendix is organized as
follows: In Appendix \ref{sec:proofs-main}, we present the proofs of the main
results Theorems~\ref{thm:LDP} -
\ref{thm:Var-Red}. Appendix \ref{sec:app-eg:tails} serves to illustrate the
wide applicability of the tail modeling framework with examples and
sufficient conditions for Assumptions~\ref{assump-marginals} -
\ref{assump:joint-Y}. Proofs of the results relating to portfolio credit risk applications, namely Proposition \ref{prop:IS-NN} and Theorem
\ref{thm:credit-risk-asymp}, are presented in
Appendix~\ref{proofs:appl}. Proofs of technical results (Lemma
~\ref{lem:Y-exp-marginals} - \ref{lem:properties-I},
Proposition~\ref{prop:Self-Structuring} and Theorems
\ref{thm:Tail-asymp-HT} - \ref{thm:Var-Red-HT}) are given in
Appendix \ref{sec:tech-proofs}. Appendix {\ref{sec:verify-alt} presents the proof of efficiency when $\mv{\kappa}=\mv{\kappa}_2$ is used in Algorithm~\ref{algo:IS}.  Appendix \ref{sec:Verify_Num} outlines the verification of Assumption~\ref{assume:V}(b). Appendix \ref{sec:opt} explores how the IS Algorithm~\ref{algo:IS} with the choice $\mv{\kappa} = \mv{\kappa}_1$ is well-suited for use in stochastic optimisation.}
\section{Proofs of main results}
\label{sec:proofs-main}
The following definitions and notational convention are used in the
proofs: For $r > 0$ and $\xx \in \Real^d,$ let
$B_r(\xx) = \{\yy \in \Real^d: \Vert \yy - \xx \Vert \leq r\}$ denote
the metric ball of radius $r$ around $\xx.$ Unless specified
explicitly, $\Vert \xx \Vert = \max_{i = 1,\ldots,d} \vert x_i \vert$
denotes the $\ell_\infty$-norm. Let
$\mathcal{S}^{d-1} = \{\xx \in \Real^d: \Vert \xx \Vert = 1 \}$ denote
the unit sphere and $\Real^d_{++} = \{\xx \in \Real^d: \xx > 0\}$
denote the interior of the positive orthant. For $M > 0,$ let
$B_M := \{\xx \in \Real^d_+: \Vert \xx \Vert \leq M\}.$ For
$A\subseteq\R^d$, let $\mathrm{cl}(A)$ denote its closure,
$\mathrm{int}(A)$ denote its interior, 
\begin{align*}
  \chi_{_A}(\xx) :=
  \begin{cases}
    0, \quad& \text{ if } \xx \in A\\
    +\infty, & \text{ otherwise}, 
  \end{cases}
               \qquad \text{ and } \qquad 
               [A]^{1+r} := \{\xx+\yy \in \Real^d_{+}:\xx\in A,
               \|\yy\|\leq r \},
\end{align*}
respectively denote the characteristic function and 
the set of points in $\Real^d_{+}$ lying within a distance
$r \in (0,\infty)$ from $A$.  For $f:\Real^d \rightarrow \Real,$
$\alpha \in \Real,$ $M > 0,$ let
\begin{align*}
  \Lev_{\alpha}^+(f) = \{\xx \in \Real^d_+: f(\xx) \geq \alpha\} \quad
  \text{ and }  \quad \Xi_{\alpha,M}(f) = \Lev_{\alpha}^+(f) \cap B_M \cap  \Real^d_{++}
\end{align*}
denote the super-level sets of $f$ restricted, respectively, to the
positive orthant and to the bounded subset $B_M \cap \Real^d_{++}.$
Recall that $\YY = \mv{\Lambda}(\XX),$ the component-wise inverse
$\mv{q}(t) =
(\Lambda_1^{\leftarrow}(t),\ldots,\Lambda_d^{\leftarrow}(t)),$ and
$\mv{q}^\ast = \lim_{t \rightarrow \infty} [\mv{q}(t)/q_\infty(t)],$
if the limit exists.  Throughout the proofs, we suppose $u>x_0$ as specified in Assumption \ref{assume:V}(a). Define
$L_{u}: \Real^d_{++} \rightarrow \Real$ and
$\fLD: \Real^d_+ \rightarrow \Real$ as,
  \begin{subequations}
    \begin{equation}
      L_{u}(\xx) := u^{-1}{L(\mv{q}(t(u)\xx))}, \qquad \fLD(\yy)
      :=L^*(\qq^* \yy^{\aalpha}), \text{ where }
      \label{eqn:fLD-def} 
    \end{equation}
    \begin{equation}
          t(u) :=
          \Lambda_{\min}\big( u^{1/\rho} \big) \quad \text{ and }
          \quad q_{\infty}(t) := \Vert \qq(t)\Vert_{\infty},
\label{eqn:fLD-def-b} 
    \end{equation}
  \end{subequations}
  Write $\textrm{Diag}(\mv{a})$ for a diagonal matrix with diagonal
  entries $a_1,\ldots, a_d$ and $\text{sgn}(\xx)$ for the vector of
  signs of $\xx$.  To avoid clutter in the expressions, define
\begin{align*}
  \YYu := {t(u)}^{-1}{\YY}, \quad  
  \mv{\psi}_u  := \mv{\Lambda} \circ \mv{T}^{-1} \circ \mv{q},
  \quad \text{ and } \quad c_{\rho}(u) := (l/u)^{1/\rho}.
\end{align*}
where the parameter $u$ in the symbol $\mv \psi_u$ is explicitly
indicated to remind the role of $u$ in $\mv{T}^{-1}.$ 


\noindent \textbf{Proof of Theorem~\ref{thm:LDP}.} A sufficient
condition (see \cite[Theorem 4.1.11]{Dembo}) to verify the existence
of LDP is to show that for all $\xx \in \Real^d_{++},$
\begin{equation}\label{eqn:LDP-Equivalent}
-I(\xx) = \inf_{\delta>0} \limsup_{t\to\infty} \frac{1}{t}\log\Prob\left( \frac{\YY}{t} \in B_{\delta}(\xx)\right) = \inf_{\delta>0} \liminf_{t\to\infty} \frac{1}{t}\log\Prob\left( \frac{\YY}{t} \in B_{\delta}(\xx)\right).
\end{equation}
Fix any $\varepsilon, M \in (0,\infty)$ and $\xx \in (0,M)^d.$ Since
$f_{\YY}(\yy) = p(\yy) \exp(-\varphi(\yy)),$
\begin{align*}
  \Prob(\YY/t \in B_{\delta}(\xx))
  = \int_{\yy/t\in B_{\delta}(\xx)}p(\yy) \exp(-\varphi(\yy))d\yy
    = t^d \int_{\zz\in B_{\delta} (\xx)} p(t\zz) \e^{-\varphi(t\zz)} d\zz.
\end{align*}
Recall that the continuous convergences in
Assumption~\ref{assump:joint-Y} imply the following uniform
convergences over compact sets not containing the origin (see
\cite[Theorem 7.14]{rockafellar2009variational}):
\begin{equation}
  \label{eqn:UC-expo}
n^{-1}\varphi(n\xx) \xrightarrow{n\to\infty} I(\xx) \text{ and }  n^{-1} \log p(n\xx) \xrightarrow{n\to\infty} 0.
\end{equation}
Due to this local uniform convergence and the continuity of $I$, there
exist $\delta_0,t_0 \in (0,\infty)$ such that,
\begin{align*}
  \left\vert \frac{\varphi(t\zz)}{t} - I(\xx)  \right\vert \leq  \left\vert \frac{\varphi(t\zz)}{t} - I(\zz)\right\vert + \left\vert I(\zz) - I(\xx)\right\vert \leq \varepsilon/2, \textrm{ for all } \zz\in B_{\delta}(\xx)
\end{align*}
and $\exp(-\varepsilon t/2) \leq p(t\zz)\leq \exp(\varepsilon t/2),$
whenever $t > t_0,$ $\delta < \delta_0$ and $B_{\delta_0}(\xx)$ does
not contain the origin.
Thus, given $\varepsilon, M$ and $\xx \in (0,M)^d,$ there exist
$\delta_0,t_0 \in (0,\infty)$ such that for all $t>t_0$ and
$\delta \in (0, \delta_0),$
\begin{equation}
  \label{eqn:Y-Bounds}
  \exp \left(-t(I(\xx) +\varepsilon)\right) \leq f_{\YY}(t\zz) \leq
  \exp\left({-t(I(\xx) -\varepsilon)}\right), \text{ uniformly over
    $\zz \in B_{\delta}(\xx);$}
\end{equation}
Then 
$t^d \mathrm{Vol}(B_{\delta}(\xx)) \exp({-t(I(\xx) + \epsilon)}) \leq
\Prob(t^{-1}\YY \in B_{\delta}(\xx)) \leq t^d
\mathrm{Vol}(B_{\delta}(\xx)) \exp({-t(I(\xx) -\epsilon)}).$
Since $\Prob(\YY/t \in B_{\delta}(\xx)) $ is increasing in $\delta$
and these bounds hold for any $\delta < \delta_0,$
\[ - I(\xx) - \epsilon \leq \inf_{\delta > 0}\liminf_{t\to\infty} \frac{1}{t} \log
  P(t^{-1}\YY \in B_{\delta}(\xx)) \leq \inf_{\delta > 0} \limsup_{t\to\infty} \frac{1}{t} \log
  P(\YY/t \in B_{\delta}(\xx)) \leq - I(\xx) +\epsilon. 
\]
Since the choices $\varepsilon, M \in (0,\infty)$ are arbitrary,
(\ref{eqn:LDP-Equivalent}) holds. $\hfill\Box$

\subsection{Proof of Theorem \ref{thm:Tail-asymp}}
For functions $f$ and $g$, let $f \land g$ (resp. $f\lor g$) denote
their point-wise minimum (resp. maximum).
\begin{lemma}
  Under Assumption \ref{assump-marginals}, we have
  $q_\infty(t(u)) = u^{1/\rho}.$ Therefore when Assumption
  \ref{assume:V} additionally holds, the events $\{L(\XX) \geq u\}$ and
  $\{ \YYu \in \Lev^+_{1}(L_u)\}$ coincide.
  \label{lem:qtequ}
\end{lemma}
\textit{Proof.}
Consider increasing real valued functions $f_1,f_2$. By the definition
of left-continuous inverses,
$(f_1\lor f_2)^{\leftarrow} (y)
 =\inf\{u: f_1(u) \geq y\} \land \inf\{u: f_2(u) \geq y\} = f_1^{\leftarrow}(y) \land f_2^{\leftarrow}(y).$
From induction,
$(\lor_{i=1}^d f_i)^{\leftarrow} = \land_{i=1}^d f_i^{\leftarrow}$,
given increasing functions $f_1,\ldots, f_d$. Since
$\{\Lambda_i: i =1,\ldots,d\}$ are continuous, 
$q_i^\leftarrow = \Lambda_i$ (see \cite[Excercise 1.1
(a)]{deHaanFerreira2010}). Consequently,
\begin{align}
  q_\infty^\leftarrow(t) = \min_{i=1,…d} q_i^\leftarrow (t) =
  \Lambda_{\min}(t).
  \label{eq:qinfty-eq-Lamdamin}
\end{align}
Then $q_{\infty} (\Lambda_{\min}(x)) = x$ for all $x \in (x_0,\infty)$
due to the strict monotonicity in Assumption \ref{assump-marginals}.
Since $\qq = \mv{\Lambda}^{\leftarrow}$ is injective,
$\{\YYu\in A\} = \{\XX\in \mathbf{q}(t(u)A)\},$ for any measurable
$A$.  With $L_u(\xx) := u^{-1}L(\mv{q}(t(u)\xx)),$
\begin{align*}
  \{\qq(t(u)\yy):\yy\in \Lev_1^+(L_u)\}=\{\qq(t(u)\yy):L(\qq(t(u)\yy))\geq u\} 
  = \{\xx:L(\xx)> u\} \cap \supp(\XX). \qquad\qquad{\Box}
\end{align*}


\begin{lemma}
  If Assumptions \ref{assume:V} and \ref{assump-marginals} hold and
  the limit $\mv{q}^\ast$ exists, the sequence of functions
  $\{L_u: u > 0\}$ converge continuously to $\fLD$ on $\Real^d_{++}.$
  Consequently, for any
  $\alpha, \varepsilon, M,K > 0,$ there exists $u_0$ large
  enough such that for all $u > u_0,$
  \begin{align*}
    \Xi_{\alpha,M} (L_u) \subseteq
    \big[\Xi_{\alpha,M}(\fLD)\big]^{1+\varepsilon/2}  \quad \text{ and }
    \quad \Xi_{\alpha + \varepsilon,M} (\fLD) \subseteq [\Xi_{\alpha +
    \varepsilon/2,M}(L_u)]^{1 + \varepsilon/K}
  \end{align*}
\label{lem:cont-conv-set-incl}
\end{lemma}
\textit{Proof.}
\textbf{1) } Consider any sequences
$\{u_n\} \subset \Real_+,\{\xx_n\} \subset \Real_+^d$ satisfying
$u_n \uparrow \infty, \xx_n \to \xx >
\mv{0}.$
We rewrite,
\[
  L_{u_n}(\xx_n) = u_n^{-1}{L\left(\qq(t(u_n)\xx_{n})/\qq(t(u_n)\mv{1})
      \cdot \hat{\qq}(t(u_n)) u_n^{1/\rho}\right)}.
\]
Consider any $\delta > 0$ such that $B_\delta(\xx)$ does not include
$\mv{0}.$ As $\Lambda_i \in \RV(\alpha_i),$ we have
$q_i \in \RV(1/\alpha_i)$ (see Part 9 of \cite[Proposition
B.1.9]{deHaanFerreira2010}). Due to the uniform convergence
$\lim_{t \to \infty }\qq(t\yy)/\qq(t\mv{1}) = \yy^{\aalpha}$ over
$\yy \in B_\delta(\xx)$ (see \cite[Proposition
B.1.4]{deHaanFerreira2010}), 
\[
\qq(t(u_n)\xx_n)/\qq(t(u_n)\mv{1})  \to \xx^{\aalpha}, \text{ and } \hat{\qq}(t(u_n)) \to {\qq^{*}},
\]
Applying triangle inequality,
$\mv{p}_n = \qq(t(u)\xx_u)/\qq(t(u)\mv{1}) \hat{\qq}(t(u)) \to
\xx^{\aalpha}{\qq}^{*}$.  Continuous convergence
\[L_{u_n}(\xx_n) := u_n^{-1}L(u_n^{1/\rho_n} \mv{p}_n) \rightarrow
  L^\ast({\qq}^{*}\xx^{\aalpha}) =: \fLD(\xx)\] follows from
Assumption~\ref{assume:V}(b).

\noindent \textbf{2) } We next prove the claims on the set inclusions
using the notions of $\limsup,\liminf$ of a sequence of sets defined
in the Kuratowski sense (see \cite[Chapter
1]{rockafellar2009variational}). Taking the ambient space
$\mathcal{X} =\R^d_{++}$ in \cite[Theorem
3.1]{beer1992characterization}, we obtain 
\begin{align}
\limsup_u \Xi_{\beta}(L_{u}) \subseteq \Xi_{\beta}(\fLD), \quad 
\liminf_u \Xi_{\beta_u}(L_{u}) \supseteq \Xi_{\beta}(\fLD) \text{ for some $\beta_u\nearrow\beta$} \label{eqn:set-sup-inf},
\end{align}
as a consequence of above verified
$L_{u_n}(\xx_n) \rightarrow \fLD(\xx).$ Refer the construction in the
proof of \cite[Theorem 3.1]{beer1992characterization} for using fixed
level $\beta$ in the first set inclusion and considering increasing
sequence $\beta_u$ in the second set inclusion in
\eqref{eqn:set-sup-inf}. Now, setting $\beta=\alpha$ in
\eqref{eqn:set-sup-inf}, using the equivalence between $(v)_b$ and
$(vii)_b$ in \cite[Theorem 2.2]{salinetti1981convergence},
\[
\Xi_{\alpha,M} (L_u) \subseteq
    \big[\Xi_{\alpha,M}(\fLD)\big]^{1+\varepsilon/2} \cap \R^d_{++} \subseteq \big[\Xi_{\alpha,M}(\fLD)\big]^{1+\varepsilon/2}.
\] 
Set $\beta=\alpha+\varepsilon$, and let
$\beta_u \nearrow \alpha+\varepsilon$ be selected as in
\eqref{eqn:set-sup-inf}. Therefore, for all large enough $u$,
$\Xi_{\beta_u,M}(L_{u}) \subseteq
\Xi_{\alpha+\varepsilon/2,M}(L_{u})$. From \eqref{eqn:set-sup-inf},
$ \Xi_{\alpha+\varepsilon,M}(\fLD) \subseteq \liminf_u
\Xi_{\beta_u}(L_{u}) \subseteq \liminf_{u}
\Xi_{\alpha+\varepsilon/2,M}(L_{u}). $ Further, from the equivalence
between $(vii)_a$ and $(v)_a$ in \cite[Theorem
2.2]{salinetti1981convergence}, 
\[
\Xi_{\alpha+\varepsilon,M}(\fLD)  \subseteq  [\Xi_{\alpha+\varepsilon/2,M}(L_{u})]^{1+\varepsilon/K} \cap \R^d_{++}\subseteq  [\Xi_{\alpha+\varepsilon/2,M}(L_{u})]^{1+\varepsilon/K}. \qquad\qquad {\Box}
\]

\begin{corollary}
  Suppose that Assumptions \ref{assume:V} and \ref{assump-marginals}
  hold and the limit $\mv{q}^\ast$ exists. Then for any
  $\alpha, \varepsilon, M > 0,$ there exists $u_0$ large enough such
  that for all $u > u_0,$
  \begin{align}
    \Lev_\alpha^+(L_u) \cap B_M \subseteq
    \big[\Xi_{\alpha,M}(\fLD)\big]^{1+\varepsilon} 
                                  \quad \text{ and } 
    \quad 
    \Xi_{\alpha+\varepsilon,M}(\fLD) \subseteq
     \Xi_{\alpha,M}(L_u).
   \label{set-incl-cor}                                            
  \end{align}
  Consequently,
  $\liminf_{n \rightarrow \infty}\chi_{\Lev_\alpha^+(L_{u_n})}(\xx_n)
  \geq \chi_{\Lev_\alpha^+(\fLD)}(\xx),$ for any
  $\xx_n \rightarrow \xx$ and $u_n \rightarrow \infty.$
\label{cor:set-inclusions}
\end{corollary}
\textit{Proof.}
Notice
that for $\varepsilon>0$,
$\Lev_{\alpha}^+(L_u) \cap B_M \subseteq
[\Xi_{\alpha,M}(L_u)]^{1+\varepsilon/2}$, as a consequence of the
definitions at the beginning of Section \ref{sec:proofs-main}. The
first set inclusion now follows from the definition of the
$[A]^{(1+r)}$ for a set $A\subset \R_{+}^d$.  For the second set
inclusion, observe that for $K>0$ (to be chosen imminently),
\[
\Xi_{\alpha+\varepsilon,M}(\fLD)  \subseteq [\Xi_{\alpha+\varepsilon/2,M}(L_{u})]^{1+\varepsilon/K} \cap \R^d_{++}
\]
 for all large enough $u$. Further, for any $\xx$, $\yy$,
\[
|L_u(\xx) - L_u(\yy)| \leq |L_u(\yy) - \fLD(\yy)| + |\fLD(\yy)  -  \fLD(\xx)| + |\fLD(\xx) - L_u(\xx)|. 
\]
Recall that $L_u\to\fLD$ uniformly on $\R_{++}^d \cap B_{2M}$ (see for example,
\cite[Theorem 7.14]{rockafellar2009variational}).
Hence 
the first and third terms above may be made less than $\varepsilon/6$
for a large enough $u$ for any choices of
$\xx,\yy\in [\Xi_{\alpha+\varepsilon/2,M}(L_{u})]^{1+\varepsilon/K}
\cap \R^d_{++}.$ Next, $\fLD$ is uniformly continuous over $B_{2M}$. Therefore, there exists a $\kappa>0$, such that for $\xx,\yy \in B_{2M}$, whenever $\|\xx-\yy\| \leq \kappa $,
$|\fLD(\xx) -\fLD(\yy)| \leq \varepsilon/6$. Consider any $K \geq \varepsilon/\kappa.$ For
any
$\yy\in [\Xi_{\alpha+\varepsilon/2,M}(L_{u})]^{1+\varepsilon/K} \cap
\R_{++}^d,$ there exists $\xx\in \Xi_{\alpha+\varepsilon/2,M}(L_u)$
such that $\|\xx-\yy\|\leq \kappa,$  and consequently,
\[
  |L_u(\xx) - L_u(\yy)| \leq \varepsilon/6 + \varepsilon/6 +\varepsilon/6\leq \varepsilon/2.
\]
Therefore, whenever
$\yy\in [\Xi_{\alpha+\varepsilon/2,M}(L_{u})]^{1+\varepsilon/K} \cap
\R_{++}^d$, $L_u(\yy) \geq \alpha$. Hence
$\Xi_{\alpha+\varepsilon,M}(\fLD) \subseteq \Xi_{\alpha,M}(L_u).$\\
To verify the conclusion on characteristic functions, we proceed as
follows: The bound in the statement is immediate if
$\xx \in \Lev_\alpha^+(\fLD),$ as
$\chi_{\Lev_\alpha^+(\fLD)}(\xx) = 0$ in that case. Consider the case
where $\xx_n \rightarrow \xx \notin \Lev_\alpha^+(\fLD).$ Then for a
suitably small $\delta > 0,$
$B_\delta(\xx) \cap [\Lev_\alpha^+(\fLD)]^{1+\delta} = \varnothing,$
because of the continuity of $\fLD$ and $\Lev_\alpha^+(\fLD)$ being a
closed set. Fix $M > \Vert \xx \Vert + \delta.$ With
$u_n \rightarrow \infty,$ we have
\begin{align*}
  \Lev_\alpha^+ (L_{u_n}) \cap B_M \subseteq \big[\Xi_{\alpha,M}(\fLD)
  \big]^{1+\delta}  \subseteq \big[ \Lev_\alpha^+(\fLD) \cap B_M
  \big]^{1+\delta}, \quad n > n_0
\end{align*}
for sufficiently large $n_0,$ due to the first inclusion in
(\ref{set-incl-cor}). For $n_1$ chosen large enough to ensure
$\{\xx_n\}_{n > n_1} \subseteq B_\delta(\xx),$ we have
$\{\xx_n: n > n_1\} \cap [\Lev_\alpha^+(\fLD)]^{1+\delta} =
\varnothing.$ Consequently,
$\xx_n \notin \Lev_\alpha^+ (L_{u_n}) \cap B_M,$ for all
$n > n_2 := \max\{n_0,n_1\}.$ Since $M > \Vert \xx \Vert + \delta$ and
$\Vert \xx_n \Vert \leq \Vert \xx \Vert + \delta$ for $n > n_2,$
$\xx_n \notin \Lev_\alpha^+ (L_{u_n}) \cap (\mathbb{R}^d \setminus
B_M)$ either. Therefore $\xx_n \notin \Lev_\alpha^+ (L_{u_n})$ and
$\chi_{\Lev_\alpha^+ (L_{u_n})}(\xx_n) = \infty$ for $n > n_2.$ \hfill$\Box$


\begin{lemma}
  Suppose that $f,g: \Real^d_+ \rightarrow \Real$ are continuous. For
  any $\alpha, \delta,M$ positive, 
  \begin{align*}
    \inf_{\xx\in [\Xi_{\alpha,M}(f)]^{1+\varepsilon}} g(\xx)
    &\geq \inf_{\xx \in \Lev^+_{\alpha}(f) \cap B_M} g(\xx) - \delta
      \quad \text{ and }\\
    \inf_{\xx\in \mathrm{int}(\Xi_{\alpha  + \varepsilon,M}(f)) } g(\xx)
    &\leq \inf_{\xx \in \Lev_{\alpha}^+(f) \cap B_M} g(\xx) + \delta,
  \end{align*}
  for all $\varepsilon$ suitably small. 
\label{lem:cont-I}
\end{lemma}

\noindent \textit{Proof}. 
Observe that the sequence of sets
$\mathcal{X}_n:=\{[\Xi_{\alpha,M}(f)]^{1+1/n}\}_{n\geq 1}$ are
uniformly bounded in $n$ and $\cup_{i\geq 1} \mathcal{X}_i$ is
relatively compact. Further,
$\mathcal{X}_n\searrow \Lev^+_\alpha(f) \cap B_M$ in the Kuratowski
sense. Therefore, from \cite[Theorem~2.2
(iii)]{langen1981convergence}, for all small enough
$\epsilon$,
\[\inf_{\xx\in [\Xi_{\alpha,M}(f)]^{1+\epsilon}} g(\xx)\geq\inf_{\xx
    \in \Lev_{\alpha}^+\cap B_M} g(\xx) - \delta.\] In a similar
spirit, due to the continuity of $f(\cdot)$, one has
$\Xi_{\alpha+1/n,M}(f) \nearrow \Xi_{\alpha,M}(f).$ Then, upon an
application of \cite[Theorem~2.2 (iii)]{langen1981convergence},
\[
\inf_{\xx\in \Xi_{\alpha  + \varepsilon,M}(f) } g(\xx)
    \leq \inf_{\xx \in \Xi_{\alpha,M}(f)} g(\xx) + \delta.
  \]
  The statement then follows from the continuity of $g(\cdot).$
  \hfill$\Box$


\noindent
\textbf{Proof of Theorem \ref{thm:Tail-asymp}.} 
Fix any $\delta, M > 0.$ Observe that $I(\cdot)$ and $L^\ast(\cdot)$
are continuous. As a consequence of Lemma \ref{lem:cont-I}, there
exists $\varepsilon > 0$ suitably small such that
  \begin{subequations}
      \begin{equation}
        \inf_{\pp\in [\Xi_{1,M}(\fLD)]^{1+\epsilon}} I(\pp) \geq \inf_{\pp
        \in \Xi_{1,M}(\fLD)} I(\pp) - \delta, \text{ and }
        \label{ldpf-epsilon-choice} 
      \end{equation}
      \begin{equation}
        \inf_{\pp\in \mathrm{int}(\Xi_{1  + \varepsilon,M}(\fLD)) } I(\pp)
        \leq \inf_{\pp \in \Lev_{1}^+(\fLD) \cap B_M} I(\pp) + \delta.
                \label{ldpf-epsilon-choice-lb} 
      \end{equation}
  \end{subequations}
  \textbf{Large deviations upper bound.}  Due to Lemma
  \ref{lem:qtequ},
  \begin{align}
    P\big(L(\XX) \geq u \big) 
    &\leq  P\big( \YYu \in \Lev^+_{1}(L_u) \cap B_M \big) + 
      P\big( \YYu \in B_M^c \big).
  \label{ld-ub-inter1}
  \end{align}
 From Corollary
\ref{cor:set-inclusions},
$P\big( \YYu \in \Lev^+_{1}(L_u) \cap B_M \big) \leq P\big( \YYu \in
\big[\Xi_{1,M}(\fLD)\big]^{1+\varepsilon}\big),$
for all $u$ sufficiently large. Since the expansion set $[A]^{1+r}$ is
closed for any $A \subseteq \Real^d_+,$ the set
$[\Xi_{\alpha,M}(\fLD)]^{1+\varepsilon}$ is closed. Therefore,
\begin{align*}
  \limsup_{u \rightarrow \infty} \frac{1}{t(u)} \log P\big( \YYu \in \Lev^+_{1}(L_u) \cap B_M \big) 
  &\leq \ \limsup_{u \rightarrow \infty} \frac{1}{t(u)} \log P\big(
    t(u)^{-1}\YY \in \big[\Xi_{1,M}(\fLD)\big]^{1+\varepsilon} \big)\\
  &\leq -\inf_{\pp \in  [\Xi_{1,M}(\fLD)]^{1+\varepsilon}}  I(\pp) \\
  &\leq -\inf_{\pp  \in \Lev^+_1(\fLD)\cap B_M} I(\pp) + \delta,  
\end{align*}
where the second inequality follows from the Tail LDP in Theorem
\ref{thm:LDP} and the third inequality is a consequence of the choice
of $\varepsilon$ satisfying (\ref{ldpf-epsilon-choice}).  Since
$\delta > 0$ is arbitrary,
\begin{align*}
  \limsup_{u \rightarrow \infty} \frac{1}{t(u)} \log P\big( \YYu \in \Lev^+_{1}(L_u) \cap B_M \big) 
  &\leq -\inf_{\pp  \in \Lev^+_1(\fLD)\cap B_M} I(\pp) \leq -I^\ast,
\end{align*}
where $I^\ast := \inf_{\pp \in \Lev^+_1(\fLD)} I(\pp).$ A similar
application of Theorem \ref{thm:LDP} results in
\begin{align}
  \limsup_{u \rightarrow \infty} \frac{1}{t(u)}
  \log P\left(\YYu \in B_M^c\right) \leq
  -\inf_{\pp \in \text{cl}(B_M^c)} I(\pp) = -M,
  \label{LDP:complement-set}
\end{align}
where the latter inequality follows from Lemma \ref{lem:properties-I}d
and the continuity of $I(\cdot).$ Combining these conclusions with
that in (\ref{ld-ub-inter1}) and \cite[Lemma 1.2.15]{Dembo},
\begin{align*}
  \limsup_{u \rightarrow \infty} \frac{1}{t(u)}
  \log P\left(L(\XX) \geq u \right) \leq
  -\min \left\{M, I^\ast \right\}.  
\end{align*}
Since $M$ can be made arbitrarily large, we have
 $ \limsup_{u \rightarrow \infty} {t(u)}^{-1}
  \log P\left(L(\XX) \geq u \right) \leq - I^\ast.$ \\  
  \textbf{Large deviations lower bound.} We have from Lemma
  \ref{lem:qtequ} and the definition of $\Xi_{\alpha,M}$ that
     $P(L(\XX )\geq u) \geq P(\YYu \in \Xi_{1,M} (L_u)).$
For $\varepsilon$ satisfying (\ref{ldpf-epsilon-choice-lb}), it
follows from the second set inclusion in
Corollary~\ref{cor:set-inclusions} that
  $P(L(\XX )\geq u)
  \geq P(\YYu \in \mathrm{int}\left(\Xi_{1+\varepsilon,M} (\fLD)\right)),$ 
for all sufficiently large enough $u.$ Then as an application of the
tail LDP in Theorem \ref{thm:LDP},
\begin{align*}
  \liminf_{u\to\infty} \frac{1}{t(u)} \log  P(L(\XX )\geq u)
  \geq -\inf_{ \pp \in \mathrm{int}\left(\Xi_{1+\varepsilon,M} (\fLD) \right)}
    I(\pp)
  \geq   -\inf_{\pp  \in \Lev^+_1(\fLD)\cap B_M} I(\pp) - \delta,
\end{align*}
where the latter inequality is a consequence of $\varepsilon$
satisfying (\ref{ldpf-epsilon-choice-lb}). Since $M,\delta$ are
arbitrary,
\begin{align*}
  \liminf_{u \rightarrow \infty} \frac{1}{t(u)}
  \log P\left(L(\XX) \geq u \right) \geq -\inf_{\pp  \in \Lev^+_1(\fLD)}
  I(\pp) = - I^\ast.    \qquad\qquad \Box 
\end{align*}
{
\begin{remark}
In the case where $L_u$ as defined in \eqref{eqn:fLD-def} is independent of $u$, it can be seen that for a fixed map $\mv{Q}:\R^d_+\to\R_d^+$, $\mv{Q}(\Lev_1^+(L_u)) = \Lev_1^+(\fLD)$. Thus, in such a case, Theorem~\ref{thm:Tail-asymp} may be derived as a consequence of the contraction principle, rather than utilise the more elaborate machinery developed in Lemmas~\ref{lem:cont-conv-set-incl}-\ref{lem:cont-I}. The former holds for example, if (i) $L(n\xx)  = n^\rho L(\xx)$ for all $n,\xx$ and (ii) $\mv{\Lambda}(\xx) =\xx^{\mv{\alpha}}$ for some $\alpha>0$.    
\end{remark}
}
\subsection{Proof of Proposition \ref{prop:unbiased} and useful bounds
  on the inverse of $\mv{T}(\cdot).$}


\noindent \textbf{Proof of Proposition~\ref{prop:unbiased}:} {First consider the case $\mv{\kappa} = \mv{\kappa}^{(1)}$.} Let
$k_{\rho}(u) :=\log(u/l).$ It is sufficient to show that the
determinant of the Jacobian of the map $\xx \mapsto \mv{T}(\xx)$
equals $J(\xx)$. In that case, the density of $\ZZ$, denoted by
$f_{\ZZ}(\cdot)$, is $f_{\ZZ}(\mv{T}(\xx)) = f_{\XX}(\xx)/J(\xx)$;
consequently, the likelihood ratio between the distributions of $\ZZ$
and that of $\XX$ (or the Radon-Nikodym derivative evaluated at the
samples $\ZZ_1\ldots \ZZ_n$) is given as in Algorithm~\ref{algo:IS}.  So the
rest of the verification is devoted to checking that $J(\cdot)$ indeed
equals the determinant of the Jacobian
$\textrm{Jac}_{\mv{T}}(\cdot)= \partial \mv{T}/\partial \xx.$ To this
end, define
$\mv{\psi}(\xx) = \log|\xx|+
k_{\rho}(u)\mv{\kappa}(\mv{\xx})$ and observe that
$\mv{T}(\xx) = \text{Diag}(\text{sgn}(\xx))\e^{\mv{\psi}(\xx)}$.  Now,
following the chain rule for Jacobians,
\begin{align*}
   \text{Jac}(\xx) &= \text{Diag}(\text{sgn}(\xx))\mathrm{Diag}(
  \e^{\mv{\psi}(\xx)}) \textrm{Jac}_{\psi}(\xx), \textrm{ for almost
    every $\xx,$ and }\\
  \textrm{Jac}_{\psi}(\xx)
  & = \mathrm{Diag}\left(\frac{\text{sgn}(\xx)}{|\xx|}\right) +
    \rho^{-1} k_{\rho}(u)\left[\mathrm{Diag}\left(\frac{\text{sgn}(\xx)}{(1+|\xx|) \|\log(1+|\xx|)\|_\infty}\right) -
    \frac{\log(1+|\xx|)}{\|\log(1+|\xx|)\|_\infty^2}
    (\mv{e}^*/(1+|\xx|))^{\intercal}\right]\\
  &=\mathrm{Diag}\left(  \text{sgn}(\xx)\left[\frac{1}{|\xx|}
    + \frac{ \rho^{-1} k_{\rho}(u)}{(1+|\xx|)\|\log(1+|\xx|)\|_\infty}\right]
    \right)
    - \rho^{-1} k_\rho(u)\left[\frac{\log(1+|\xx|)}{\|\log(1+|\xx|)\|_\infty^2}
    (\mv{e}^*/(1+|\xx|))^{\intercal}\right]
\end{align*}
where $\mv{e}^*_i =\text{sgn}(x_i)$ if $|x_i| = \|\xx\|_\infty$,  and $\mv{e}^{*}_i=0$ otherwise. Notice that for this component, $\|\xx\|_\infty= |x_i|$. Now, recall that if $M =A + \mathbf{u}\mathbf{v}^\intercal$, then $|M| = (1+\mathbf{u}^\intercal A^{-1}\mathbf{v})|A|$.  Set
\begin{align*}
  \mv{u} & = \log(1+|\xx|)/\|\log(1+|\xx|)\|_\infty^2, \quad
           \mv{v} =  -(k_\rho(u)\mv{e}^*/(1+|\xx|))^{\intercal}, \text{ and} \\
  A &=\mathrm{Diag}\left(  \text{sgn}(\xx)\left
      [\frac{1}{|\xx|}  + \frac{k_{\rho}(u)}{(1+|\xx|)\|\log(1+|\xx|)\|_\infty}\right]\right).
\end{align*}
Then, almost everywhere,
\begin{align*}
1+\mv{u}^\intercal A^{-1}\mv{v} &= 1-\frac{\rho^{-1} k_{\rho}(u)\|\xx\|_\infty}{\|1+|\xx|\|_\infty\|\log(1+|\xx|)\|_\infty + k_{\rho}(u) \|\xx\|_\infty}, \text{ and }\\
|A| &= \prod_{i=1}^d\frac{1}{|x_i|}\times\prod_{i=1}^d\left(1+\frac{ \rho^{-1} k_{\rho}(u)|x_i|}{(1+|x_i|)\|\log(1+|\xx|)\|_\infty}\right).  
\end{align*}
To complete the verification, observe  $|\mathrm{Diag}( \e^{\mv{\psi}(\xx)}) | = \prod_{i=1}^d |x_i| \cdot (u/l)^{\mv{1}^\intercal \mv{\kappa}(\xx)}$.\\
We next show that
$\mv{T}(\cdot)$ is onto. Fix a $\yy \in \Real^d_+$. Let
$I=\{i: y_i = \|\yy\|_{\infty}\}$.  Define the set,
\begin{align}
  S = \{\xx: x_i = \|\yy\|_\infty c_\rho(u)\textrm{ for all } i\in I,
  \, x_i \in [0, \|\yy\|_\infty c_\rho(u) ], \, i\not\in I\},
  \label{eqn-S-define}
\end{align}
where recall $c_{\rho}(u) := (l/u)^{1/\rho}.$ Notice that for all
$\xx\in S,$
\begin{align*}
  \mv{T}_i(\xx) =
  \begin{cases}
    x_i[c_\rho(u)]^{\frac{-\log (1+x_i)}{\log(1+ c_\rho(u)\|\yy\|_\infty)}}
    &\textrm{ for  } i\not\in I,\\
    \|\yy\|_\infty &\textrm{ for  } i\in I.
  \end{cases}
\end{align*}
Restricted to $\xx\in S$, $\mv{T}_i(\xx)$ is only a function of $x_i$
for $i\not\in I$. Fixing $i \notin I,$ see that $\mv{T}_i(\xx) < y_i$,
if $x_i<y_i c_\rho(u);$ and
$\mv{T}_i(\xx) = \Vert \yy\Vert_\infty > y_i$ if
$x_i=\|\yy\|_\infty c_\rho(u).$ Since $\mv{T}_i$ are all continuous
maps, by the intermediate value theorem, there exists some
$\xx^\prime \in [\yy c_\rho(u) , \|\yy\|_\infty c_\rho(u)]^d$ such
that $\mv{T}(\xx^\prime) = \yy.$ The above argument also shows that if $\xx_1\neq \xx^{\prime}$, $\mv{T}(\xx_1)\neq \mv{T}(\xx^{\prime})=\yy$, that is, $\mv{T}(\cdot)$ is $1\leftrightarrow 1$. Since $\mv{T}(\cdot)$ is symmetric
about the origin, one can similarly extend the proof to the case where
$\yy\in \R^d$.

{
Now suppose that $\mv{\kappa}=\mv{\kappa}^{(2)}$. Here, observe that by a direct application of the chain rule,
\[
\left[\frac{\partial \mv{\kappa}(\xx)}{\partial \xx} \right] = \text{Diag}\left(\frac{\log(u/l)}{\log l} \frac{\text{sgn}(\xx)}{1+|\xx|} \right).\]
Substituting in the expression for the Jacobian in \eqref{eqn:Jac_T_general} completes the proof. \hfill$\Box$

In Lemmas \ref{lemma-T-bound} and \ref{lem:property-T} below, assume that $\mv{\kappa}=\mv{\kappa}^{(1)}$:
}
\begin{lemma}
  \label{lemma-T-bound}
  For any $\yy \in \Real^d_+$ satisfying
  $\|\yy\|_\infty \geq 1/c_{\rho}(u),$
\begin{align}
  \yy c_\rho(u) 
  \leq \mv{T}^{-1}(\yy) \leq
  \min\left\{\yy [c_{\rho}(u)]^{\frac{\log\yy}{\log \| \yy\|_\infty}} \mv{\lor 1},
  \Vert \yy \Vert_\infty c_\rho(u) \right\} .
  \label{T-bounds}
\end{align}
\end{lemma}
\textit{Proof.}  We verify the bounds by exhibiting $\xx^\prime$
sandwiched component-wise between the left and right hand sides of
(\ref{T-bounds}) and satisfying $\mv{T}(\xx^\prime) = \yy.$ For any
$\yy$ as in the statement, we first set
\begin{align}
  \tilde{x}_i =
  \begin{cases}
    y_i \left[c_{\rho}(u)\right]^{\frac{\log
        y_i}{\log \|\yy\|_\infty}}
    &\text{ if } y_i\geq 1,\\
    1 &\text{ otherwise,}
  \end{cases}
        \label{tilde-x}
\end{align}
for $i=1,\ldots,d.$ See that
$\tilde{x}_i \in [1,\Vert \yy \Vert_\infty c_\rho(u)].$ This is
because of the following two observations: 
1)
$\log \tilde{x}_i = \left(1 + \log c_\rho(u)/\log
  \|\yy\|_\infty\right) \log y_i \geq 0,$ when
$\Vert \yy \Vert_\infty \geq 1/c_\rho(u) > 1;$ and 2) likewise,
\begin{align*}
  \log \tilde{x}_i -  \log (\|\yy\|_\infty c_{\rho}(u))
  &=\left( 1 + \frac{\log c_\rho(u)}{\log \Vert \yy \Vert_\infty}
    \right) \left(\log y_i - \log \Vert \yy \Vert_\infty \right)
    \leq 0. 
\end{align*}
Set $I=\{i: y_i = \|\yy\|_{\infty}\}.$ With the set $S$ defined as in
(\ref{eqn-S-define}), we thus have
$\tilde{\xx} = (\tilde{x}_1,\ldots,\tilde{x}_d) \in S$ and
$\tilde{x}_i \geq 1$ for all $i.$ Since $\log(1+t)/\log(t)$ is
decreasing over $t\geq 1,$ we have
\begin{equation}
   \label{eqn:T-old-new}
  \mv{T}_i(\xx) = x_i [c_\rho(u)]^{-\frac{\log(1+x_i)}{\|\log(1+\xx)\|_\infty}} \geq x_i [c_\rho(u)]^{^{\frac{-\log x_i}{\ \log \|\xx\|_\infty}}}, 
\end{equation}  
for any $\xx$ such that $x_i\geq 1, i = 1,\ldots,d.$ With
$\tilde{\xx}$ defined via (\ref{tilde-x}), we obtain the following
from the bound in (\ref{eqn:T-old-new}) and the observation
$(\log \tilde{x}_i)/(\log \Vert \tilde{\xx} \Vert_\infty) = (\log
y_i)/(\log \Vert \yy \Vert_\infty):$
\begin{align*}
  \mv{T}_{i}(\tilde{\xx})
  = \begin{cases}
    \mv{T}_i(y_i [c_\rho(u)]^{\frac{\log y_i}{\log\| \yy\|_\infty}}) &   \textrm{ if } y_i \geq 1, \\
    \mv{T}_i(1) &\text{ otherwise},
  \end{cases}
        \quad\geq\   \begin{cases}
          y_i &   \textrm{ if } y_i \geq 1, \\
          1 &\text{ otherwise}.
\end{cases}
\end{align*}
For $i \notin I,$ the map $\mv{T}_i(\xx),$ when restricted to
$\xx \in S,$ is a function only of $x_i$ and satisfies
$\mv{T}_i(\tilde{\xx}) \geq y_i$ (regardless of whether $y_i > 1$ or
$y_i < 1$). Applying the intermediate value theorem component-wise, we
see that there exists some $x_{i}^\prime$ in the interval
$[y_ic_{\rho}(u), \tilde{x}_i ]$ containing $y_i$ such that
$T_i(x_{i}^\prime) = y_i$ for all $i\not\in I$. Setting
$x_i^\prime = \Vert \yy \Vert_\infty c_\rho(u)$ for $i \in I,$ we
obtain $\mv{T}(\mv{x}^\prime) = \yy.$ Hence the bounds
(\ref{T-bounds}) hold. \hfill$\Box$

\begin{lemma} \label{lem:property-T} Suppose that Assumptions
  \ref{assump-marginals} - \ref{assump:joint-Y} hold, $l(u)$ is slowly
  varying in $u$, and $\lim_{u \rightarrow \infty}l(u) = \infty.$ Then
  for any $\gamma>0$, the below convergence holds uniformly over $\pp$
  in compact subsets of $\Real^d_{++}:$ 
  \[\Vert \psi_u(t(u)\pp) \Vert_\infty = o(t(u)), \text{ as } u
    \rightarrow \infty.\]
  
\end{lemma} 
\textit{Proof.}  Recall
$\psi_u := \mv{\Lambda} \circ \mv{T}^{-1} \circ \mv{q},$
$c_\rho(u) := (l(u)/u)^{1/\rho}$ and
$q_\infty(t(u)) := \Vert \mv{q}(t(u)) \Vert_\infty.$ Fix any
$M > 0, \gamma \in (0,M)$ and $\pp \in B_M \setminus B_\gamma.$ As
$\lim_{u \rightarrow \infty}l(u) = +\infty,$ $\qq \in \RV,$ and
$q_\infty(t(u)) = u^{1/\rho}$ (see Lemma~\ref{lem:qtequ}), 
\begin{equation}
  \Vert \qq(t(u)\pp) \Vert_\infty c_\rho(u) = l(u)^{1/\rho}\frac{\Vert
    \qq(t(u)\pp)\Vert_\infty}{\Vert \qq(t(u)) \Vert_\infty}
  \rightarrow \infty.
  \label{q-growing}
\end{equation}
Then from (\ref{T-bounds}),
$\mv{T}^{-1}_i (\qq(t(u)\pp)) \leq \|\qq(t(u)\pp)\|_\infty
c_\rho(u)\leq \|\qq(t(u)M\mv{1})\|_\infty c_\rho(u),$ for
$i = 1,\ldots,d,$ as $\qq$ is monotone and
$\Vert \pp \Vert_\infty \leq M$. With $q_\infty(t(u)) = u^{1/\rho},$
 \[
   \qq(t(u)M\mv{1})c_{\rho}(u) =
   \frac{\qq(t(u)M\mv{1})}{\qq(t(u))}
   \frac{\mv{q}(t(u))}{q_\infty(t(u))}l(u)^{1/\rho}
   \leq \frac{\qq(t(u)M\mv{1})}{\qq(t(u))}l(u)^{1/\rho}.
 \] 
 Since $\qq\in \RV(\aalpha),$ 
 we have
 $ \qq(t(u)M\mv{1})c_{\rho}(u) \leq M^{\aalpha}
 l(u)^{1/\rho}(\mv{1}+o(\mv{1})).$ Therefore, 
 \[\mv{T}^{-1}_i (\qq(t(u)\pp)) \leq \|\qq(t(u)M\mv{1})\|_\infty c_\rho(u)
   \leq \max_{i=1,\ldots,d} M^{1/\alpha_i}
   l(u)^{1/\rho}({1}+o({1})), 
 \]
 uniformly over $\pp \in B_M \setminus B_\gamma(\mv{0}).$ Recall
 $\mv{\Lambda}\in \RV(\mv{\alpha})$ is monotone.  Write
 $\bar{\alpha} = \max_i\alpha_i.$ Then for $\delta > 0,$
\[
  \|\mv{\psi}_u(t(u)\pp))\|_\infty = \max_{i=1,\ldots,d}
  {\Lambda}_i\left( \mv{T}^{-1}_i (\qq(t(u)\pp)) \right) \leq
  (1+\delta)M_0 l(u)^{\bar{\alpha}/\rho + \delta}, 
\]
for all $u$ sufficiently large (see \cite[Proposition B.1.9
(7)]{deHaanFerreira2010}). Here $M_0>0$ is a suitably large
constant. Since $t(u) = \Lambda_{\min}(u^{1/\rho}),$ noting
$l(u)^{\bar{\alpha}/\rho + \delta} = o(t(u))$ completes the
proof. \hfill$\Box$

\subsection{Proof of Theorem \ref{thm:Var-Red} {with $\mv{\kappa} = \mv{\kappa_1}$}}
{Here, we demonstrate the proof of Theorem~\ref{thm:Var-Red} in the case where $\mv{\kappa} =\mv{\kappa}_1$. The proof for $\mv{\kappa}=\mv{\kappa}_2$ is outlined in \ref{sec:verify-alt}}. Unless explicitly specified, the only assumption made in the proofs
below is that $\XX$ has a probability density $f_{\XX}(\cdot).$ As a
consequence, the hazard rates
\[\lambda_i(x) = f_{X_i}(x)/\Lambda_i(x),\quad i = 1,\ldots,d\]
are well-defined. In the above, $f_{X_i}(\cdot)$ denotes the
probability density of the component $X_i.$ 

\begin{lemma}
  \label{lem:M2-rewrite}
  The second moment $M_{2,u} = \Expc[\exp(-t(u)F_u(\YYu))],$ where
  $F_u:\Real^d_{+} \rightarrow \Real \cup \{ +\infty\}$ is,
  \begin{align*}
    F_u(\pp) := a_u(\pp) + b_u(\pp) - 2d\frac{\log t(u)}{t(u)} +
    \chi_{\Lev^+_{1}(L_u)}(\pp), 
  \end{align*}
  for $u > 0,$ and $a_u:\Real^d_{+} \rightarrow \Real,$
  $b_u:\Real^d_{+} \rightarrow \Real$ are defined as follows:
  \begin{align*}
    a_u(\pp) &= \frac{1}{t(u)} \left[ \log f_{\YY}(\mv{\psi}_u(t(u)\pp)    - \log f_{\YY}(t(u)\pp) \right], \\
    b_u(\pp) &= \frac{1}{t(u)} \left[\sum_{i=1}^d \left[
    \log\lambda_i(\mv{T}^{-1}_i(\mv{q}(t(u)\pp))) -
    \log\lambda_i(q_i(t(u)p_i))\right] -
    \log J(\mv{T}^{-1}(\mv{q}(t(u)\pp))) \right]. 
  \end{align*}  
\end{lemma}

\noindent \textit{Proof}. Since the
change of measure is effected by the map $\ZZ = \mv{T}(\XX),$
\begin{align}
  M_{2,u} &= {\Expc}\left[\left( \frac{f_{\XX}(\ZZ)}{{f}_{\ZZ}(\ZZ)}\right)^2
            \mathbb{I}(L(\ZZ) \geq  u)\right]
        = \Expc\left[\left( \frac{f_{\XX}(\XX)}{{f}_{\ZZ}(\XX)}\right)\mathbb{I}(L(\XX) \geq  u)\right]\nonumber\\
          & = \int_{L(\xx) \geq u}  \frac{f_{\XX}(\xx)}{ f_{\XX}(\mv{T}^{-1}(\xx))} J(\mv{T}^{-1} (\xx)) f_{\XX}(\xx) d\xx. \label{eqn:COV-2}
\end{align}
 Changing variables from $\xx$ to $\yy =\qq^{-1}(\xx)$, we have 
 $f_{\XX}(\xx) = \prod_{i=1}^d \lambda_i(x_i) f_{\YY}(\qq^{-1}(\xx)),$
where $\lambda_i(x) = f_{X_i}(x)/\Fbar_{X_i}(x)$ is the hazard rate of $X_i$. Thus,
\begin{align*}
  M_{2,u}& =  \int_{L(\qq(\yy)) \geq u}  \prod_{i=1}^{d}\frac{ \lambda_i(q_i(y_i))}{\lambda_i(\mv{T}^{-1}_i(q_i(y_i)))} \frac{f_{\YY}(\yy)}{ f_{\YY}(\mv{\psi}_u(\yy)))} J(\mv{T}^{-1} (\qq(\yy))) f_{\YY}(\yy) d\yy\\ 
         &=  t(u)^{d}\int_{\pp\in \Lev_1^+(L_u)}  \underbrace{\frac{f_{\YY}(t(u)\pp)}{ f_{\YY}(\mv{\psi}_u(t(u)\pp)))}  }_{(a)} \underbrace{\prod_{i=1}^{d}\frac{ \lambda_i(q_i(t(u)p_i))}{\lambda_i(\mv{T}^{-1}_i(\qq(t(u)\pp)))}J(\mv{T}^{-1} (\qq(t(u)\pp)))}_{(b)} f_{\YY}(t(u)\pp)
           d\pp, 
\end{align*}
Since $\YYu = t(u)^{-1}\YY,$
$f_{\YY}(t(u)\pp) = t(u)^{d} f_{\YYu}(\pp)$.  Checking the terms
labeled $(a)$ and $(b)$ in the above expression equal
$\exp(-t(u)a_u(\pp))$ and $\exp(-t(u)b_u(\pp)),$ respectively, we
obtain
\begin{equation*}
  M_{2,u} = t(u)^{2d} \Expc\left[\exp\left\{-t(u)\left[a_u(\YYu) + b_u(\YYu)
        + \chi_{\Lev_1^+(L_u)}(\YYu)\right]\right\}\right]. \qquad\qquad \Box
\end{equation*}

Thanks to Lemma \ref{lemma-T-bound} - \ref{lem:property-T}, the terms
in Lemma \ref{lem:M2-rewrite} enjoy the following bounds.
\begin{lemma}
  \label{lemma: RV-1}
  Suppose that Assumptions \ref{assump-marginals} -
  \ref{assump:joint-Y} are satisfied and $l(u)$ is slowly varying in
  $u.$ Then $a_u(\pp) \geq I(\pp) + o(1),$ as $u \rightarrow \infty,$
  uniformly over $\pp$ in compact subsets of $\Real^d_{++}.$
\end{lemma}

\noindent \textit{Proof.} 
Under Assumption~\ref{assump:joint-Y}, uniformly over
$\hat{\yy} := \yy / \Vert \yy \Vert$ on
$\mathcal{S}^{d-1} \cap \Real^d_+,$
\begin{equation}
  -\frac{\log f_{\YY}(\|\yy\| \hat{\yy} )}{\|\yy\|} \to I(\hat{\yy}),
  \text{ as } \|\yy\| \to\infty \quad \text { (replacing
    $n$ there by $\|\yy\|$) }.
  \label{fY-unit-sphere}
\end{equation}
Fix any $M > 0, \gamma \in (0,M)$ and $\varepsilon \in (0,1).$ From
the monotonicity of $\mv{\Lambda},$ the lower bound in
(\ref{T-bounds}), and (\ref{q-growing}),
$\|\mv{\psi}_u(t(u)\pp)\|_\infty := \Vert
\mv{\Lambda}(\mv{T}^{-1}(\mv{q}(t(u)\pp)))\Vert_\infty \geq
\|\mv{\Lambda}\left(\qq(t(u)\pp)c_\rho(u)\right) \|_\infty \rightarrow
\infty,$ for any $\pp \in B_M \setminus B_\gamma,$ as
$u \rightarrow \infty.$ Then due to (\ref{fY-unit-sphere}) and the
upper bound in (\ref{T-bounds}),
\[
  - \log f_{\YY}(\mv{\psi}(t(u)\pp)) \leq \Vert \mv{\psi}(t(u)\pp)
  \Vert_\infty \big[\sup_{\hat{\yy} \in \mathcal{S}^{d-1} \cap
      \Real^d_+} I(\hat{\yy})) + \varepsilon\big] \leq \varepsilon
  t(u)(M_1+\varepsilon)),
\]
for all sufficiently large $u;$ here
$M_1 := \sup\{I(\hat{\yy}): \hat{\yy} \in \mathcal{S}^{d-1} \cap
\Real^d_+\}$ is a finite positive constant (due to the regularity
properties of $I$ in Lemma \ref{lem:properties-I}).  Observe that from
the upper bound in \eqref{eqn:Y-Bounds}, 
\begin{align*}
  - \log f_{\YY}((t(u)\pp)) \geq  t(u) (I(\pp) -
  \varepsilon), 
\end{align*}
uniformly over $\pp\in B_M\setminus B_{\gamma}$ and all $u$
sufficiently large. Combining the above displayed bounds on
$-\log f_{\YY} (\cdot)$ terms, we obtain from the definition of
$a_u(\cdot)$ that
$a_u(\pp) \geq I(\pp) - \varepsilon - \varepsilon (M_1+\varepsilon).$
\hfill$\Box$

\begin{lemma}
  \label{lemma:Jac-Light}
  Suppose that Assumptions \ref{assump-marginals} and \ref{assume:mhr}
  are satisfied and $l(u)$ is slowly varying in $u.$ Then
  $\liminf_{u \rightarrow \infty}b_u(\pp) \geq 0,$ where the
  convergence is uniform over $\pp \in \Real^d_+.$
\end{lemma}
\noindent\textit{Proof.} Recall the definitions of $J(\xx)$ and $\tilde{J}_i(\xx)$
in  Table~\ref{tab:Jacobians}.  Since $\mv{\kappa}(\xx)$ in
(\ref{IS-transf}) satisfies $\mv{\kappa}(\xx) \in [0,1]^d,$ we have
$\mv{1}^\intercal \mv{\kappa}(\xx) \leq d.$ Next observe that for all
$t\geq 0$, $t/[(1+t)\log(1+t)]\leq \e$. Therefore for
$\xx\in \R^d_{++}$,
\[
  \prod_{i=1}^d \tilde{J}_i(\xx) \leq \prod_{i=1}^d \left[1 + \frac{
      \rho^{-1}\log(u/l)}{\log(1+|\xx|_i)}\frac{|x_i|}{1+|x_i|}
  \right] \leq \left(1 + \e \rho^{-1}  \log (u/l)
  \right)^{d}.\]
Moreover, $\max\{\tilde{J}_i(\xx), \ldots, \tilde{J}_d(\xx)\} \geq 1.$
Combining these observations we obtain that
\begin{align}
  J(\xx) \leq  \left[1 + \e \rho^{-1}  \log (u/l)
  \right]^{d} (u/l)^d, \quad \xx \in \Real^d_{++}. 
  \label{inter-jac-light-1}
\end{align}
To bound the terms involving hazard rates $\lambda_i(\cdot),$ we
proceed as follows: Due to Assumption \ref{assume:mhr},
$\lambda_i(\cdot)$ is eventually monotone for any $i$. From
Lemma~\ref{lemma-T-bound}, if $\lambda_i$ is eventually decreasing,
$\lambda_i(\mv{T}^{-1}_i(\qq(t(u)\pp)) > \lambda_i(q_i(t(u)p_i))$. If
$\lambda_i$ is eventually increasing, the bound
$\mv{T}^{-1}(\qq(t(u)\pp)) \geq \qq(t(u)\pp)c_\rho(u)$ from
(\ref{T-bounds}) implies
$\lambda_i(\mv{T}_i^{-1}(q_i(t(u)p_i)))\geq
\lambda_i(q_i(t(u)p_i)c_\rho(u)).$ In either case,
\begin{align}
  \log\lambda_i\big(\mv{T}^{-1}_i(\mv{q}(t(u)\pp))\big) -
  \log\lambda_i\big(q_i(t(u)p_i)\big)  \geq
  \log \lambda_i\big(q_i(t(u)p_i)c_\rho(u)\big) - \log \lambda_i\big(q_i(t(u)p_i)\big). 
  \label{Jac-light-inter-2}
\end{align}
Since $\Lambda_i\in \RV(\alpha_i)$ and
$\lambda_i = \Lambda_i^{\prime}$ is monotone,
$\lambda_i\in \RV(\alpha_i - 1)$ (see \cite[Proposition B.1.9
(11)]{deHaanFerreira2010}). Given $\varepsilon > 0,$ an application of
Potter's bounds \cite[Proposition B.1.9 (7)]{deHaanFerreira2010}
yields
$\lambda_i(q_i(t(u)p_i )c_\rho(u))/\lambda_i(q_i(t(u)p_i)) \geq
c_\rho(u)^{\alpha_i-1+\epsilon},$ for all $u$ sufficiently large.
Then from
\eqref{Jac-light-inter-2}, 
$ \log\lambda_i\big(\mv{T}^{-1}_i(\mv{q}(t(u)\pp))\big) -
\log\lambda_i\big(q_i(t(u)p_i)\big) \geq (\alpha_i - 1 + \varepsilon)
\log c_\rho(u),$ for $i = 1, \ldots,d.$
Since $c_\rho(u) := (l/u)^{1/\rho},$ we obtain the following by
combining this bound with (\ref{inter-jac-light-1}):
\begin{align*}
  \inf_{\pp \in \Real^d_+}b_u(\pp) \geq -\rho^{-1}\frac{\log(u/l)}{t(u)} \sum_{i=1}^d (\alpha_i - 1 + \varepsilon)
  - d \frac{\log \left(1 + \e \rho^{-1} \log(u/l)\right)}{t(u)} -d\frac{\log(u/l)}{t(u)}
  \rightarrow 0,
\end{align*}
where the convergence follows from noting
$t(u) := \Lambda_{\min}(u^{1/\rho})$ and $\log(u/l) = o(t(u)).$
\hfill$\Box$

\begin{lemma}
\label{lemma:Not-Zero}
  Suppose that Assumptions \ref{assume:V} and \ref{assump-marginals}
  are satisfied. Then for all sufficiently large $u,$
  $\Lev^+_1(L_u) \subseteq \Real^d_+ \setminus B_\gamma,$ for
  some $\gamma > 0.$
\end{lemma}
\noindent\textit{Proof}.   Recall the definition
$\fLD(\yy): = L^*(\qq^* \yy^{\aalpha}).$ The function $\fLD(\cdot)$ is
therefore continuous and bounded on the unit sphere
$\mathcal{S}^{d-1} \cap \Real^d_+.$ Since $q_i^\ast = 0$ if
$\alpha_i > \min_j \alpha_j$, we have for all $c>0$,
\[\fLD(c\yy) = L^*(\qq^*(c\yy)^{\aalpha})
  = L^*(c^{1/\alpha_*} \qq^*\yy^{\aalpha}) =
  c^{\rho/\alpha_*}L^*(\qq^{*}\yy^{\aalpha}) =
  c^{\rho/\alpha_*}\fLD(\yy) > 0,\] from the homogeneity of $L^\ast.$
Combining these observations, we obtain that for any $\gamma > 0,$
$\sup_{\yy \in B_\gamma}\fLD(\yy) < \gamma^{\rho/\alpha_*} M_1$, where
$M_1 > \max_{\yy \in B_1}\fLD(\yy)$. Choosing
$\gamma<(2M_1)^{-\alpha_*/\rho}$ ensures
$\sup_{\yy \in B_\gamma}\fLD(\yy)<1/2$. Thus
$[\Xi_{1,\gamma}(\fLD)]^{1+\varepsilon} \cap B_{\gamma} =\emptyset$,
if $\varepsilon > 0$ is chosen suitably small.  Then from the
inclusion
$\Lev^+_1(L_u) \cap B_\gamma \subseteq
[\Xi_{1,\gamma}(\fLD)]^{1+\varepsilon}$ from Corollarly
\ref{cor:set-inclusions}, we have
$\Lev_1^+(L_u)\cap B_{\gamma} =\emptyset,$ for all $u$ sufficiently
large. \hfill$\Box$

Recall from Lemma \ref{lem:M2-rewrite} that the second moment
$M_{2,u} = E\left[ \exp \big(-t(u)F_u(\YYu) \big) \right].$ 
\begin{lemma}
  \label{lem:non-negative-Fu}
  Suppose that Assumptions \ref{assume:V} - \ref{assume:mhr} are
  satisfied and $l(u)$ is taken to be slowly varying in $u.$ Then
  there exists $u_1$ sufficiently large such that for all $u > u_1,$
  $\inf_{\pp \in \R_{+}^{d}}F_u(\pp) \geq 0.$
\end{lemma}
Checking this non-negativity in Lemma \ref{lem:non-negative-Fu}, while
is executed along similar lines as in the proofs of Lemma \ref{lemma:
  RV-1} - \ref{lemma:Jac-Light}, is technically more involved. Its
proof is therefore given later in Section \ref{sec:tech-proofs}, which
is devoted to technical results that are repetitive in terms of the
key ideas involved. We now prove the key variance reduction result,
namely, Theorem \ref{thm:Var-Red}.


\noindent\textbf{Proof of Theorem \ref{thm:Var-Red}.}
From Lemma \ref{lem:M2-rewrite}, we have
$M_{2,u} = E\left[ \exp \{-t(u)F_u(\YYu) \} \right].$ Define the
function $F: \Real^d_+ \rightarrow \Real \cup \{+\infty \}$ as,
\begin{align*}
  F(\pp) := I(\pp) + \chi_{\Lev^+_1(\fLD)}(\pp). 
\end{align*}
Since $\Lev^+_1(\fLD)$ is closed and $I(\cdot)$ is continuous,
$F(\cdot)$ is lower semi-continuous. Consider sequences
$\{u_n\} \subseteq \Real_+,\{\pp_n\} \subseteq \Real^d$ satisfying
$u_n \rightarrow \infty$ and $\pp_n \rightarrow \pp.$ Due to Lemma
\ref{lemma:Not-Zero}, there exists $\gamma,n_0 > 0$ such that
$\Lev_1^+(L_{u_n}) \cap B_\gamma = \varnothing,$ for all $n > n_0.$
Suppose $\pp \notin B_{\gamma/2}.$ Then from the uniform convergences
of $a_u(\cdot),b_u(\cdot)$ in Lemma \ref{lemma: RV-1} -
\ref{lemma:Jac-Light} and that of $\chi_{\Lev^+_1(L_u)}(\cdot)$ in
Corollary \ref{cor:set-inclusions},
\begin{align}
  \liminf_{n \rightarrow \infty} F_{u_n}(\pp_n) :=
  \liminf_{n \rightarrow \infty} a_{u_n}(\pp_n) + b_{u_n}(\pp_n) -
  2d\frac{\log t(u_n)}{t(u_n)} + \chi_{\Lev^+_{1}(L_{u_n})}(\pp_n) \geq F(\pp).
  \label{lim-Fu}
\end{align}
On the other hand, if $\pp \in B_{\gamma/2},$ we have
$\{\pp_n: n \geq n_1\} \subseteq B_{\gamma}$ for some $n_1 > n_0.$
Since $\Lev_1^+(L_{u_n}) \cap B_\gamma = \varnothing$ for all
$n > n_1,$ we obtain $\inf_{n \geq n_1} F_{u_n}(\pp_n) = \infty.$
Thus, regardless of the membership of $\pp$ (in the ball
$B_{\gamma/2}$), (\ref{lim-Fu}) holds.  From Lemma
\ref{lem:non-negative-Fu}, we deduce that there exists $n_2 > n_1$
satisfying that the family $\{F_{u_n}: n \geq n_2\}$ comprises
non-negative valued functions. Recall from Theorem \ref{thm:LDP} that
the sequence $\{\YY_{\hspace{-4pt}{u_n}}: n \geq 1\}$ satisfies LDP
with rate function $I(\cdot).$ Then due to a general version of
Varadhan's integral lemma (see \cite[Theorem 2.2]{LD_Varadhan}), we
obtain from (\ref{lim-Fu}) that
\begin{align*}
  \limsup_{n \rightarrow \infty} \frac{1}{t(u_n)} \log E\left[ \exp\big\{ -t(u_n)
  F_{u_n}(\YY_{\hspace{-4pt}{u_n}})\big\}\right] \leq -\inf_{\pp \in \Real^d} \left\{
  F(\pp) + I(\pp)\right\} = -2 \inf_{\pp \in \Lev^+_1(\fLD)}I(\pp).
\end{align*}
Since $M_{2,u} = E\left[ \exp \{-t(u)F_u(\YYu) \} \right],$ combining
this conclusion with the bounds $M_{2,u} \geq p_u^2$ and
$\liminf_{u \rightarrow \infty} [t(u)^{-1}\log p_u^2] \geq -2I^\ast$
from Theorem \ref{thm:Tail-asymp}, we obtain (\ref{log-eff-thm}).
\hfill$\Box$

{
\subsection{Proofs of  Propositions~\ref{prop:zv-ldp}-\ref{prop:Self-Structuring} and Corollaries~\ref{cor:Self-struc}-\ref{cor:Self-struc-2}}

\noindent \textbf{Proof of Proposition~\ref{prop:zv-ldp}: } Recall that $\YY_u = \YY/t(u)$. Observe that $\mv{\Lambda}(\ZZ_u^\ast)/t(u)$ has the distribution of $\YY_{u}$ given $L_u(\YY_{u}) \geq 1$. Now, notice that for any $\pp$, $\delta>0$, 
\begin{equation}\label{eqn:Condtional-LDP}
P(\YY_{u}\in B_{\delta}(\pp) \mid \{L_u(\YY_{u}) \geq  1\}) = \frac{P(\YY_{u}\in B_{\delta}(\pp) \cap \{L_u(\YY_{u}) \geq 1\}) }{P(L_u(\YY_{u}) \geq 1)}
\end{equation}
The numerator in \eqref{eqn:Condtional-LDP} can be upper bounded invoking the LDP for $\YY_{u}$:
\begin{equation}\label{eqn:Conditional_Upper_Bound}
    \inf_{\delta >0}\limsup_{u\to\infty}\frac{1}{t(u)} \log P(\YY_{u}\in B_{\delta}(\pp) \cap \{L_u(\YY_{u}) \geq 1\}) \leq -[I(\pp)  + \chi_{1}(\pp)] 
\end{equation}
To see this, notice that if $\pp\in \{\fLD(\pp)\geq1\}$, \eqref{eqn:Conditional_Upper_Bound} follows from the continuity of $I(\cdot)$ and the LDP upper bound for $\YY_u$. Now suppose, $\pp\not\in  \{\fLD(\pp) \geq 1\}$. From the continuous convergence of $L_u$ to $\fLD$, following the proof of Corollary \ref{cor:set-inclusions}, there exist $\delta_0,u_0$, such that $B_{\delta}(\pp) \cap \{L_u(\YY_u) \geq 1\} = \emptyset \ \forall \delta<\delta_0$ and $u>u_0$. Thus, for all $\delta<\delta_0, u>u_0$, the probability in \eqref{eqn:Conditional_Upper_Bound} evaluates to $0$. The bound in \eqref{eqn:Conditional_Upper_Bound} now follows. Using the tail asymptotic (Thm.~\ref{thm:Tail-asymp}) in the denominator of \eqref{eqn:Condtional-LDP},
\[
\inf_{\delta >0}\limsup_{u\to\infty}\frac{1}{t(u)} \log P(\YY_{u}\in B_{\delta}(\pp) \mid \{L_u(\YY_{u}) \geq 1\}) \leq -[I(\pp) - I^* + \chi_{1}(\pp)] 
\]
Fix an arbitrary $\varepsilon>0$ and let $M>0$ sufficiently large so that $B_1(\pp) \subset B_M$. Recall that as a consequence of Corollary~\ref{cor:set-inclusions},  $\Xi_{1+\varepsilon,M}(L_u) \subseteq \Xi_{1,M}(\fLD) \subseteq \Lev_{1}^+(L_u)$ for all $u$ large enough. Then, \[
P(\YY_{u}\in B_{\delta}(\pp) \cap  \{L_u(\YY_{u}) \geq  1\})  \geq P(\YY_{u}\in B_{\delta}(\pp)\cap \{\YY_{u} \in\textrm{Int}(\Xi_{1+\varepsilon,M}(\fLD)))\}).
\]
Apply the LDP lower bound for $\YY_{u}$ to the left hand side above, and use the continuity of $I(\cdot)$ (refer to the proof of Theorem~\ref{thm:LDP})
\begin{equation}\label{eqn:Partial-lower}
\inf_{\delta>0}\liminf _{u\to\infty}\frac{1}{t(u)} \log P(\YY_{u}\in B_{\delta}(\pp) \cap \{L_u(\YY_{u}) \geq  1\}) \geq -[I(\pp) +\varepsilon + \chi_{\Lev_{1+\varepsilon}^+(\fLD)}(\pp)]  
\end{equation}
From Theorem~\ref{thm:Tail-asymp}, observe that
\[
\lim_{u\to\infty}\frac{1}{t(u)} \log  P(L_u(\YY_{u}) \geq 1) = -I^*.
\]
Noting that $\varepsilon>0$ in \eqref{eqn:Partial-lower} is arbitrary,
\[
\inf_{\delta >0}\liminf_{u\to\infty}\frac{1}{t(u)} \log P(\YY_{u}\in B_{\delta}(\pp) \mid \{L_u(\YY_{u}) \geq 1\}) \geq -[I(\pp) - I^* + \chi_{1}(\pp)] 
\]
The LDP required in Proposition~\ref{prop:zv-ldp} follows as a consequence \cite[Theorem 4.1.11]{Dembo}. To see the last statement observe that as $u\to\infty$,
\[
\frac{\ZZ_u^*}{\qq(t(u))} = \frac{{\qq\left(t(u)\frac{\mv{\Lambda}(\ZZ_u^*)}{t(u)}\right)}}{\qq(t(u))} \ \text{ and } \ \frac{\qq(t(u)\pp_u)}{\qq(t(u))} \to \pp^{1/\mv{\alpha}} \text{ whenever } \pp_u\to \pp.
\]
An application of the approximate contraction principle \cite[Theorem 4.2.23]{Dembo} to the sequences $\{{\mv{\Lambda}(\ZZ_u^*)}/{t(u)}:u>0\}$ shows that $\{{\ZZ_u^*}/{\qq(t(u))}:u>0\}$ satisfies an LDP with rate function $I_1(\pp) = I(\pp^{\mv{\alpha}}) - I^* + \chi_1(\pp^{\mv{\alpha}})$. The last statement now follows as a consequence of the definition of the LDP  \citep[see][Section 1.2]{Dembo}.\hfill $\Box$

\noindent \textbf{Proof of Proposition~\ref{prop:Self-Structuring}: }
Let $\qq$, $\mv{\Lambda}$ and $t(u)$ be defined as before. Under the assumptions of the proposition, uniformly over $\pp$ in compact sets of $\R^d_+\setminus\{ 0\}$,
\begin{align*}
    \frac{\mv{T}_u \circ \qq(t(u)\pp)}{\qq(t(u))} &= [\pp s^{\alpha_\ast/\rho}]^{1/\mv{\alpha}}[1+o(1)] \implies\\
    \mv{T}_u \circ \qq(t(u)\pp) &=\qq(t(u)) ([\pp s^{\alpha_\ast/\rho}]^{1/\mv{\alpha}} +o(1))\implies\\
    \frac{\mv{\Lambda} \circ \mv{T}_u \circ \qq (t(u)\pp)}{t(u)} &= \pp s^{\alpha_\ast/\rho} [1+o(1)] \textrm{ (apply $\mv{\Lambda} \in \RV(\mv{\alpha})$ to both the sides of the above)}.
\end{align*}

For convenience, denote $[t(u)]^{-1}\mv{\Lambda} \circ \mv{T}_u\circ\qq(t(u)\pp) = \mv{\phi}_u(\pp)$. Observe that $\mv{\phi}_u(\cdot)$ converges uniformly over compact sets to $s^{\alpha_\ast/\rho} \text{Id}$, where $\text{Id}$ is the identify function. Then, from the approximate contraction principle (see \cite[Theorem 4.2.23]{Dembo}), and the homogeneity of $I(\cdot)$, $\mv{\phi}_u(\YY_{u})$ satisfies an LDP with rate function $I_1(\xx) = s^{-\alpha_\ast/\rho}I(\xx)$. Observe next that $[t(u)]^{-1}\mv{\Lambda}(\ZZ_u)$ has the distribution of $\mv{\phi}_u(\YY_{u})$ given $L_u(\mv{\phi}_u(\YY_{u}))  \geq 1$. Then, 
\begin{equation}\label{eqn:2nd-LDP}
P(\mv{\phi}_u(\YY_{u})\in B_{\delta}(\pp) \mid \{L_u(\mv{\phi}_u(\YY_{u})) \geq 1\}) = \frac{P(\mv{\phi}_u(\YY_{u})\in B_{\delta}(\pp) \cap \Lev_1^+(L_u)) }{P(\mv{\phi}_u(\YY_u) \in \Lev_1^+(L_u))}    
\end{equation}
%
The limit of the numerator of \eqref{eqn:2nd-LDP} can be  evaluated by invoking the LDP of $\mv{\phi}_u(\YY_{u})$ and proceeding as in the proof of Proposition~\ref{prop:zv-ldp}, replacing $I(\cdot)$ there by $s^{-\alpha_\ast/\rho}I(\cdot)$:
\begin{align*}
    \inf_{\delta >0}\limsup_{u\to\infty}\frac{1}{t(u)} \log P(\mv{\phi}_u(\YY_{u})\in B_{\delta}(\pp) \cap \Lev_1^+(L_u)) &=  \inf_{\delta >0}\liminf_{u\to\infty}\frac{1}{t(u)}\log  P(\mv{\phi}_u(\YY_{u})\in B_{\delta}(\pp) \cap \Lev_1^+(L_u))\\
    &=  -s^{-\alpha_\ast/\rho}I(\pp) + \chi_{1}(\pp)\\
    &= -s^{-\alpha_\ast/\rho}[I(\pp) + \chi_{1}(\pp)] 
\end{align*}
where the last statement follows since $\chi_{1}(\pp)$ equals either $0$ or $+\infty$. The denominator can be evaluated by noting that
\[
\lim_{u\to\infty} \frac{1}{t(u)} \log P(L_u(\mv{\phi}_u(\YY_{u})) \geq 1) = -\inf_{\pp\in \Lev_{1}^+(\fLD)} s^{-\alpha_\ast/\rho} I(\pp) = -s^{-\alpha_\ast/\rho} I^*.
\]
Combining everything together and observing that $s^{-\alpha_\ast/\rho}t(u)= t(u/s)(1+o(1))$, $\{\mv{\Lambda}(\ZZ_u)/t(u): u>0\}$ satisfies LDP with rate $I-I^*+\chi_1$ and speed $t(u/s)$ as a consequence of \cite[Theorem 4.1.11]{Dembo}. Finally, to verify condition (ii), observe that from the above calculation, $\log P(L(\mv{T}_u(\XX) \geq u) = - t(u/s) [I^*+o(1)] = [1+o(1)]\log  P(L(\XX) \geq u/s)$. \hfill $\Box$

\noindent 
\textbf{Proof of Corollary~\ref{cor:Self-struc}: }  Consider $\{\pp_{u_n}\}_{n \geq 1} \subset \Real^d_+ $ such that
$\pp_{u_n} \rightarrow \pp \in \Real^d_+ \setminus \{\mv{0}\}$ as
$n \rightarrow \infty.$ 

\noindent \textbf{1) } Suppose $\pp > \mv{0}.$ Since with $\mv{q} \in \RV(1/\mv{\alpha}),$  from the definition of
$\mv{\kappa}(\cdot)$ in (\ref{IS-transf}), with $(u_n/l_n)\to s$,
$\lim_{n \rightarrow \infty}\mv{\kappa}(\mv{q}((t(u_n)\pp_{u_n})) =
\alpha_\ast/(\rho \mv{\alpha});$. 
which proves the first part of the corollary. 

\noindent \textbf{2) } When $\pp> \mv{0}$ from the definition of $\mv{T}$ in (\ref{IS-transf}),
\[\mv{T}_{u_n}\big(\mv{q}(t(u_n)\pp_n) \big) = \qq(t(u_n)\pp_n)
  (s+o(1))^{\frac{\alpha_\ast}{\rho
        \mv{\alpha}}}(1+o(1)) = \mv{q}(t(u_n)) (s^{\alpha_*/\mv{\alpha}\rho} \pp^{1/\mv{\alpha}})(1+o(1)) \text{
  as } n \rightarrow \infty.\]
 Suppose $\pp = (p_1,\ldots,p_d)$ is such that the
subset of indices $I := \{i:p_i = 0\}$ is non-empty. Since
$\pp \neq \mv{0},$ $I$ is a strict subset of $\{1,\ldots,d\}.$ 
Since $p_i = 0$ for $i \in I,$ we have
$q_i(t(u_n)p_{i,u_n}) = o(q_i(t(u_n))).$ Then, $\mv{T}_{u_n,i}\big({\mv{q}}(t(u_n))\pp_{u_n}^{1/\mv{\alpha}} \big) = o(q_i(t(u_n)))$ 
for all $i\in I$. Then, the sequence of functions 
\[
\frac{\mv{T}_{u_n}(\mv{q}(t(u_n)\pp_{u_n})}{\mv{q}(t(u_n))} \to \pp^{1/\mv{\alpha}}s^{\alpha_\ast/\mv{\alpha}\rho}= [\mv{T}_s^*(\pp)]^{1/\mv{\alpha}}.
\]
From \cite[Theorem 7.14]{rockafellar2009variational} and the above continuous, uniformly over compact subsets of $\R^d_+\setminus\{ 0 \}$, $\mv{T}_u(\qq(t(u)\pp))= \qq(t(u)\pp)s^{\alpha_{\ast}/\mv{\alpha\rho}}$. Thus, the statement in \eqref{eqn:Conc-Preserving-Ts} holds uniformly over compact $A\subseteq \R^d_+\setminus\{0\}\}$
and $\{\mv{T}_{u}^{(1)}\}_{u>0}$ is therefore rate function preserving. \hfill $\Box$

\noindent \textbf{Proof of Corollary~\ref{cor:Self-struc-2}:} Suppose that $\pp_u \to \pp>0$. Observe that
\begin{align}
    \mv{T}^{(2)}(\mv{q}(t(u)\pp_u)) &= \mv{q}(t(u)\pp_u)\left(s+o(1)\right)^{\frac{\log(1+\mv{q}(t(u)\pp_u))}{\log u -\log (s+o(1))}}\nonumber\\
    &= \mv{q}(t(u))  \pp^{\mv{1/\mv{\alpha}}}[s+o(1)]^{\frac{\log(1+u^{{\alpha_\ast/\mv{\alpha\rho}}})}{\log u}} (1+o(1)) \nonumber\textrm{ since $\qq\in \RV(1/\mv{\alpha})$}\\
    &= \mv{q}(t(u)) \pp^{\mv{1/\mv{\alpha}}} [s^{\alpha_\ast/\mv{\alpha}\rho}+o(1)]\label{eqn:Alt_Verify}
\end{align}
The case where the set $\{i:p_i=0\} \neq \emptyset$ can the handled similar to the proof of Corollary~\ref{cor:Self-struc}.\hfill$\Box$
}

\section{Sufficient conditions and examples for the tail models considered}
\label{sec:app-eg:tails}
This section serves to illustrate the variety of multivariate
distribution families which come under the tail modeling framework
considered and to provide sufficient conditions. In Tables
\ref{tab:marginals-light} - \ref{tab:marginals-heavy} below, we
provide some examples of uni-variate distribution families which
satisfy either the marginal assumptions in Assumption
\ref{assump-marginals}, or its heavier-tailed counterpart in
Assumption \ref{assume:marginals-HT}.
\begin{table}[h!]
  \caption{Examples of some marginal distributions satisfying
    Assumption \ref{assump-marginals} and their right tail parameter
    $\alpha.$ Larger the parameter $\alpha,$ the lighter the
    respective tail is.}
\begin{center}
\begin{tabular}{ l  c c }
  \hline
  Distribution families  & Tail parameter $\alpha$\\
  \hline
  Exponential, Erlang, Gumbel, Logistic    & 1\\
  Gamma, Chi-squared, phase-type & 1\\
  Gaussian, Chi, mixtures of Gaussians, Rayleigh & 2\\
  Weibull with shape parameter $k$  & $k$ \\
  Generalized-gamma  with shape parameter $k$  & $k$ \\
  \hline
\end{tabular}
\end{center}
\label{tab:marginals-light}
\end{table}


\begin{table}[h!]
  \caption{Examples of some heavy-tailed marginal distributions
    satisfying Assumption \ref{assume:marginals-HT} and the respective
    right tail parameter $\alpha.$ Larger the parameter $\alpha,$
    relatively lighter is the respective tail.}
\begin{center}
\begin{tabular}{ l c c }
  \hline
  Distribution families &  Tail parameter $\gamma$\\
  \hline
  Lognormal & 2\\
  Generalized Pareto, Student's $t,$ Regularly varying & 1\\
  Log-Laplace, Frechet, Lomax, Log-logistic, Cauchy    & 1\\
  \hline
\end{tabular}
\end{center}
\label{tab:marginals-heavy}
\end{table}

The notion of multivariate regular variation turns out to be
convenient in understanding the condition in Assumption
\ref{assump:joint-Y} in terms of the probability density of $\XX.$ We
say that a function $f:\R^d_+ \rightarrow \R_+$ is
\textit{multivariate regularly varying} if for any sequence $\xx_n$ of
$\Real^d_+$ satisfying $\xx_n \rightarrow \xx \neq \mv{0},$
\begin{align}
  \lim_{n \rightarrow \infty} n^{-1}f(\mv{h}(n)\mv{x}_n)
   = f^\ast(\mv{x}), 
  \label{MRV-het}
\end{align}
for some limiting $f^\ast: \R^d_{+} \rightarrow (0,\infty)$ and a
component-wise increasing $\mv{h}(t) = (h_1(t),\ldots,h_d(t))$
satisfying $h_i \in \RV(1/\rho_i), \ \rho_i > 0, i = 1,\ldots,d.$ It
follows from \eqref{MRV-het} that $f^\ast(\cdot)$
satisfies 
$f^\ast(s^{1/\mv{\rho}} \mv{x}) = sf^\ast(\mv{x}).$ In the above, the
notation $s^{1/\mv{\rho}}$ is to be interpreted as the vector
$s^{1/\mv{\rho}} = (s^{1/\rho_1},\ldots,s^{1/\rho_d}).$ When referring
to \eqref{MRV-het}, we write $f \in \MRV,$ or more specifically,
$f \in \MRV(f^\ast,\mv{h})$ if there is a need to explicitly specify
the scaling functions $\mv{h}(\cdot)$ and the respective limit
function $f^\ast(\cdot).$

For instance, the function $f:\R_+^2 \rightarrow \R_+$ defined as in,
$f(\mv{x}) = x_1^{2.5}(1-\exp(-x_2)) + x_2^{0.5},$ satisfies
$f \in \MRV(f^\ast,\mv{h})$ with
$f^\ast(\mv{x}) = x_1^{2.5} + x_2^{0.5}$ and
$\mv{h}(t) = (t^{1/2.5},t^2).$ See Table \ref{tab:multivariate} below
for some useful examples which arise in the context of tail modeling
and \cite{resnick2007heavy} for a detailed treatment. The following
result on $\MRV$ functions is required in the subsequent proofs. Let
$\mathrm{Id}(\xx) = \xx.$

\begin{lemma}
  \label{lemma:RV-New}
  Suppose $\mv{g}:\Real^d_+ \rightarrow \Real_+^d$ is such that
  $g_i \in \RV({\beta}_i) $ is monotone, for $i \in \{1,\ldots,d\}$.
  For $s>0$, define $\tilde{\mv{g}}_s(t) := \mv{g}(st)$ and
  $ \mv{v}_{s,\mv{\beta}}(\xx) := s\xx^{1/\mv{\beta}}.$ Then for
  $f:\R^d_+ \to \R$, we have
  $f\circ {\mv{g}} \in \MRV(u^\ast, \mathrm{Id}),$ so long as
  $f \in \MRV(u^\ast \circ \mv{v}_{s,\mv{\beta}}, \tilde{\mv{g}}_s)$
  for some $s > 0.$
\end{lemma}
\noindent\textit{Proof.}
Consider any $M > 0.$ Since $g_i\in \RV(\beta_i), i = 1,\ldots,d,$ and are monotone, we
have $g_i(tx)/g_i(t) \to x^{\beta_i},$ uniformly over $x\in[0,M].$
Consequently for $s>0$ and $\xx_n \rightarrow \xx \in [0,M]^d,$ we
obtain
$\mv{g}(sn\cdot s^{-1}\xx_n)/\mv{g}(sn) =
(s^{-1}\xx)^{\mv{\beta}}(1+o(1)).$ Therefore,
\begin{align*}
  n^{-1}(f \circ \mv{g}) (n\xx_n)
  &= n^{-1}f\left(\frac{\mv{g}(ns \cdot s^{-1}\xx_n)}{\mv{g}(ns)} \mv{g}(ns)\right)
    =\frac{f\left((s^{-1}\xx)^{\mv{\beta}}(1+o(1))\mv{g}(ns)\right)}{n}
    = u^\ast(\xx)(1+o(1)), 
\end{align*}
where the last equality follows from
$f\in \MRV(u^\ast\circ \mv{v}_{s,\mv{\beta}} ,
\tilde{\mv{g}}_s)$. \hfill$\Box.$

Suppose  that the support of $\XX$ contains $\R^d_+$. Propositions \ref{prop:jointpdf-X} - \ref{prop:jointpdf-X-HT} below
give sufficient conditions 
under which Assumption \ref{assump:joint-Y} is satisfied.


\begin{proposition}[\textnormal{Sufficient conditions on the
    density of $\XX$}]
  Suppose that the marginal distributions of $\XX$ satisfy either
  Assumptions \ref{assump-marginals} and \ref{assume:mhr}, and the
  density of $\XX$ when written in the form,
  \begin{align}
  f_{\mv{X}}(\mv{x}) = \exp(-\psi(\mv{x})), \quad \text{ for } \xx \in \R_d^+,
    \label{X-pdf-form}
  \end{align}
  satisfies $\psi \in \textnormal{\MRV}(\psi^\ast,\mv{h}).$ Then $\XX$
  satisfies Assumptions \ref{assump:joint-Y}. In particular, the
  hazard functions $\mv{\Lambda} = (\Lambda_1,\ldots,\Lambda_d)$ in
  Assumption \ref{assump-marginals} and the limiting function
  $I(\cdot)$ in Assumption \ref{assump:joint-Y} are related to
  $\mv{h}$ and $\psi^\ast$ as follows: there exists
  $\mv{c} \in \mathbb{R}^d_{++}$ such that
  \begin{align*}
    I(\xx) = \psi^\ast(\mv{c}\xx^{1/\mv{\alpha}}) \quad \text{ and }
    \quad \mv{\mv{h}}(\xx) = \mv{q}(\xx)(\mv{c}^{-1} + o(1)),
  \end{align*}
  as $\Vert \xx \Vert \rightarrow \infty,$ and
  $\mv{\alpha} = (\alpha_1,\ldots,\alpha_d) \in \Real^d_{++}$ is such
  that $h_i \in \RV(1/\alpha_i), i = 1,\ldots,d.$
  \label{prop:jointpdf-X}
\end{proposition}
\noindent\textit{Proof.}  Since $\YY = \mv{\Lambda}(\XX)$ and
$\XX =\qq(\YY),$ we have \eqref{pdf-Y} satisfied with
$\varphi(\yy) = \psi(\qq(\yy))$ and
$p(\yy) = \prod_i[\lambda_i(q_i(y_i))]^{-1},$
as a consequence of change of variables. Observe that
$\Lambda_i(t) = \int_{-\infty}^{t} \lambda_i(x)dx$, for monotone
$\lambda_i$. Therefore using \cite[Proposition
B.1.9(11)]{deHaanFerreira2010},
${t \lambda_i(t)}/{\Lambda_i(t)} \to \alpha_i$ and
$\lambda_i\in \RV(\alpha_i-1)$. This implies
$\log \lambda_i\circ q_i\in \RV(0)$. Since the support of $\XX$
contains $\R^d_+$, whenever $\yy_n\to\yy\in \R_d^+$,
$n^{-\epsilon}\log p(n\yy_n)= o(1)$ for all $\epsilon>0$.  Therefore
Assumption~\ref{assump:joint-Y} holds if
$\varphi \in \MRV(I,\textrm{Id})$.

To see the latter, recall that $\varphi(\yy) =
\psi(\qq(\yy))$. Substituting $\mv{g}(\yy) = \qq(\yy)$ and
$\mv{\beta} = \aalpha$ in Lemma~\ref{lemma:RV-New}, a sufficient
condition for $\varphi \in \MRV(I,\textrm{Id})$ is that
$\psi \in \MRV(I\circ \mv{v}_{s,\mv{1/\alpha}}, \qq_s),$ for some
$s > 0.$ Under the proposition assumptions,
$\psi \in \MRV(\psi^*,\mv{h})$. Equating parameters,
$I(\xx) = \psi^*(\mv{c} \xx^{\aalpha})$, and
$\mv{h}(\xx) = \qq(\xx)(\mv{c}^{-1} +o(1))$, where
$\mv{c} =s^{-\aalpha}\in \R_{++}^d$.\hfill$\Box$

To identify sufficient conditions in the presence of heavier tailed
distributions, let $\mathcal{L}$ denote the collection of
indices of the components $(X_1,\ldots,X_d)$ which satisfy the lighter
tailed assumption in Assumption \ref{assump-marginals}. For
$i \notin \mathcal{L},$ we have the respective $X_i$ satisfying the
heavier tailed assumption in Assumption \ref{assume:marginals-HT}. Let
$\mv{Z} = (Z_1,\ldots,Z_d)$ be defined as follows:
\begin{align}
  Z_i = 
  \begin{cases}
    \log (1 + X_i) \quad &\text{ if } X_i > 0 \text{ and } i \notin \mathcal{L},\\
    X_i \quad &\text{ otherwise.}
  \end{cases}
  \label{defn:Z}
\end{align}

\begin{proposition}[\textnormal{Sufficient conditions in the presence
    of heavier tails}]
  Suppose that the marginal distributions of $\XX$ satisfy Assumptions
  \ref{assume:mhr} and \ref{assume:marginals-HT}. Let the probability
  density of $\ZZ = (Z_1,\ldots,Z_d)$ in (\ref{defn:Z}), when written
  in the form,
  \[f_{\mv{Z}}(\mv{z}) = \exp(-\hat{\psi}(\mv{z})), \quad \text{ for }
    \zz > 0,\] satisfy
  $\hat{\psi} \in \textnormal{\MRV}(\psi^\ast,\mv{h}).$ Then $\XX$
  satisfies Assumption \ref{assump:joint-Y}. In particular, hazard
  function $\mv{\Lambda}$ and the limiting $I(\cdot)$ in Assumption
  \ref{assump:joint-Y} are related to $\mv{h}$ and $\psi^\ast$ as
  follows: there exists $\mv{c} \in \mathbb{R}^d_{++}$ such that
    \begin{align*}
    I(\xx) = \psi^\ast(\mv{c}\xx^{1/\mv{\alpha}}) \quad \text{ and }
    \quad {h}_i(x_i) =
    \begin{cases}
      {q}_i(x_i)(c_i^{-1} + o(1)) \quad & \text{ if } i
      \in \mathcal{L},\\
      \log({q}_i(x_i))(c_i^{-1} + o(1)))
      &\text{otherwise},
    \end{cases}
  \end{align*}
  as $\Vert \xx \Vert \rightarrow \infty,$ and
  $\mv{\alpha} = (\alpha_1,\ldots,\alpha_d) \in \Real^d_{++}$ is such
  that $h_i \in \RV(1/\alpha_i), i = 1,\ldots,d.$
   \label{prop:jointpdf-X-HT}    
\end{proposition}
\noindent\textit{Proof of Proposition \ref{prop:jointpdf-X-HT}.}
For $i \in \{1,\ldots,d\},$ let $\tilde{\Lambda}_{i}$ and
$\tilde{\lambda}_{i}$ denote the hazard function and hazard rate of
$Z_i$, respectively.  Let $\tilde{q}_i := \tilde{\Lambda}_i^{\leftarrow}.$ To
rewrite the density of $\YY$ in terms of that of $\ZZ,$ see that
$\tilde{\Lambda}_{i}(z) = \Lambda_i(\exp(z)-1),$ when 
$i \notin \mathcal{L}.$
Using a change of variables, 
\[
  f_{\YY}(\yy) = \frac{1}{\prod_{i=1}^d \tilde{\lambda}_{i}(\tilde{q}_{i}(y_i)) } \exp
  \left(-\hat{\psi}(\tilde{\qq}(\yy))\right). 
\]
Recall that under the assumptions of the proposition,
$\tilde{\Lambda}_{i} \in \RV(\alpha_i)$ and the support of $\ZZ$
contains $\R_d^+$.  Since
$\tilde{\Lambda}_i(x) \sim \Lambda_i(\exp(x))$ as
$x \rightarrow \infty,$ we obtain the desired conclusion following the
steps in the proof of Proposition~\ref{prop:jointpdf-X} with
$p(\yy) = {1}/{\prod_{i=1}^d \tilde{\lambda}_{i}(\tilde{q}_i(y_i) )}$
and $\varphi(\yy) =\hat{\psi}(\tilde{\mv{q}}(\yy))$. \hfill$\Box$

\begin{remark}
  \textnormal{If the density of $\XX$ is written in the form
    (\ref{X-pdf-form}) in the positive orthant, then the respective
    density for $\ZZ$ in the positive orthant is given by
    $f_{\ZZ}(\zz) = \exp(-\hat{\psi}(\zz)),$ where
\begin{align*}
  \hat{\psi}(\zz) :=  \psi \circ \mv{E} (\zz) - \mathbf{1}_{\mathcal{H}}^\intercal
  \zz,
\end{align*}
with the map
$\mv{E}: (x_1,\ldots,x_d) \mapsto \big(E_1(x_1),\ldots, E_d(x_d)\big)$
and the vector $\mathbf{1}_{\mathcal{H}} \in \mathbb{R}^d_+$ defined
as follows:
\begin{align}\label{indicate-HT}
  E_i(x_i) :=
  \begin{cases}
    x_i  &\text{ if } i \in \mathcal{L},\\
    \exp(x_i)-1 &\text{ if } i \notin \mathcal{L},
  \end{cases}
\quad \quad \text{ and } \quad \quad 
  \mathbf{1}_{\mathcal{H}} =
  \begin{cases}
    0  &\text{ if } i \in \mathcal{L},\\
    1  &\text{ if } i \notin \mathcal{L}. 
  \end{cases}                  
\end{align}
Then one can restate the condition in
Proposition~\ref{prop:jointpdf-X-HT}, directly in terms of density of
$\XX,$ as follows: If
$\hat{\psi}(\zz) = \psi \circ \mv{E} (\zz) -
\mathbf{1}_{\mathcal{H}}^\intercal \zz \in
\textnormal{\MRV}(\psi^\ast,\mv{h}),$ then the conclusion in
Proposition \ref{prop:jointpdf-X-HT} holds.}
\end{remark}

\begin{example}[Multivariate $t$ distribution]\label{e.g.t-dist}
  \textnormal{Suppose $\XX$ is distributed according to multivariate
    $t$ distribution with density,
    $ f_{\XX}(\xx) = c_{\rho} \exp(-\psi(\xx)),$
    where
    \[\psi(\xx) =
      \frac{\rho+d}{2}\log\left(1+{\frac{\xx^{\intercal}\Sigma^{-1}\xx}{\rho}}\right),\]
    $\rho$ is a suitable positive integer and
    $c_\rho$ is the respective normalizing constant.  Since the
    marginals of a multivariate
    $t$ distribution are heavy-tailed, we have $\mv{E}(\xx) = \exp(\xx)-\mv{1}.$
    With $\ZZ = \log (1+\XX)$, $f_{\ZZ}(\zz) = \e^{-\|\zz\|_1 }
    f_{\XX}(\e^{\zz}-1),$ due to change of variables.  Then in this
    case,
\[
  \hat{\psi}(\zz) = -\|\zz\|_1 + \frac{\rho+d}{2}
  \log\left(1+\frac{(\e^{\zz}-1)\Sigma^{-1}(\e^{\zz}-1)}{\rho}\right).\]
Thus $\hat{\psi} \in \textnormal{\MRV}(\psi^\ast,
\mv{h}),$ where $\psi^\ast(\zz) := (\rho+ d)\|\zz\|_{\infty} - \Vert
\zz \Vert_1$ and $\mv{h}(t) = t\mv{1}.$ \hfill$\Box$ }
\end{example}

Example \ref{eg:gauss-copula} in Section \ref{sec:Mod-Framework} and
Example \ref{e.g.t-dist} above serve as illustrations for the sequence
of steps involved in verifying memberships of commonly used
multivariate distributions and copula models in the considered tail
modeling framework. Table \ref{tab:multivariate} below is intended to
serve as a reference for identifying the limiting function
$I(\cdot)$ in Assumption \ref{assump:joint-Y} (or) the respective
function
$\psi^\ast(\cdot)$ which arise in the characterizations in
Propositions \ref{prop:jointpdf-X} - \ref{prop:jointpdf-X-HT}.  Thanks
to standardization, the limiting function
$I(\cdot)$ is unique despite
$\psi^\ast(\cdot)$ depending on the specific scaling function
$\mv{h}(\cdot)$ employed.

\begin{table}
  \caption{Some commonly used density families which satisfy
    Assumptions \ref{assump-marginals} - \ref{assump:joint-Y}, along
    with their limiting functions $I(\cdot)$ and $\psi^\ast(\cdot)$
    (where applicable).  Certain constants are written as $c$ or
    $\mv{c}$ (if $\mv{c} \in \R^d)$ to minimize clutter}
\begin{center}
  \resizebox{\textwidth}{!}{
    \begin{tabular}{ l c c  c }
      \hline
      Density families &  Limiting function   &            Respective   &  Limiting $I(\xx)$ \\ 
                       &    $\psi^\ast(\xx)$ &  copula family  & in Assumption \ref{assump:joint-Y}b \\
      \hline \hline 
      \textbf{Elliptical densities}
      \tablefootnote{(Footnotes for Table \ref{tab:multivariate}) $\mu,\Sigma$  are the location, scale parameters, $\mv{U}$ is  uniformly distributed on the unit sphere in $\R^d$ and is independent of $\mathcal{R};$ includes the special cases of factor models if $\Sigma = \Gamma^\intercal \Gamma + \sigma^2\mathbf{I}_d$ for some factor matrix $\Gamma \in \R^{k \times d}$ with $k < d,$ and graphical models where $\mathcal{R}$ is Gaussian and the inverse covariance matrix is sparse.} & \hspace{-125pt}  & \hspace{-180pt} given by $\XX \overset{d}{=} \mv{\mu} + \mathcal{R}\Sigma^{1/2} \mv{U}:$  &  \\
      \vspace{-6pt}  & & & \\
      1) Multivariate normal:  & ${\xx^\intercal \Sigma^{-1}\xx}$ & Gaussian & $(\xx^{1/2})^\intercal R^{-1} \xx^{1/2},$ with  \\
      mean  = $\mv{\mu},$ covariance $\Sigma$ & $ $ & copula & $R = $  correlation matrix  \\
      \vspace{-6pt}  &   & & \\
      2) Multivariate $t-$ & $(\rho+{d})\Vert \xx\Vert_{\infty} - \Vert \xx \Vert_1$ & Student-$t$ &$(1+\frac{d}{\rho})\Vert \xx\Vert_{\infty} - \frac{1}{\rho}\Vert \xx \Vert_1$ \\
      distribution\tablefootnote{student-distributed with $d,\rho$ degrees of freedom}: $\mathcal{R} \sim F_{d,\rho}$  &   & copula &  \\
      \vspace{-6pt}  & & & \\
      3) $\mathcal{R}$ is light-tailed with & ${(\xx^\intercal \Sigma^{-1}\xx)^{k/2}}$ & Elliptical  & ${((\xx^{1/k})^\intercal R^{-1}\xx^{1/k})^{k/2}},$\\
      p.d.f. $f_{\mathcal{R}}(r) = \exp(-g(r)),$  & $ $&  copula family & with $R = $  correlation \\
      for  $g \in \RV(k),$ $k > 0$ &  &  & matrix \\
      \vspace{-6pt}  & & & \\
      4) $\mathcal{R}$ is heavy-tailed with & $\Vert \xx \Vert_\infty^k$ & Elliptical & $\Vert \xx \Vert_\infty$\\
      p.d.f. $f_{\mathcal{R}}(r) = \exp(-g(r)),$  &  &  copula family & \\
      for $g \circ \exp \in \RV(k),$ $k > 1$ &  &  & \\
      \vspace{-6pt}  & & & \\
      \hline \hline
      \vspace{-6pt}  & & & \\
      \textbf{Exponential family}  & \hspace{-65pt}with p.d.f.  $f_{\XX}(\xx) \propto$   &\hspace{-85pt} $g(\xx)\exp\big( \mv{\eta}^\intercal \mv{T}(\xx) \big)$ & \\
      \vspace{-6pt}  & & & \\
      Minimal, light-tailed:  & $  T^\ast_\eta(\xx)$   & -  & $T^\ast_\eta(\mathbf{c}\xx^{1/k})$ with \\
      $\quad \mv{\eta}^\intercal \mv{T} \in \MRV(T^\ast_\eta,n\mv{1})$  & &  & $k$ as in $f \in \RV(k)$\\
      for some $T^\ast_\eta(\cdot),$ $g \in \RV$  & & & for $f(n) = \mv{\eta}^\intercal \mv{T}(n\mathbf{1})$\\
      \vspace{-6pt}  & & & \\
      Generalized linear & \hspace{-75pt} models: $\mv{\xi} = (\mv{X},\YY)$ with  & \hspace{-60pt}    $f_{\YY\, \vert\, \XX}(\yy\, \vert\, \xx) \propto b(\yy)$ & \hspace{-80pt} $\exp\left(\mv{\ell}^{-1}(\mv{\beta}^T \mv{x})^\intercal T_{\yy}(\yy)\right)$ \\
      $\XX$ with p.d.f in & \hspace{-133pt} exponential family & & \\
      $b\in \MRV$, $f_{\xx}(\xx) = \mathrm{e}^{-\psi(\xx)}$ &  & & $\psi^*\circ \mv{\pi}_1^{-1}$, with\\
      $\psi + u \in \MRV(\psi^*, (h,r))$ &$\psi^*(\xx,\yy)$ & - &$\mv{\pi}_1(\xx) = (\xx^{\mv{\alpha}}, \yy^{\mv{\beta}})$    \\
      $u(\xx,\yy)= \mv{\ell}^{-1}(\mv{\beta}^T \mv{x})^\intercal T_{\yy}(\yy)$, & &  &\\
      $h_i\in\RV(\alpha_i)$, $r_i\in\RV(\beta_i)$ & & & \\
      \vspace{-6pt}  & & & \\
      \hline \hline
      \vspace{-6pt}  & & & \\
      \textbf{Logconcave densities}  & \hspace{-50pt} with p.d.f.  $f_{\XX}(\xx) =$ &\hspace{-95pt} $\exp(-\psi(\xx))$ & \\ 
      \vspace{-6pt}  & & & \\
      convex $\psi \in \MRV$ &  $\psi^\ast(\cdot)$ limit with  & & $c\psi^\ast \circ \mv{\pi}^{-1}$ \\
                       &scaling $\mv{h}(t) = \mv{\Lambda}^{-1}(t)$ & & \\
  \vspace{-6pt}  & & & \\
  \hline \hline
  \vspace{-6pt}  & & & \\
  \textbf{Archimedian copula }   & \hspace{-90pt}family with & \hspace{-190pt} $C(\mv{u}) = $ & \hspace{-250pt} $\phi^{-1}(\phi(u_1) + \ldots +  \phi(u_d))$ \\ 
  \vspace{-6pt}  & & & \\
  \hspace{50pt} -  & -  & Gumbel  &  $\|\xx\|_{\theta}$\\
                   &   & $\phi(u) = (-\log u)^{\theta}$ & \\
  \vspace{-6pt}  & & & \\
  \hspace{50pt} -  & -  & Clayton,  & $(1+\theta{d}) \Vert \xx \Vert_\infty - {\theta}\Vert \xx \Vert_1$\\
                   &   & $\phi(u) = \frac{(t^{-\theta} - 1)}{\theta}$ & \\
  \vspace{-6pt}  & & & \\
  \hspace{50pt} -  & -  & Independence  & $\Vert \xx \Vert_1$\\
  \hline \hline
  \vspace{-6pt}  & & & \\
      Mixtures of $K$ normal &$\min_{i=1}^{K} \xx^\intercal \Sigma_k^{-1}\xx$ & -& $(\xx^{1/2})^\intercal R^{-1}_{k^*}\xx^{1/2}$,   \\
      variables with     & & &where $k^\ast$ minimizes \\
      covariances $\Sigma_1,\ldots,\Sigma_k$ & & & $\{(\xx^{1/2})^\intercal R_k^{-1}\xx^{1/2}:k \leq K\}$\\
      \hline 
\end{tabular}}
\end{center}
\label{tab:multivariate}
\end{table}

\section{Proofs of results on application to
  credit risk}\label{proofs:appl}

Let $[J]$ denote the set $\{1,\ldots,J\}.$ For any $\varepsilon > 0,$
$j \in [J], I \subseteq [J],$ and $\xx \in \mathbb{R}^d_+,$ we define,
\begin{alignat}{2}
  &C_{\xx,\varepsilon} := \{i \in [m]: W_{t(i)}(\xx,\mv{v}_i) >
  \gamma  (1-\varepsilon)\} \label{eqn:sepsilon}\qquad\qquad
  &s_\varepsilon(\xx) := \frac{1}{m}\sum_{i \in C_{\xx,\varepsilon}} e_i,\\
  &e_m(I) := \frac{1}{m} \sum_{i \in [m]: t(i) \in I} e_i,
  & e_\infty(I) := \sum_{j\in I} \bar{c}_j, \nonumber\\
  &\underline{w}_j(\xx) := \min_{i \in [m]: t(i) = j}
  W_{t(i)}(\xx,\mv{v}_i), \text{ and } &e^*_m = \max_{I \notin
    \mathcal{J}_m} e_m(I). \label{eqn:estar}
\end{alignat}
 
\begin{lemma}
  \label{lemma:Small-conc}
  $P(\mathcal{E}_m\, \vert \, \XX) \leq \exp(-0.5 m \gamma
  \varepsilon_m e_0^{-1} \left( q\bar{e}_m - s_{\varepsilon_m}(\XX)
  \right)^+ + \exp(-0.5 \gamma \varepsilon_m)),$ almost surely, where $ s_{\varepsilon_m}(\cdot)$ is as
  defined in \eqref{eqn:sepsilon}.
\end{lemma}
\textit{Proof.}  For any $\xx \in \mathbb{R}^d_+,$ let
$P_{\xx}(\cdot)$ denote the conditional law of the default variables
$(Y_1,\ldots,Y_m)$ given $\XX = \xx;$ let $E_{\xx}[\cdot]$ denote the
associated expectation. For any $\lambda > 0,$ we obtain from Markov's
inequality,
$P_{\xx}(L_m > q\bar{e}_m) \leq \exp\left( -m\lambda q \bar{e}_m +
  \log E_{\xx}[\exp(m \lambda L_m)] \right),$ due to the independence
of the default variables $Y_1,\ldots,Y_m$ given $\XX$.
  Letting
  $\psi_m(\lambda,\xx) := \frac{1}{m} \sum_{i=1}^m\log E_{\xx} \left[
    \exp(\lambda e_iY_i)\right]$ and
  $g_m(\xx) = \sup_{\lambda > 0} \left\{\lambda q \bar{e}_m -
    \psi_m(\lambda,\xx) \right\},$ we obtain 
  \begin{align}
    P_{\xx}(L_m > q\bar{e}_m) \leq \exp\left( -m g_m( \xx)\right),
    \label{chernoff}
  \end{align}
  as a consequence.  With $P_{\xx}(Y_i=1)$ given as in
  \eqref{loan-def-prob}, we have
  \begin{align*}
    \psi_m(\lambda,\xx) 
    &= \frac{1}{m} \sum_{i=1}^m \log \left(1 +
      \frac{\exp(\lambda e_i) - 1}{1 + \exp(\gamma - W_{t(i)}(\xx,v_i))}
      \right)\\
    &\leq  \frac{1}{m} \sum_{i \in C_{\xx,\varepsilon}} \log \left(1 +
      {\exp(\lambda e_i) - 1} \right) +
      \frac{1}{m} \sum_{i \notin C_{\xx,\varepsilon}} \log \left(1 +
      \frac{\exp(\lambda e_0) - 1}{1 +
      \exp(\gamma - \gamma (1-\varepsilon))}\right),
  \end{align*}
  where we have used that
  $W_{t(i)}(\xx,v_i) \leq \gamma (1-\varepsilon)$ for every
  $i \notin C_{\xx,\varepsilon}$ and that $e_i \leq (0,e_0].$ Then,
  \begin{align*}
    \psi_m(\lambda,\xx) &\leq \lambda s_\varepsilon(\xx) + \log \left( 1 + \exp(\lambda e_0 - \gamma \varepsilon)\right), 
  \end{align*}
  from the definition of $s_\varepsilon(\xx).$ If
  $q\bar{e}_m > s_\varepsilon(\xx),$ we obtain a lower bound for
  $g_m(\xx)$ by 
  considering the specific value
  $\lambda = \lambda^\ast_m(\xx) := 0.5 e_0^{-1} \gamma \varepsilon$ as
  below:
  \begin{align*}
    g_m(\xx) &\geq \lambda^\ast_m(\xx) q \bar{e}_m - \lambda^\ast_m(\xx) s_\varepsilon(\xx)
    - \log \left( 1 + \exp(\lambda^\ast_m(\xx) e_0 - \gamma \varepsilon)\right)\\
    &\geq  0.5 e_0^{-1}\gamma \varepsilon \left( q \bar{e}_m -
    s_\varepsilon(\xx) \right) - \exp(-0.5 \gamma\varepsilon).  
  \end{align*}
  If $q\bar{e}_m \leq s_\varepsilon(\xx),$ a trivial bound
  $g_m(\xx) \geq 0$ is obtained by letting
  $\lambda = \lambda^\ast_m(\xx) = 0.$ Thus,
  $g_m(\xx) \geq 0.5 e_0^{-1}\gamma \varepsilon \left( q \bar{e}_m -
    s_\varepsilon(\xx) \right)^+ - \exp(-0.5 \gamma \varepsilon_m).$
  Combining this with (\ref{chernoff}) yields the desired
  result. \hfill$\Box$

 Recall from the definitions
that $\mathcal{J}_m := \{ I \subseteq [J]: e_m(I) \geq q\bar{e}_m\}$
and $\mathcal{J} := \{ I \subseteq [J]: e_\infty(I) \geq q\bar{e}\}.$

\begin{lemma}
  \label{lemma:Jmconv}
  There exists a positive integer $m_0$ such that
  $\mathcal{J}_m = \mathcal{J}$ for all $m>m_0.$ Consequently,
  \begin{align*}
    \inf_{m > m_0} (q\bar{e}_m - e_m^\ast) > 0 \quad \text{ and } \quad
    \inf_{m > m_0, I \in \mathcal{J}_m} (e_m(I) - q\bar{e}_m) > 0.
  \end{align*}
\end{lemma}
\textit{Proof.}
  Notice that under the model assumptions stated in Section
  \ref{sec:CR}, $\bar{e}_m\to\bar{e}$ and
  $q\bar{e} \not\in\{e_\infty(I):I \subseteq [J]\}$. The latter implies
  that there exists some $\delta_1>0$ such that
\begin{equation}\label{eqn:Jdiff}
  \min_{I \subseteq [J]} |e_\infty(I) -q \bar{e}| > \delta_1.
\end{equation}
Further, for all $j\in[J],$ $m^{-1}\sum_{i:t(i)=j} e_i \to
\bar{c}_j$. Consider any $\delta \in (0,\delta_1/2).$ Due to these
convergences, 
there exists $m_0$ suitably large
such that for all $m > m_0,$
\begin{align}
  e_m(I) \in (e_\infty(I) - \delta , e_\infty(I) + \delta), \text{ and  } \bar{e}_m\in(\bar{e}-\delta,\bar{e}+\delta).
  \label{conv-emI}
\end{align}
uniformly over $I \subseteq [J].$ Since $q \in (0,1),$ the above
bounds imply that $e_\infty(I) \geq q\bar{e} -2\delta$ for any
$I \in \mathcal{J}_m.$ With $\delta<\delta_1/2$,
$e_\infty(I)>q\bar{e}-\delta_1$, or equivalently,
$e_{\infty}(I)\geq q\bar{e}$ (due to (\ref{eqn:Jdiff})). Therefore
$\mathcal{J}_m \subseteq \mathcal{J}$ for all $m > m_0.$ Similarly if
$I \in\mathcal{J}$ is such that $e_\infty(I) \geq q\bar{e},$ we
automatically have $e_\infty(I)\geq q\bar{e}+\delta_1$ (due to
(\ref{eqn:Jdiff})). Similarly from (\ref{conv-emI}),
$e_m(I)\geq q\bar{e}_m + \delta_1 - 2\delta \geq q\bar{e}_m.$
Therefore $\mathcal{J} \subseteq \mathcal{J}_m$ for all $m > m_0.$
Combining the two inclusions result in the desired conclusion that
$\mathcal{J} = \mathcal{J}_m.$

With $e_m^\ast$ defined as in (\ref{eqn:estar}) and
$\mathcal{J}_m = \mathcal{J}$ for all $m > m_0,$ we have
$e_m^\ast = \max_{ I \not\in\mathcal{J}} e_m(I)$ for all $m > m_0.$
Again for $\delta \in (0,\delta_1/2),$ we obtain from (\ref{conv-emI})
that
$e_m^\ast = \max_{ I \not\in\mathcal{J}}e_m(I) < \max_{I \not
  \in\mathcal{J}}e_\infty(I) +\delta_1/2$ for all $m > m_0.$ However
$\max_{I \not \in\mathcal{J}}e_\infty(I) < q\bar{e} - \delta_1$ due to
(\ref{eqn:Jdiff}).Therefore $e_m^\ast < q\bar{e} - \delta_1/2.$ Since
$q\bar{e} < q\bar{e}_m + \delta,$ we arrive at the conclusion that
$q\bar{e}_m - e_m^\ast > \delta_1/2 - \delta,$ for all $m > m_0.$
Recalling that $\delta < \delta_1/2,$ the first inequality in the
lemma statement stands verified. Observing that $e_m(I) > q\bar{e}_m$
for any $I \in \mathcal{J}_m,$ the second inequality follows from
completely analogous arguments. \hfill$\Box$

\noindent \textbf{Proof of Theorem~\ref{thm:credit-risk-asymp}: } We
treat the terms in the bound,
$P(\mathcal{E}_m) \leq P(\mathcal{A}_m) +
P(\mathcal{E}_m \setminus \mathcal{A}_m),$ separately.\\
\noindent \textbf{Step 1) To obtain an upper bound for
  $P(\mathcal{A}_m)$:} Define
$\Lcr^*(\xx) := \max_{I \in \mathcal{J}} \min_{k \in I}
W_{k}^\ast(\mv{q}^\ast \xx^{\mv{1}/\mv{\alpha}}).$ For any sequence
$\{\xx_n\} \subset \Real^d_+$ satisfying
$\xx_n \rightarrow \xx \neq 0,$ we first show that
$n^{-\rho}\Lcr(n\xx_n) \to \Lcr^*(\xx).$ To see this, recall the
definition of $\Lcr(\xx)$ in (\ref{loss-cr}).
From the continuous convergence of $W_i(\cdot,\mv{v})$ in
Assumption~\ref{assume:Wj}, 
$\sup_{I\subset[J], t(i)\in I} | n^{-\rho} W_i(n\xx_n,\mv{v}_i) -
W_{t(i)}^*(\xx)|\to 0 \text{ as } n\to\infty.$ Consequently, for any
$I \subset [J],$
$\min_{i:t(i)\in I} n^{-\rho} W_i(n\xx_n,\mv{v}_i) \to \min_{j\in I}
W_j^*(\xx).$ Since $\mathcal{J}_m=\mathcal{J}$ for all $m > m_0$ (see
Lemma~\ref{lemma:Jmconv}),
\begin{equation}
  n^{-\rho}\Lcr(n\xx_n) \to \Lcr^*(\xx),
  \label{cconv-pcr}
\end{equation}
for any $\xx_n\to\xx\neq \mv{0}.$ Since
$\mathcal{A}_m := \{\Lcr(\XX) > c(1-\varepsilon_m)m^\eta\}$ and
$\varepsilon_m\searrow 0,$ this continuous convergence implies that
one can apply the asymptotics from Theorem~\ref{thm:Tail-asymp} to
evaluate $P(\mathcal{A}_m)$ as below:
\begin{equation}
  \label{eqn:PCR-key}
  \lim_{m\to\infty} \frac{P(\mathcal{A}_m)}{\Lambda_{\min}(c^{1/\rho}m^{\eta/\rho})} = 
  \lim_{m\to\infty} \frac{\log \Prob( \Lcr(\XX) \geq cm^\eta)}
  {\Lambda_{\min}(c^{1/\rho}m^{\eta/\rho})}  =-\Icr,
\end{equation}
where $\Icr$ is defined as in (\ref{Icr}).

\noindent \textbf{Step 2) To show
  $P(\mathcal{E}_m \setminus \mathcal{A}_m) = o(P(\mathcal{A}_m))$ :}
Recall that
$\mathcal{A}_m = \{ \Lcr (\XX) > c(1-\varepsilon_m) m^\eta\}.$ For any
$\xx \in \mathbb{R}^d_+$ such that
$\Lcr(\xx) \leq c (1-\varepsilon_m)m^{\eta},$ we have from the
definition of $\Lcr(\cdot)$ that,
  \begin{align*}
    \max_{J \in \mathcal{J}_m} \min_{j \in J} \underline{w}_j(\xx) \leq 
    c (1-\varepsilon_m) m^\eta. 
  \end{align*}
  For such $\xx,$ the collection
  $\underline{J} := \{j \in [J]: \underline{w}_j(\xx) > c
  (1-\varepsilon)m^\eta \}$ is not a member of $\mathcal{J}_m$. Hence,
  \begin{align*}
    s_{\varepsilon_m}(\xx) = \frac{1}{m} \sum_{i \in C_{\xx,{\varepsilon_m}}} e_i   \leq \frac{1}{m} \sum_{j \in \underline{J}} \sum_{\{i: t(i) = j\}}e_i   \leq e^\ast_m,
  \end{align*}
  where $e_m^\ast$ is defined as in \eqref{eqn:estar}.
  Thus we have from Lemma~\ref{lemma:Small-conc} that
  \begin{equation}
    \label{eqn:PCR-small-pre}
    P\left(\mathcal{E}_m \setminus \mathcal{A}_m \right) \leq \exp\left( -0.5 m\gamma \varepsilon_m e_0^{-1} \left( q\bar{e}_m - e^\ast_m\right)^+ +\exp(-0.5 \gamma \varepsilon_m) \right). 
  \end{equation}
  Recall that $\gamma = cm^\eta (1 + o(1))$ and $\varepsilon_m$ is
  chosen such that
  $ m^{\eta(1-\alpha_*/\rho) +1}\varepsilon_m\to \infty$. Therefore,
\[
  \lim_{m\to\infty}
  \frac{m\gamma\varepsilon_m}{\Lambda_{\min}(m^{\eta/\rho})} =
  \lim_{m \to \infty} \frac{m^{1+\eta}\varepsilon_m}{m^{\alpha_\ast
      \eta/\rho (1 + o(1))}} = \infty.
\]
Since $\inf_{m > m_0} (q\bar{e}_m - e_m^\ast) > 0$ for $m_0$ large
(see Lemma \ref{lemma:Jmconv}(b)), we obtain from
\eqref{eqn:PCR-small-pre},
\begin{equation}
  \label{eqn:PCR-small}
  \limsup_{m\to\infty} \frac{\log  P\left(\mathcal{E}_m\setminus \mathcal{A}_m\right)}
  {\Lambda_{\min}(m^{\eta/\rho})}  \leq
  \limsup_{m\to\infty}\frac{m \exp(- 0.5 \gamma \varepsilon_m) - 0.5 m \gamma \varepsilon_me_0^{-1}
      \left( q\bar{e}_m - e^\ast_m\right)^+}
  {\Lambda_{\min}(m^{\eta/\rho})} = -\infty.
\end{equation}
Combining \eqref{eqn:PCR-small} and \eqref{eqn:PCR-key}, we arrive at
$P(\mathcal{E}_m\setminus \mathcal{A}_m) = o(\Prob( \mathcal{A}_m)).$\\
\textbf{Step 3) To obtain a matching lower bound:} Choose $\delta_m$
such that $\delta_m\searrow 0$ and $\delta_m\gamma
\to\infty$. Consider any fixed
$\xx\in B_m := \{\Lcr(\xx) \geq \gamma(1+\delta_m)\}.$ Notice that
there exists a set $I \in \mathcal{J}_m$ satisfying
\[\min_{i:t(i) \in I} W_{t(i)}(\xx,v_i) > \gamma(1+\delta_m).\]
For the chosen $\xx$ and the resulting index set $I,$ we have for all
$i \in I,$
\[
P(Y_i = 1 \ \vert \ \XX =\xx) = \frac{1 }{1+ \exp(\gamma - W_{t(i)}(\xx,\mv{v}_i)) } \geq \frac{1}{1+\exp(-\gamma \delta_m)} \geq 1 - \delta,
\]
where $\delta > 0$ is suitably
small. 
The resulting conditional loss for the chosen $\xx$ is given by,
\begin{align*}
  E[L_m\mid \XX = \xx]  = \frac{1}{m}\sum_{i=1}^{m} e_i \Prob(Y_i=1\mid \XX=\xx)
  \geq \frac{1}{m}\sum_{i:t(i) \in I} e_i P(Y_i =1 \ \vert \ \XX= \xx ) \geq (1-\delta) e_{m}(I).
\end{align*}
Since $I \in\mathcal{J}_m$,
$ \inf_{m > m_0} (e_m(I) - q\bar{e}_m) > 0$ as a consequence of Lemma
\ref{lemma:Jmconv}. As a result,
$E[L_m \mid \XX = x] \geq q\bar{e}_m + \kappa,$ for some $\kappa > 0.$
Given a realization of $\XX$, $L_m = \sum_{i=1}^me_iY_i$ is the sum of
$m$ independent random variables. So for any $\varepsilon > 0$ and
$\xx \in B_m,$
\begin{align*}
  P(L_m\geq q\bar{e}_m \mid \XX=\xx) \geq
  P(L_m \geq E[L_m \mid \XX = \xx] - \kappa \mid \XX=\xx)
      \geq 1 -\varepsilon,
\end{align*}
for all $m$ sufficiently large due to concentration properties of
independent sums (see, for example, Cantelli's inequality). Therefore, 
\begin{align*}
  \Prob(\mathcal{E}_m \cap B_m)
  = \int_{\xx\in B_m} \Prob(\mathcal{E}_m\mid \XX=\xx) dF(\xx)
  \geq \inf_{\xx\in B_m}  \Prob(\mathcal{E}_m\mid \XX=\xx) \Prob(B_m)
    \geq (1-\varepsilon)\Prob(B_m).
\end{align*}
To conclude the proof, notice that
$ \Prob(\mathcal{E}_m \cap B_m) \leq P(\mathcal{E}_m) \leq
\Prob(\mathcal{A}_m) + o(\Prob(\mathcal{A}_m)), $ where
$\mathcal{A}_m:=\{\Lcr(\xx) \geq cm^{\eta}(1-\varepsilon_m)\}$ and
define $B_m:=\{\xx: \Lcr(\xx) \geq \gamma (1+\delta_m) \}$.  Since
$\varepsilon_m,\delta_m \to 0$ and $\gamma = cm^\eta(1+o(1)),$
$\log P(B_m) = \log P(\mathcal{A}_m)(1+o(1))$ as
$m \rightarrow \infty.$ Combining this with the observation in
(\ref{eqn:PCR-key}), we obtain the following from
$\Lambda_{\min}\in \RV(\alpha_*):$
\[\log P(\mathcal{E}_m) = -
  \Lambda_{\min}(c^{1/\rho}m^{\eta/\rho})(\Icr+o(1)) =
  c^{\alpha_*/\rho}\Lambda_{\min}(m^{\eta/\rho}) (\Icr+o(1)). \qquad
  \Box \]

\noindent \textbf{Proof of Proposition~\ref{prop:IS-NN}.} 
First, we write the second moment of the IS estimator as (see
\cite[Appendix, Pg. 1]{GKS2008} for details on how to arrive at the
first expression below),
\begin{align*}
  M_{2,m}
  &=\Expc\left[ \exp\left\{-m L_m \lambda_m(\XX) +
    m\psi_m({\XX},\lambda_m(\XX)) \right\}
    \frac{f_{\XX}({\XX})}{{f}_{\ZZ}({\XX})} \mathbb{I}(\mathcal{E}_m)\right]\\
  & \leq  \underbrace{\Expc\left[\frac{f_{\XX}({\XX})}{{f}_{\ZZ}({\XX})}
    \mathbb{I}(\mathcal{A}_{m}) \right]}_{I_{1,m}}  +
    \underbrace{\Expc\left[ \exp\left\{-mq\lambda_m(\XX) \bar{e}_m+
    m\psi_m(\XX,\lambda_m(\XX))\right\} \frac{f_{\XX}({\XX})}{{f}_{\ZZ}({\XX})}
    \mathbb{I}( \mathcal{E}_m \setminus \mathcal{A}_{m}) \right]}_{I_{2,m}}, 
\end{align*}
where we use the non-positivity of the term in the exponent and drop
$\mathbb{I}(\mathcal{E}_m)$ in
$\mathbb{I}(\mathcal{E}_m \cap \mathcal{A}_m)$ to arrive at the term
labelled $I_{1,m}.$ Note that $I_{1,m}$ equals the second moment of the
IS estimator in Algorithm \ref{algo:IS} when used for the problem of
estimating $P(\mathcal{A}_m).$ Here
$\mathcal{A}_m = \{ \Lcr (\XX) > u_m\}$ and
$u_m := cm^\eta(1-\varepsilon_m) \rightarrow \infty$ as
$m \rightarrow \infty.$ Recall the specific choice $u = cm^\eta$
employed in Algorithm \ref{algo-IS-PCR-NN}; in addition, $l$ is taken
to be slowly varying in $m.$ Due to the continuous convergence in
(\ref{cconv-pcr}) and these choices of $l$ and $u$, all the
requirements in Theorem~\ref{thm:Var-Red} stand satisfied under the
assumptions stated in Proposition~\ref{prop:IS-NN}. Therefore,
\begin{equation}
  \label{I1-pcr}
  \lim_{m\to\infty} \frac{\log I_{1,m}}{\log P(\mathcal{A}_m)^2} = 1,
\end{equation}
due to Theorem~\ref{thm:Var-Red}.  To bound $I_{2,m}$, we obtain the
following from the steps leading to (\ref{eqn:PCR-small-pre}):
\begin{align}
  I_{2,m} \leq \exp\left\{ -m\gamma \varepsilon_m e_0^{-1}
  \left( q\bar{e}_m - e^\ast_m\right)^+ + \exp(-0.5\gamma\varepsilon_m)\right\}
  \Expc\left[  \frac{f_{\XX}(\XX) }{f_{\XX} (\mv{T}^{-1}(\XX))}
  J(\mv{T}^{-1}(\XX))\mathbb{I}( \mathcal{E}_m \setminus \mathcal{A}_{m})\right].
  \label{inter-pcr-vr-1}
\end{align}
Since $u = cm^\eta(1+o(1))$ and $l$ is slowly varying in $u,$ we have
$(u/l)^{1/\rho} \leq (1+\delta) m^{\eta/\rho (1 + \delta)}$ as a
consequence of Potter's bounds \cite[Proposition
B.1.9(7)]{deHaanFerreira2010}. Then from the uniform bound for the
Jacobian $J(\cdot)$ in (\ref{inter-jac-light-1}), we have
$ J(\mv{T}^{-1}(\xx)) \leq m^{K},$ for some suitably large constant
$K \in (0,\infty).$ Since $ f_{\XX}(\cdot)$ is bounded below on
compact subsets of $\R_+^d,$ we have
$\frac{f_{\XX}(\xx) }{f_{\XX} (\mv{T}^{-1}(\xx))} \leq M$, for some
$M$ large enough. Combining these observations with
(\ref{inter-pcr-vr-1}),
\begin{align*}
  \limsup_{m\to\infty}\frac{\log I_{2,m}}{\Lambda_{\min}(m^{\eta/\rho})}
  &\leq \frac{ m\exp(-0.5\gamma\varepsilon_m) - m\gamma \varepsilon_m e_0^{-1}
    \left( q\bar{e}_m - e^\ast_m\right)^+  + \log M + K\log m
    + \log  P(\mathcal{E}_m \setminus \mathcal{A}_m)}
    {\Lambda_{\min}(m^{\eta/\rho})}\\
  &\leq \frac{ 2m\exp(-0.5\gamma\varepsilon_m) - 2m\gamma \varepsilon_m e_0^{-1}
    \left( q\bar{e}_m - e^\ast_m\right)^+  + \log M + K\log m}
    {\Lambda_{\min}(m^{\eta/\rho})}  = -\infty, 
\end{align*}
where the last inequality follows from (\ref{eqn:PCR-small}). Together
with (\ref{I1-pcr}) and the asymptotic for $\log P(\mathcal{E}_m)$ in
the statement of Theorem \ref{thm:credit-risk-asymp}, we thus arrive at  
\[
  \liminf_{m\to\infty}\frac{\log M_{2,m}}{\log P(\mathcal{E}_m)^2}
  \geq 1.
\]
Since the second moment $M_{2,m}$ also satisfies
$M_{2,m} \geq P(\mathcal{E}_m)^2,$ we obtain the desired
limit. \hfill$\Box$


\section{Proofs of Technical Results}
\label{sec:tech-proofs}
\noindent \textbf{Proof of Lemma~\ref{lem:Y-exp-marginals}.} Since
$Y_i := \Lambda_i(X_i)$ and $\Lambda_i(x_i) := \log (1 - F(x_i)),$ we
obtain
$\Prob(Y_i\geq y) = \Prob(\Lambda_i(X_i) \geq y) = \Prob(F_i(X_i) \geq
1-\e^{-y}) = \e^{-y}.$ \hfill$\Box$

\noindent \textbf{Proof of Lemma~\ref{lem:properties-I}.} \textbf{a)}
Observe that $I(\cdot)$ is the limit of continuously converging
functions $n^{-1}\varphi(n\xx)$. Therefore, from \cite[Theorem
7.14]{rockafellar2009variational}, $I(\cdot)$ is continuous.  Parts
\textbf{b), c), d)} of the lemma statement follow directly as a
consequence of Theorem \ref{thm:LDP} and \cite[Proposition
3]{de2016approximation}.\hfill$\Box$

\subsection{Proof of Lemma \ref{lem:non-negative-Fu}}
\label{sec:non-neg-proof}
\noindent \textbf{Step 1- Develop an asymptotic upper bound for $\mv{\psi}_u(t(u)\pp)$:} Define the function $G_u(\xx) = u^{-1} L(\xx u^{1/\rho})$. Then, under Assumption~\ref{assume:V}, $\Lev_1^+(G_u) \cap B_M\subset [\Lev_1^+(L^*)\cap B_M]^{1+\varepsilon}$ for all large enough $u$. Notice  that from the continuity and $\rho$-homogeneity of the limit $L^*$, the 1-level set of $L^*$ is disjoint from $B_{\delta_1-\varepsilon}(\mv{0})$, for some small enough $\delta_1,\varepsilon$.  Thus, for all large enough $u$, $G_u(\xx) \geq 1 \implies \|\xx\|_\infty\geq \delta_1$ for some $\delta_1>0$. Write
\[
L_u(\pp) = u^{-1}{L\left(\frac{\qq(t(u)\pp)}{u^{1/\rho}} u^{1/\rho}\right)}= G_u\left(\frac{\qq(t(u)\pp)}{u^{1/\rho}}\right).
\] 
Then for all large enough $u$, $\Lev_1^+(L_u) \subseteq \{\pp: \|\qq(t(u)\pp)\|_\infty \geq \delta_1 u^{1/\rho}\}$.
 Since $(u/l) = o(u)$, for all large enough $u$, $\|\qq(t(u)\pp)\|_\infty> 1/c_{\rho}(u)$. Recall that from Lemma~\ref{lemma-T-bound}, whenever $\|\yy\|_\infty \geq 1/c_\rho(u)$, 
$\mv{T}^{-1}(\yy) \leq \yy [c_{\rho}(u)]^{\frac{\log\yy}{\log \|\yy\|_\infty}} \mv{\lor 1}$.
Therefore for all $u > u_0$ where $u_0$ is sufficiently large,
$\mv{T}^{-1}(\qq((t(u)\pp)) \leq \exp({\log \qq(t(u)\pp) r_u(\pp)}) \mv{\lor 1}$, where 
\[
r_u(\pp) = \left(1+\frac{\log c_\rho(u)}{\log \|\qq(t(u)\pp)\|_\infty}\right)\in (0,1), \text{ since } \|\qq(t(u)\pp)\|_\infty > 1/c_{\rho}(u).
\]
From the monotonicity of $\mv{\Lambda}$, it follows that $\mv{\psi}_u(t(u)\pp) \leq \mv{\Lambda} \left(\e^{\log \qq(t(u)\pp) r_u(\pp)}\right) \lor\mv{\Lambda}({\mv{1}})$.

\noindent \textbf{Step 2 - Derive partial upper bounds:} To this end, for $k\in [d]$, define the set
\[
E_{k,u} = \left\{\pp: \Lambda_k\left(\e^{\log q_k (t(u)\pp) r_u(\pp)}\right) \geq \Lambda_j\left(\e^{\log q_j(t(u)\pp) r_u(\pp)} \right) , \, \forall  j\neq k\right\} \cap \Lev_{1}^+(L_u).
\]
This is the set of all $\pp$, such that the $k$th component of $\mv{\Lambda}(\e^{\log \qq(t(u)\pp) r_u(\pp)})$ is the largest, and therefore achieves the maximum in $\|\mv{\Lambda} \left(\e^{\log \qq(t(u)\pp) r_u(\pp)}\right)\|_\infty$. Therefore, $\cup_k E_{k,u} = \Lev_1^{+}(L_u)$ (since the maximum in the $\|\cdot\|_\infty$ norm is achieved for some $k\in [d]$). For $\pp\in E_{k,u}$, $\|\mv{\psi}_u(t(u)\pp) \|_\infty \leq \Lambda_k\left(\e^{\log q_k (t(u)\pp) r_u(\pp)}\right)\lor\max_{k}\Lambda_k(1)$.
By definition, over $E_{k,u}$, for all $j\in [d]$,
\begin{equation}\label{eqn:Max-RF}
{\Lambda_k\left(\e^{\log q_k (t(u)p_k) r_u(\pp)}\right) }\geq {\Lambda_j\left(\e^{\log q_j(t(u)p_j) r_u(\pp)} \right)  }
\end{equation}

\noindent \textbf{Step 3 - Establish a lower bound for $\e^{\log q_{k}(t(u)p_k)r_u(\pp)}$:} Given $\pp$, let $i(\pp)$ be the index which achieves the maximum in $\qq(t(u)\pp)$ (if there are multiple, select one arbitrarily). Then, we have that $q_{i(\pp)}(t(u)p_{i(\pp)}) \geq \delta_1 u^{1/\rho}$. Recall that $c_\rho(u) =(l/u)^{1/\rho}$. Therefore, 
\[
q_{i(\pp)}(t(u)p_{i(\pp)}) c_\rho^{\frac{\log q_{i(\pp)}(t(u)p_{i(\pp)}) }{\log \|\qq(t(u)\pp)\|_\infty}}(u) = q_{i(\pp)}(t(u)p_{i(\pp)}) c_\rho(u) \geq \delta_1 l^{1/\rho}.
\]
Let $\kappa$ be such that $\kappa \max_j\alpha_j< \min_{j}\alpha_j$. Then, notice that for all $k$
\begin{equation}\label{eqn:kappa}
\frac{\min_j\alpha_j }{\alpha_k}-\kappa>0.
\end{equation}
 Since $\mv{\Lambda}\in \RV(\mv{\alpha})$,  an application of  \cite[Proposition B.1.9 (1)]{deHaanFerreira2010} shows that for all $\kappa>0$, there exists an $x_j$ such that for all $x>x_j$, $\Lambda_k^{-1}(\Lambda_j(x)) \geq x^{\alpha_j/\alpha_k -\kappa}$.
Further since $l\to\infty$ as $u\to\infty$, there exists a $u_{2,k}$, such that for all $u>u_{2,k}$, for $\kappa$ selected as in \eqref{eqn:kappa}, $\delta_1 l^{1/\rho} > \max_{j}x_j$. Therefore, for all $u>u_{2,k}$, $\Lambda^{-1}_k(\Lambda_j(\delta_1 l^{1/\rho})) \geq (\delta_1 l^{1/\rho})^{\alpha_j/\alpha_k-\kappa}$ for all $j\in[d]$. Now, for all $\pp\in E_{k,u}$, for all $u>u_0\lor u_1\lor u_{2,k}$,
\begin{align}\label{eqn:LB-1}
\left(\e^{\log q_k (t(u)\pp) r_u(\pp)}\right) &\geq \Lambda_k^{-1}(\Lambda_{i(\pp)}(q_{i(\pp)} (t(u) p_{i(\pp)})))\text{ since $\pp\in E_{k,u}$}\nonumber\\
&\geq \Lambda_k^{-1}(\Lambda_{i(\pp)}(\delta_1 l^{1/\rho})) \geq (\delta_1 l^{1/\rho})^{\frac{\min_j\alpha_j}{\alpha_k} -\kappa}  \text{ by choice $u>\max\{u_1,u_{2,k}\}$.}
\end{align}
 Notice that since $l\to\infty$ as $u\to\infty$, the RHS above can be made arbitrarily large by appropriate choice of $u$.

\noindent \textbf{Step 4 - Evaluate and combine partial bounds:} Fix $\delta>0$. Observe that due to the monotonicity of $\Lambda_i$, \eqref{eqn:Max-RF} implies that for all $\pp\in E_{k,u}$, for all $j$,
$\Lambda_{j}^{-1} \left(\Lambda_k \left(\e^{\log q_k(t(u)p_k)) r_u(\pp)} \right)\right) \geq \exp({\log q_{j}(t(u) p_{j}) r_u(\pp)})$.
 Next, observe that from an application of  \cite[Proposition B.1.9 (1)]{deHaanFerreira2010}, $\Lambda_j^{-1}(\Lambda_k(x)) \leq x^{\frac{\alpha_k(1+\delta)}{\alpha_j(1-\delta)}}$ whenever $x\geq x_{j,k}$. Now choose $u$ large enough (say $u>u_{3,k}$) so that $(\delta_1 l^{1/\rho})^{\frac{\min_j\alpha_j}{\alpha_k} -\kappa} > \max_j x_{j,k}$. Therefore, for $u>u_0\lor u_{1}\lor u_{2,k} \lor u_{3,k}$, $\inf_{\pp\in E_{k,u}}\exp(\log (q_k(t(u)p_k)r_u(\pp) ) \geq \max_{j}x_{j,k}$. Then, 
\[
\exp\left({\frac{\alpha_k(1+\delta)}{\alpha_j(1-\delta)} \log q_k(t(u)p_k)r_u(\pp) }\right) \geq \Lambda_{j}^{-1} \left(\Lambda_k \left(\e^{\log q_k(t(u)p_k)) r_u(\pp)} \right)\right), \text{  for all $j\in[d]$. }
\]
Thus, for all $j$,
\[
\exp\left({\log q_k (t(u)p_k) r_u(\pp) - \frac{\alpha_{j}(1-\delta)}{\alpha_k(1+\delta) } \log q_j (t(u)p_j) r_u(\pp) }\right)\geq 1.
\]
With $r_u(\pp) >0$, this requires that for all $j$,
\[
\frac{\log q_k (t(u)p_k)}{ \log q_j (t(u)p_j)} \geq  \frac{\alpha_j(1-\delta)}{\alpha_k(1+\delta) }.
\]
Upon selecting the $j$ which achieves the maximum in $\|\qq(t(u)\pp)\|_\infty$, notice that  
\[
\frac{\log q_k (t(u)p_k)}{\log \|\qq(t(u)\pp)\|_\infty} \geq \frac{\min\alpha_j(1-\delta)}{\alpha_k(1+\delta) } \text{ and therefore, since $c_\rho(u) = o(1)$},
\]
\[
q_k(t(u)p_k) [c_\rho(u)]^{\frac{\log q_k(t(u)p_k)}{\log\|\qq(t(u)\pp)\|_\infty}} \leq q_k(t(u)p_k) [c_\rho(u)]^{\frac{\min_{j}\alpha_j(1-\delta)}{\alpha_k(1+\delta)}}  \leq q_k(t(u)p_k)\varepsilon^{1/\alpha_k} \text{ where $\varepsilon$ is arbitrary}.
\]
Therefore,  $\Lambda_k(\e^{q_k(t(u) p_k) r_u(\pp)}) \leq \Lambda_k(q_k(t(u) p_k) \varepsilon^{1/\alpha_k}) \leq \varepsilon t(u)p_k$.
Now, whenever $u>u_k$, for $\pp\in E_{k,u}$, $\|\mv{\psi}_u(t(u)\pp)\|_\infty \leq \varepsilon t(u) p_k \leq \varepsilon t(u)\|\pp\|_\infty$,
uniformly over $\pp\in E_{k,u}$. By symmetry, for $u>\max_k(u_0\lor u_{1}\lor u_{2,k} \lor u_{3,k})$, uniformly over $\pp\in \Lev_1^+(L_u)$,
\begin{equation}\label{eqn:Verify-1}
\|\mv{\psi}_u(t(u)\pp)\|_\infty  \leq \varepsilon t(u)\|\pp\|_\infty    
\end{equation}

\noindent \textbf{Step 5 - Establish non-negativity:} With the bound on $\|\mv{\psi}_u(t(u)\pp)\|_\infty$ established, from \eqref{fY-unit-sphere} and Lemma~\ref{lemma:Jac-Light} respectively, for $\varepsilon>0$ for all large enough $u,$
\[
a_{u}(\pp) \geq t(u) \|\pp\|_\infty(1-\varepsilon) , \quad  \frac{\log t(u)}{t(u)}\leq \varepsilon \quad \textrm{ and } \quad b_{u}(\pp)\geq -\varepsilon t(u).
\]
From Lemma~\ref{lemma:Not-Zero}, $\pp\in \Lev_{1}^+(L_u) \implies \|\pp\|_\infty > \gamma$. Thus, $F_u(\pp)\geq 0$ over $\pp\in \Lev_{1}^{+}(L_u)$ for all large enough $u$. Finally, notice that $\chi_{\Lev_{1}^{+}(L_u)}(\pp) = \infty$, for $\pp\not\in \Lev_{1}^{+}(L_u)$. This completes the proof. \hfill $\Box$

\subsection{Proof of log-efficiency in the presence of heavier tails}
Recall that $\mathcal{L}$ is the collection of indices of components
$(X_1,\ldots,X_d)$ which satisfy the lighter tailed assumption in
Assumption \ref{assump-marginals}. Then
$\mathcal{H} := \{1,\ldots,d\} \setminus \mathcal{L}$ denotes the
heavy-tailed components.

\noindent \textbf{Proof of Theorem~\ref{thm:Tail-asymp-HT}.} Under
Assumption \ref{assume:marginals-HT},
$\bar{\Lambda}_i \in \RV(\alpha_i)$ for $i \in \mathcal{H}.$ Its
respective inverse is, 
\begin{align*}
  \bar{q}_i := \bar{\Lambda}_i^{\leftarrow} = \log q_i \in \RV(1/\alpha_i),
\end{align*}
due to \cite[Proposition B.1.9(9)]{deHaanFerreira2010}. Let
$\bar{\mv{q}}(\yy) := (\bar{q}_1(y_1),\ldots q_d(y_d)).$ Define the
following counterparts to quantities $L_u,\fLD, t(u)$ and
$q_\infty(t)$ used in the proof of Theorem \ref{thm:Tail-asymp}:
\begin{subequations}
  \begin{equation}
    \bar{L}_{u} (\xx)
    := \frac{\log L(\e^{ \bar{\qq}({t}(u)\xx)})}{\log u} \quad
    \text{ and }\quad  \fldh(\xx) := \bar{L}^\ast({\qq}^\ast  \xx^{\aalpha}), 
    \label{ht:defn-Vu-fld}
  \end{equation}
  \begin{equation}
    \text{where }  t(u)
    := \min_{i} \bar{\Lambda}_i(\log u) = \Lambda_{\min}(u),
    \quad  \text{ and } \quad  \bar{q}_\infty(t)
    := \max_{i = 1,\ldots,d}\bar{q}_i(t).
    \label{ht:defn-bartu-barqinfty}
    \end{equation}
\end{subequations}
Since $\bar{q}_\infty^\leftarrow = \min_i \bar{\Lambda}_{i},$ we have
$\bar{q}_\infty(t(u)) = \log u$ (see Lemma \ref{lem:qtequ} and
(\ref{eq:qinfty-eq-Lamdamin})). Letting
$\YYu := t(u)^{-1}\YY = t(u)^{-1}\mv{\Lambda}(\XX),$ we 
have the following equivalence of events,
\begin{align*}
  \left\{ L(\XX) > u\right\} = \left\{ \YYu \in \Lev^+_1(\bar{L}_u)\right\},
\end{align*}
from the definition of $\bar{L}_u$ and injectivity of
$\bar{\mv{q}} = \bar{\mv{\Lambda}}^{\leftarrow}.$

As before, we proceed by showing continuous convergence of $L_u$ to
$\fldh,$ as $u \rightarrow \infty.$ For this purpose, consider
sequences $\{u_n\}_{n \geq 1} \subset \Real_+,$
$\{\xx_n\} \subset \Real^d_{++}$ such that
$u_n \rightarrow \infty, \xx_n\to \xx>\mv{0}$. Since
$\bar{q}_\infty(t(u)) = \log u,$ $\bar{{q}}_i \in\RV(1/\alpha_i)$ for
$i \in \mathcal{H},$ and $\bar{{q}}_i \in\RV(0), \hat{q}_i\ast = 0$
for $i \in \mathcal{L},$
\begin{align}
  \frac{\bar{\qq}(t(u_n)\xx_n)}{\log u_n} =
  \frac{\bar{\qq}(t(u_n)\xx_n)}{\bar{\qq}(t(u_n)\mv{1})}
  \hat{\qq}(t(u_n)) \rightarrow \xx^{\aalpha} \mv{q}^\ast
  \label{q-conv-HT}
\end{align}
from (\ref{defn:qhat-HT}) and \cite[Proposition
B.1.9(4)]{deHaanFerreira2010}, as $n \rightarrow \infty.$
Consequently,
\begin{align*}
  \bar{L}_{u_n}(\xx_n) = \frac{\log L\left(\exp\{\bar{\qq}
  (t(u_n)\xx_n)\}\right)}{\log u_n}
  = \frac{\log L\left( \exp\{\mv{q}^\ast\xx^{\aalpha} \log u_n
  (1 + o(1))\}\right) }{\log u_n},
\end{align*}
uniformly over $\xx$ in compact subsets. Letting
$\mv{e}(n,\xx) := \e^{n\xx}/\Vert \e^{n\xx}\Vert_\infty$ be the unit
vector, we have 
\[
  L(\e^{n\xx_n}) = L(\|\e^{n\xx_n}\|_\infty\mv{e}(n,\xx_n)) =
  \|\e^{n\xx}\|_\infty^\rho L^*(\mv{e}(n,\xx))(1+o(1)) = L^*(\e^{n\xx})
  (1+o(1)),
\]
where the second equality follows from the compact convergence of
$L(\cdot)$ in Assumption \ref{assume:V} and the last equality is due
to the homogeneity $L^*(c\xx) = c^\rho L^\ast(\xx).$ Then from
Assumption~\ref{assume:V-HT},
\begin{align*}
  \bar{L}_{u_n}(\xx_n) = \frac{\log L^\ast\left( \exp\{
  \mv{q}^\ast\xx^{\aalpha}\log u_n\}\right)(1 + o(1)) }
  { \log u_n}
  \rightarrow \bar{L}^\ast \left( \mv{q}^\ast \xx^{\aalpha}\right) =:
  \fldh(\xx). 
\end{align*}

 Then as a
consequence of the continuous convergence $L_u \rightarrow \fldh$
above, we obtain the following from exactly the same reasoning in the
proofs of Lemma \ref{lem:cont-conv-set-incl} and Corollary
\ref{cor:set-inclusions}: given $\varepsilon, M > 0,$ there exists
$u_0$ large enough such that for all $u > u_0,$
\begin{align}\label{eqn:Conv_HT_1}
  \Lev_1^+(\bar{L}_u) \cap B_M \subseteq
  \big[\Xi_{1,M}(\fldh)\big]^{1+\varepsilon},  \text{ and } 
  \inf_{n: u_n > u}\chi_{\Lev_1^+(\bar{L}_{u_n})}(\xx_n)
  \geq \chi_{\Lev_1^+(\fldh)}(\xx),
\end{align}
for any $\xx_n \rightarrow \xx$ and $u_n \rightarrow \infty.$ Thanks
to these set inclusions, the conclusion in
Theorem~\ref{thm:Tail-asymp-HT} follows by repeating the steps in the
proof of Theorem \ref{thm:Tail-asymp} with $L_u,\fLD$ replaced by
$\bar{L}_u,\fldh.$ \hfill$\Box$

To analyse the variance of the IS estimator, we have the following
Lemma.
\begin{lemma}\label{lem:RV-2}
  Suppose that Assumption \ref{assume:marginals-HT} holds, the
  parameter $l$ in (\ref{IS-transf}) is taken to be slowly varying in
  $u,$ and $\rho = 1$ in (\ref{IS-transf}). Then uniformly over
  compact subsets of $\Real^d_{+} \setminus \{\mv{0}\},$
  \begin{align}
    \frac{\mv{\psi}_u\big( t(u) \mv{p}\big)}{t(u)} =
    \pp\mv{1}_{\mathcal{H}}\left(1 - \frac{1}
    {\Vert\mv{q}^\ast\pp^{\aalpha}\Vert_\infty}
    \right)^{\mv{\alpha}}(1+o(1)),
    \text{ as } u\rightarrow \infty,
    \label{T-similarity-property-2}
  \end{align}
  where the vector $\mv{1}_{\mathcal{H}}$ is the indicator vector (for
  the heavy-tailed components) defined as in (\ref{indicate-HT}).
\end{lemma}
\textit{Proof.} With $\rho = 1,$ we have $c_\rho(u) = (l(u)/u).$ For
$\xx\in \R^d_+$,
$\mv{T}(\xx) \leq \tilde{\mv{T}}(\xx) := (1+\xx)
[c_{\rho}(u)]^{-\mv{\kappa}(\xx)},$ component-wise.  Note that
$\tilde{\mv{T}}^{-1} \circ \tilde{\mv{T}}(\xx) = \xx$ for
$\xx \in \Real^d_+$ when we take
$\tilde{\mv{T}}^{-1}(\xx) = \xx[c_\rho(u)]^{\mv{\kappa}(\xx-1)}-1.$
This yields
$\mv{T}^{-1}(\xx) \geq (\xx
[c_\rho(u)]^{\mv{\kappa}(\xx-1)}-\mv{1})^+.$ Combining this with the
upper bound in (\ref{T-bounds}), we arrive at the following: For any
$\pp \in \Real^d_+,\delta > 0$ there exists $u_0$ large enough
such that for all $u > u_0,$
\begin{align*}
  & \left[\exp\left\{ \bar{\qq}(t(u)\pp) r_u(\pp)\right\} - 1\right]^+
    \ \leq \ \mv{T}^{-1}(\qq(t(u)\pp)) \  \leq \ 
    \exp\left\{\bar{\qq}(t(u)\pp)r_u(\pp)\right\}
    \qquad \pp \in B_\delta(\pp),\\
  &\qquad \text{ where } r_u(\pp) := 1+\frac{\log c_\rho(u)}{
    \|\bar{\qq}(t(u)\pp)\|_\infty} = 1 - \frac{1 + o(1)}
    {\Vert \pp^{\aalpha}\mv{q}^\ast\Vert_\infty} \qquad\quad 
    \text{[due to (\ref{q-conv-HT})]}.
\end{align*}
Since $\mv{\psi}_u : = \mv{\Lambda}\circ \mv{T}^{-1}\circ \qq$ and
$\mv{\Lambda}$ is increasing component wise, 
\begin{align}
  \mv{\Lambda}\big(\left[\exp \left\{\bar{\qq}(t(u)\pp)
      r_u(\pp)\right\} - \mv{1}\right]^+\big) \ \leq \
  \mv{\psi}_u(t(u)\pp) \ \leq \
  \mv{\Lambda}\big(\exp\left\{\bar{\qq}(t(u)\pp)r_u(\pp)
  \right\}\big).
  \label{psi-bnd-HT}
\end{align}
As $\Lambda_i \circ \exp = \bar{\Lambda}_i \in \RV(\alpha_i)$ for
$i \in \mathcal{H}$ and $\bar{q}_i := \bar{\Lambda}_i^\leftarrow,$
the term
\begin{align*}
  \Lambda_i\big(\exp \left\{\bar{q}_i(t(u)p_i)r_u(p_i)\right\}\big) =
  r_u(p_i)^{\alpha_i}\bar{\Lambda}_i \circ \bar{q}_i(t(u)p_i)(1 + o(1)) =
  r_u^{\alpha_i}(p_i) t(u)p_i(1+o(1)),  \quad \text{ for }
  i \in \mathcal{H},
\end{align*}
On the other hand when $i \in \mathcal{L},$ we have
$\Lambda_i \in \RV(\alpha_i)$ and $q_i \in \RV(1/\alpha_i).$ In this
case,
\[\Lambda_i\big(\exp \left\{\bar{q}_i(t(u)p_i)r_u(p_i)\right\}\big) =
  \Lambda_i \left( q_i(t(u)\pp)^{r_u(p_i)}\right) =
  O(t(u)^{1-\Vert\pp^{\aalpha}\mv{q}^\ast\Vert_\infty^{-1} + o(1)}).
  \quad \text{ for } i \in \mathcal{L}\]
Due to the above deduction that
$r_u(\pp)= 1-\Vert\pp^{\aalpha}\mv{q}^\ast\Vert_\infty^{-1} (1+o(1)).$
Since the above convergences uniformly over compact subsets,
(\ref{psi-bnd-HT}) results in,
\begin{equation*}
    \lim_{u\to\infty}  t(u)^{-1}\mv{\psi}_{u,i}(t(u)\pp) =
  \begin{cases}
    p_i\left(1-\Vert\pp^{\aalpha}\mv{q}^\ast\Vert_\infty^{-1}
    \right)^{\alpha_i}
    \quad& \text{ for } i \in \mathcal{H},\\
    0 & \text{ for } i \in \mathcal{L}.
  \end{cases} \qquad\qquad \Box
\end{equation*}

\noindent \textbf{Proof of Theorem~\ref{thm:Var-Red-HT}:} Following
the reasoning in the proof of Theorem~\ref{thm:Var-Red-HT}, notice that the second moment of the IS estimator may be written as 
$M_{2,u} = t^{2d}(u)\Expc\left[\exp\left\{-t(u) \bar{F}_u(\bar{\YY}_u)\right\}  \right]$, where 
\begin{equation}\label{eqn:HT-VIL_1}
\bar{F}_u(\pp) = a_u(\pp)+b_u(\pp) + {\chi}_{\Lev_{1}^+(\bar{L}_u)}(\pp).    
\end{equation}
Here, $a_u(\pp)$ and $b_{u}(\pp)$ are as defined in Lemma~\ref{lem:M2-rewrite}. Following Lemma~\ref{lem:RV-2} and the proof of Lemma~\ref{lemma: RV-1}, we obtain
$a_{u}(\pp) \geq I(\pp) - I\left(\pp \mv{1}_{\mathcal{H}}\left(1-1/\|\qq^*\pp^{\aalpha}\| \right)^{\mv{\alpha}} \right) +o(1)$.
uniformly over compact subsets of $\R_{++}^d$. 
We have from Assumption~\ref{assume:marginals-HT} that for
$i \in \mathcal{H}$, $\Lambda\circ \exp \in \RV(\alpha_i),$ for some
$\alpha_i\geq 1$. Therefore, $\Lambda_i \in \RV(0)$ whenever
$i\not\in\mathcal{L}$. Now, since $\lambda_i(\cdot)$ are monotone,
\cite[Proposition B.1.9 (7)]{deHaanFerreira2010} implies that
$\lambda_i\in \RV(\gamma_i-1)$. Here, $\gamma_i=0$ if
$i\in\mathcal{H}$.  Following the steps in the proof of
Lemma~\ref{lemma: RV-1}, bound the product (for large enough $u$) in
(b) as
\begin{equation}\label{eqn:Jac_Heavy_1}
\prod_{i=1}^{d}\frac{ \lambda_i(q_i(t(u)p_i))}{\lambda_i(\mv{T}^{-1}_i(\qq(t(u)\pp)))}J(\mv{T}^{-1}_i (\qq(t(u)\pp))) \leq \exp\left( \log (u/l) \left[\sum_{i=1 }^d \gamma_i \mv{\kappa}_i(\qq(t(u)\pp)) + o(1)\right]\right)
\end{equation}
Whenever $\mathcal{L} \neq [d]$ and $i\in \mathcal{L}$, it is easy to
see that $\bar{q}_i^* = 0$. For all such $i$,
$\mv{\kappa}_i(\qq(t(u)\pp)) =o(1)$. Further, for all
$i\in \mathcal{H}$, $\gamma_i = 0$. Finally, observe that with
$\bar{\Lambda}_i\in\RV(\alpha_i)$ for $i\in\mathcal{H}$, we have
$\log u= O(t(u))$. Now, \eqref{eqn:Jac_Heavy_1} suggests
$b_u(\pp) \geq {-t(u)\varepsilon}$ for all large enough $u$. Noting
the convergences in \eqref{eqn:Conv_HT_1}, and repeating the arguments
from the proof of Theorem~\ref{thm:Var-Red}, replacing $F_u(\pp)$
there by $\bar{F}_u(\pp)$,
\[
\limsup_{u\to\infty} [\Lambda_{\min}(u)]^{-1}\log M_{2,u} \leq -\inf_{\pp\in \Lev_1^+(\fldh)} 2I(\pp) + I\left(\pp\mathbf{1}_{\mathcal{H}} \left(1 - 1/{\|\qq^*\cdot \pp^{\aalpha}\|_\infty }\right)^{\mv{\alpha}}\right) + 2\varepsilon.  
\] 
Due to the homogeneity of $I(\cdot)$ (see
Lemma~\ref{lem:properties-I}(b)), it can be seen that the above
infimum occurs at the boundary,
$\| \mv{q}^*\cdot \mv{p}^{\aalpha}\|_{\infty}= 1$, and therefore,
$\limsup_{u\to\infty}\frac{1}{\Lambda_{\min}(u)} \log M_{2,u} \leq
-2I^{*} +2\varepsilon$.  $\qquad \Box$

\noindent \textbf{Verification of Remark \ref{rem:q-suff}:} 
For any $f,g\in \RV(p)$ that are eventually strictly increasing and
satisfying $\lim_{x
  \to\infty}f(x)/g(x) = c\in(0,\infty),$ we first
show that
$\lim_{x \to \infty}g^{\leftarrow}(x)/f^{\leftarrow}(x) = c^{1/p}.$
For this purpose, observe
$ \lim_{t \to \infty}{g^{\leftarrow}(tx)}/{g^{\leftarrow}(t)} =
x^{1/p}$ uniformly over $x \in (c/2,2c)$ as $t\to\infty.$ Setting
$t=g(f^{\leftarrow}(x)),$ we have $t\to\infty$ and
$f (f^{\leftarrow}(x))/g(f^{\leftarrow}(x)) \to c$ as $x \to \infty.$
Therefore,
\begin{align*}
  g^{\leftarrow}(x) =g^{\leftarrow}\left(g(f^{\leftarrow}(x)) \cdot
  \frac{f (f^{\leftarrow}(x))}{g(f^{\leftarrow}(x))}  \right)
  \sim
  c^{1/p} f^{\leftarrow}(x). 
\end{align*}
This verifies the claim
$g^{\leftarrow}(x)/f^{\leftarrow}(x) \to c^{1/p}$.  To see
(\ref{ratio-Lambda}) as a consequence, fix any $i \in \{1,\ldots,d\}$
such that $q_i^\ast$ exists and $q_i^\ast(x) >
0.$ 
Setting $f = q_i$ and $g(\cdot) = \Vert \mv{q}(\cdot) \Vert_\infty,$
we have $f^{\leftarrow} = \Lambda_i, g^{\leftarrow} = \Lambda_{\min}$
(see (\ref{eq:qinfty-eq-Lamdamin})).  
Since $\Lambda_{\min} \in \RV(\alpha_*),$  $q_i^* = (\lim_{x\to\infty} \Lambda_{\min}(x) / \Lambda_i(x))^{1/\alpha_*}$. If $q_i^\ast = 0,$ the conclusion is immediate from the differing
rates of growths of the numerator and the denominator. Finally to
verify the sufficient condition on the derivative, consider any
sequence $\{x_n\} \subset \Real$ increasing to infinity. Since
$ \vert r_i^\prime(x) \vert \leq Mx^{-(1+\varepsilon)}$ for suitable constants
$M,\varepsilon > 0,$
\[ \vert r_{i}(x_{m+n}) - r_i(x_{m}) \vert \leq \int_{x_{m}}^{x_{m+n}}
  \vert r_i^\prime(x)\vert dx \leq \varepsilon^{-1} M
  x_{m}^{-\varepsilon},\] for all sufficiently large $m.$ %
Therefore the sequence $\{r_i(x_n): n \geq 1\}$ is Cauchy and is
convergent.
{
\section{Proof Theorem~\ref{thm:Var-Red} with $\mv{\kappa} = \mv{\kappa}_2$}\label{sec:verify-alt}
To avoid complicating notation, we omit the dependence of $\mv{T}_{u}^{(2)}$ on $u$ and in the subsequent proof, simply use $\mv{T}$ instead.
Observe that $\mv{T}(\xx) = (T_{1,c}(x_1),\ldots, T_{1,c}(x_d))$ for a 1-1 onto function $T_{1,c}$ (defined imminently) and therefore itself 1-1 and onto. Define the function $\mv{\psi}_{u,1} = \mv{\Lambda} \circ \mv{T}^{-1} \circ \qq$. To proceed, we check that the conditions required in the proof of Theorem~\ref{thm:Var-Red} hold. As in that case, first, we bound $\mv{T}^{-1}$ from above. Observe that $\mv{T}(\xx) = (T_{1,c}(x_1),\ldots,T_{1,c}(x_d))$, where 
\[
T_{1,c}(y) = y\left[u/l\right]^{\frac{\log(1+y)}{\log l}}  \geq
\begin{cases}
   & y^{\log u/\log l} \quad \quad \text{ whenever } y\geq 1\\
   &y \quad \quad \ \quad \quad \quad  \text{ otherwise}
\end{cases}
\]
Therefore
\begin{equation}\label{eqn:Alt-T-bound}
    T_{1,c}^{-1}(y) \leq \begin{cases}
  &y^{\log l/\log u} \quad \quad \text{ whenever } y\geq 1\\
  & y \quad \quad \ \quad \quad \quad \text{ otherwise.}
\end{cases}
\end{equation}
 Denote $\mv{\psi}_{u,1} = \mv{\Lambda} \circ \mv{T}^{-1}\circ \qq$. The bound on $\mv{T}^{-1}$ established, we now proceed to check the technical conditions required for log-efficiency. This amounts to verifying \eqref{eqn:Verify-1} with $\mv{\psi}_u$ replaced by $\mv{\psi}_{u,1}$, and bounding the Jacobian determinant of $\mv{T}$. With these bounds established, the rest of the proof is similar to case when $\mv{\kappa}=\mv{\kappa}_1$.

\noindent \textbf{Step 1: Verifying condition \eqref{eqn:Verify-1}: }  Observe that given \eqref{eqn:Alt-T-bound}, 
\[
\Lambda_{i}(\mv{T}^{-1}_{i}(\qq(t(u)\pp))) \leq \Lambda_i\left(q_i(t(u)p_i)^{\log l/\log u}\right) \lor \Lambda_i(1)
\]
With $l$ being slowly varying in $u$,  $\|\mv{\psi}_{u,1}(t(u)\pp)\|_\infty \leq \|\mv{\Lambda}[(\qq(t(u)\pp))]^{o(1)}\|_{\infty}$ and therefore, one has the bound $\|\mv{\psi}_{u,1}(t(u)\pp)\|_\infty \leq \varepsilon t(u)\|\pp\|_\infty$ for all $u$ large enough. This further implies that uniformly over compact subsets of $\R^d_+$, $\|\mv{\psi}_{u,1}(t(u)\pp)\|_\infty = o(t(u))$, and establishes an equivalent of Lemma~\ref{lem:property-T}, but with $\mv{\kappa}^{(1)}$ replaced by $\mv{\kappa}^{(2)}$ in the definition of $\mv{T}$.

\noindent \textbf{Step 2: Bounding the Jacobian determinant:} We establish that $\log J(\qq(t(u)\pp)) = o(t(u))$. To this end, recall that for $\xx\in \R^d_+$ (the case where $\xx$ is allowed to be in $\R^d$ can be handled similarly),
\begin{align*}
    J(\xx) &= \left(\frac{u}{l}\right)^{\mv{1}^\intercal\mv{\kappa}^{(2)}(\xx)}\prod_{i=1}^d\left(1+ \frac{\log(u/l)}{\log l} \frac{x_i}{1+x_i}\right) =  \prod_{i=1}^d\left[(u/l)^{\frac{\log(1+x_i)}{\log l}} + \frac{T_{1,c}(x_i)}{1+x_i}\frac{\log(u/l)}{\log l}\right].
\end{align*}
Therefore, 
\[
\log J(\mv{T}^{-1}(\qq(t(u)\pp))) = \sum_{i=1}^d \log\left[(u/l)^{\frac{\log(1+T^{-1}_{1,c}(q_i(t(u)p_i)))}{\log l}} + \frac{q_i(t(u)p_i)}{(1+T^{-1}_{1,c}(q_i(t(u)p_i))))}\frac{\log(u/l)}{\log l}\right]
\]
The rate of increase of the right hand side above is determined by the larger of 
\begin{equation}\label{eqn:inter-1}
    \sum_{i=1}^d\log\left[(u/l)^{\frac{\log(1+T^{-1}_{1,c}(q_i(t(u)p_i)))}{\log l}} \right] \text{ and } \log\left[ \frac{q_i(t(u)p_i)}{(1+T^{-1}_{1,c}(q_i(t(u)p_i))))}\frac{\log(u/l)}{\log l}\right].
\end{equation}
Since $l$ is slowly varying in $u$,  the bound in \eqref{eqn:Alt-T-bound} yields
\[
\limsup_{u\to\infty}\frac{1}{t(u)}\sum_{i=1}^d\log\left[(u/l)^{\frac{\log(1+T^{-1}_{1,c}(q_i(t(u)p_i)))}{\log l}} \right] \leq \limsup_{u\to\infty} \frac{\log(u/l)}{t(u) \log l} 
\|[\qq(t(u)\pp)]^{o(1)}\|_1 =0
\]
Similarly, the second term of \eqref{eqn:inter-1} may be bounded as 
\[
\limsup_{u\to\infty}\frac{1}{t(u)}\log\left[ \frac{q_i(t(u)p_i)}{(1+T^{-1}_{1,c}(q_i(t(u)p_i))))}\frac{\log(u/l)}{\log l}\right]  = 0, \text{ and consequently,}
\]
\[
\log J_1(\mv{T}^{-1}(\qq(t(u)\pp))) = o(t(u)) \quad \text{ (see for e.g. \cite[Lemma 1.2.14]{Dembo}). }
\]
\noindent \textbf{Step 3: Combine the bounds:} Note the expression for the second moment of the IS estimator with $\mv{T}^{(2)}$ instead of $\mv{T}^{(1)}$ is given by replacing $\mv{\psi}_u$ by $\mv{\psi}_{u,1}$ and substituting the appropriate Jacobian as given in Table~\ref{tab:Jacobians}. From Steps 1 and 2, the consequences of Lemmas~\ref{lemma: RV-1}- \ref{lem:non-negative-Fu} continue to hold, and the rest of the proof follows from the proof of Theorem~\ref{thm:Var-Red}. \hfill $\Box$

}

{

\section{Verifying Assumption~\ref{assume:V}(b)} \label{sec:Verify_Num}
Recall that a random vector $\YY$ is said to be multivariate regularly varying with index $\rho$ if for any set $A$ not containing the origin,
\begin{equation}\label{eqn:MRV_Heavy_Tails}
t P\left[\frac{\YY}{t^{\rho}}\in A\right] \to \mu(A), 
\end{equation}
for some non-zero radon measure $\mu$. 
An equivalent formulation \citep[Theorem 6.1]{resnick2007heavy} is that for some probability measures $\mu_{1/\rho}$ on the line and $M$ on the sphere,
\begin{equation}\label{eqn:Radial_Angular}
    t P\left[\left(\|\YY\| , \frac{\YY}{\|\YY\|}\right)  \in A\right] \to (\mu_{1/\rho}\times M) (A),
\end{equation}
 for any set $A $ not containing the origin.
Here $\mu_{1/\rho}$ is taken to be
 $\mu_\rho(c,\infty] = c^{-1/\rho},$  without loss of generality,    for any $c>0.$ To verify Assumption~\ref{assume:V}(b), we develop the following characterisation of Assumption~\ref{assume:V}(b) based on multivariate regular variation: Let $\mathcal{S}^{d-1}_{+} := \{\xx \in \R^d_+: \Vert \xx \Vert = 1\}$ denote the intersection of unit sphere and the positive orthant. 
\begin{proposition}\label{prop:MRV_Loss}
  Let $R$ be a random variable satisfying $P(R \leq r) = 1-1/r,\  r\geq 1$ and $\mv{\Theta}$ be uniformly distributed on $\mathcal{S}^{d-1}_{+},$  independently of $R$. Then  $L(\cdot)$ satisfies Assumption~\ref{assume:V}(b) with $\rho \in (0,\infty)$  if and only if the random vector $L(R\mv{\Theta}) \cdot \mv{\Theta}$ is a multivariate regularly varying random vector with index $\rho.$
\end{proposition}
Verifying regular variation of a random vector is well-studied and continues to be a topic of active research \citep[see][and references therein]{einmahl2021testing}.  As a consequence of Proposition \ref{prop:MRV_Loss}, one can obtain independent samples of $L(R\mv{\Theta})\mv{\Theta}$ and use the statistical test developed in  \citep{einmahl2021testing}  to verify if $L(\cdot)$ satisfies Assumption \ref{assume:V}b merely from  oracle queries to the evaluations of $L(\cdot).$ The rest of this section sketches the proof of Proposition~\ref{prop:MRV_Loss}. 
\begin{remark}
   Suppose that $L(\cdot)$ has an approximation $\tilde{L}(\cdot)$ which satisfies either (i) $|L(n\ttheta) - \tilde{L}(n\ttheta)| = o(L(n\ttheta)),$ or more generally, (ii) $\tilde{L}(n\ttheta) = c(\ttheta) L(n\ttheta)(1+o(1)),$ uniformly over $\ttheta$ on $\mathcal{S}^{d-1}_+$ and for some continuous $c(\cdot).$   Then it is sufficient to verify Assumption~\ref{assume:V}(b) for the approximate functional $\tilde{L}$. Such a verification is useful if, for example, $\tilde{L}(\cdot)$ is either available in explicit form (or) if its evaluations are computationally less expensive than those of $L(\cdot).$
\end{remark}
\noindent \textbf{Proof for Proposition~\ref{prop:MRV_Loss}: } \textbf{a) }First suppose that $L(\cdot)$ satisfies Assumption~\ref{assume:V}(b). Recall that as a consequence of \citep[Theorem 6.1, Lemma 6.2]{resnick2007heavy}, to verify the multivariate regularly varying property, it is sufficient to verify that for all $\xx>0$,
\[
n P\left(\frac{L(R\mv{\Theta}) \cdot \mv{\Theta} }{n^{\rho}} \in [\mv{0},\xx]^{c}\right) \to \mu[\mv{0},\xx]^{c},
\]
for some measure $\mu.$ Let $R_n = R/n$. Using the independence of $(R,\mv{\Theta})$, the probability above equals
\begin{equation}\label{eqn:Necessity}
k_d\int_{\mv{\ttheta} \in S_{d-1}^{\infty} } P\left(\frac{L(nR_n \ttheta)}{n^{\rho}}\ttheta \in [\mv{0},\xx]^c\right) d\ttheta = k_d\int_{\mv{\ttheta} \in S_{d-1}^{\infty}} P(R_n \in S_{n,\ttheta,\xx})d\ttheta,     
\end{equation}
where $k_d$ equals the 1/(volume of $\mathcal{S}_+^{d-1})$ and the set $S_{n,\ttheta,\xx}  =\{r: \ttheta L(nr\ttheta)/n^{\rho}  \in [0,\xx]^{c}\}$. Notice that owing to the convergence of $ L(nr\ttheta)/n^{\rho}$ to $r^{\rho}L^*(\ttheta)$, the set $S_{n,\ttheta,\xx}$  converges to $S_{\ttheta,\xx}^* = \{r: r^{\rho}L^*(\ttheta) \cdot \ttheta \in [\mv{0},\xx]^{c} \}$ in the Painlev\'{e}-Kuratowski sense \citep[Section 4.B]{rockafellar2009variational}. Consequently, $n P(R_n \in S_{n,\ttheta,\xx}) \to \mu_0(S^*_{\ttheta,\xx})$ where the density of $\mu_0(dr) = r^{-2}dr$ on the line (note that the measure $\mu_0$ is a Radon measure, and not a probability measure). Plugging this back into  the integral in \eqref{eqn:Necessity}, it can be seen that whenever $L(\cdot)$ satisfies Assumption~\ref{assume:V}(b), $L(R\mv{\Theta}) \cdot \mv{\Theta}$ is multivariate regularly varying (with limiting measure $\mu[\mv{0},\xx]^{c} = k_d\int_{\ttheta\in S_{d-1}^{\infty}}\mu_0(S_{\ttheta,\xx})d\ttheta $ ).

\textbf{b) }  Now, suppose that $L(R\mv{\Theta})\cdot \mv{\Theta}$ is multivariate regularly varying with index $\rho$. Let $L_{\ttheta,n}(r) = L(nr\ttheta)/n^{\rho}$. Then, as a consequence of representation \eqref{eqn:Radial_Angular}, with $\mv{\Theta}$ and $R_n$ being independent, for any $c>0$ (the limits below being taken in an appropriate sense), 
\begin{equation}\label{eqn:Suff-1}
nP(R_n\in \Lev_{c}^+(L_{\ttheta,n})) = n P\left(\frac{L(nR_n\mv{\Theta})}{n^{\rho}} \in [c,\infty) ,  \ \mv{\Theta} \in d\ttheta\right) \to c^{-1/\rho} m(\ttheta)d\ttheta,
\end{equation}
for some $m:\mathcal{S}_+^{d-1} \rightarrow \R_+.$ Define the function $L^*(\xx) = (\|\xx\|/m(\xx/\|\xx\|))^\rho$ and let $L^*_{\ttheta}(r) = L^*(r\ttheta)$. Observe then that $\{r: L^*_{\ttheta}(r) \geq c\} = [c^{-1/\rho}m(\ttheta),\infty)$. Now, \eqref{eqn:Suff-1} implies that for every $\ttheta$ and $c>0$
 \[
n P(R_n \in \Lev_c^+(L_{n,\ttheta})) \to c^{1/\rho}m(\ttheta) d\ttheta \implies  \textrm{ess-inf}\ \Lev_c^+(L_{n,\ttheta}) \to \textrm{ess-inf} \ \Lev_c^+(L^*_{\ttheta}).
\]
Since the above holds for every $c>0$, the level sets themselves must converge, that is for every $c>0$, $\Lev_c^+(L_{n,\ttheta}) \to \Lev_{c}^+ (L_{\ttheta}^*)$. Further, observe that  
\[
\Lev_{c}^+(L_n) = \bigcup_{\ttheta\in S^{d-1}_\infty}\{(r,\ttheta): r\in \Lev_c^+(L_{n,\ttheta})\} \quad \textrm{and} \quad \Lev_{c}^+(L_{\ttheta}^*) =  \bigcup_{\ttheta\in S^{d-1}_\infty} \{(r,\ttheta):  r\in \Lev_c(L^*)\}.
\]
Now, use the uniformity of convergence in \eqref{eqn:Suff-1} over $\ttheta$ (refer to \cite[Theorem 6.4]{resnick2007heavy}) to observe that  for all $c>0$, $\Lev_{c}^+(L_n) \to  \Lev_c^+(L^*)$. An application of \cite[Proposition 7.7]{rockafellar2009variational} implies that since $L_n$ converges epigraphically  to $L^*$, it converges uniformly on compact subsets. Finally, from \cite[Theorem 7.14]{rockafellar2009variational} uniform convergence implies the continuous convergence in Assumption~\ref{assume:V}(b). \hfill$\Box$

\section{Application to CVaR Minimisation}
\label{sec:opt}
Suppose that  $\ell(\XX,\ttheta)$ denotes the loss associated with a decision choice $\ttheta$ under a random vector $\XX.$ Let  $v_\beta({\ttheta})$ denote the $(1-\beta)$-th quantile of $\ell(\XX,\ttheta).$ Then its CVaR at the tail-level $\beta \in (0,1)$ is 
\[
C_{\beta}(\ttheta) := E\left[\ell(\XX,\ttheta) \ \vert\ \ell(\XX,\ttheta) \geq v_\beta({\ttheta}) \right], 
\]
Minimizing CVaR $C_\beta(\ttheta)$ over a compact set $\Theta$ enjoys the following variational representation  \citep[see][]{rockafellar2000optimization}, 
\begin{equation}\label{eqn:CVaR_OpT}
    \inf_{u\in \R,\ttheta\in \Theta} \left[u+ \beta^{-1} E\left(\ell(\XX,\ttheta)-u\right)^+  \right] =\inf_{u\in\R,\ttheta\in \Theta} f(u,\ttheta),
\end{equation}
If $\ell(\XX, \cdot)$ is convex, then $f(\cdot)$ is convex. As a result, CVaR minimization has become the most prominent vehicle in a number of applications for arriving at decisions with low tail risks. To perform the minimization without expending exorbitant computational effort in problems with small $\beta,$ one can consider minimising the following IS weighted Sample Average Approximation:
\begin{equation}\label{eqn:FIS}
\hat{f}_{is,n}(u,\ttheta) = \left[u+ \frac{1}{n\beta}\sum_{i=1}^n (\ell(\ZZ_{u,i},\ttheta) -u)^+\mathcal{L}_{i}\right],  
\end{equation}
where, as before, $\mathcal{L}$ is defined to be the likelihood between $\XX$ and $\ZZ_u = \mv{T}_u(\XX) := \XX[u/l]^{\mv{\kappa}(\XX)}$:
\begin{align}
\mathcal{L} = \frac{f_{\XX}\left(\XX [u/l]^{\mv{\kappa}(\XX)}\right)}{f_{\XX}(\XX)} J(\XX) \text{  where } J(\xx) = \text{Det}[\partial \mv{T}_u(\xx)/\partial \xx] \text{ is as defined in \eqref{eqn:Jac_T_general}}.
\end{align}
Notice that since $\mathcal{L}$ depends on $u$ through the factor $[u/l]^{\mv{\kappa}(\XX)}.$ Therefore even if $f(\cdot)$ as defined in \eqref{eqn:CVaR_OpT} were convex in $(u,\ttheta)$, the IS weighted objective \eqref{eqn:FIS}  need not be convex. In turn, such lack of convexity due to the introduction of likelihood ratio and absence of efficient change of measure prescriptions which hold uniformly well simultaneosuly over feasible $(u,\mv{\theta})$ have been the primary bottlenecks in using IS, in general, for optimization. 

Since the loss structure is explicitly known in optimization settings, the growth rate $\rho$ in Assumption \ref{assume:V}b is typically readily known and the IS transformation $\mv{T}$ in \eqref{IS-transf} employed with $\mv{\kappa} = \mv{\kappa}_1$ is particularly well-suited to overcome the above difficulties. In particular, by (i) changing variable as in  $[u/l] = s$ where $s \in [1,\infty)$, and (ii) using  $\mv{\kappa} = \mv{\kappa}^{(1)}$ in \eqref{defn:kappa},  the resulting  IS weighted objective \eqref{eqn:CVaR-I.S}  remains convex in $(u,\mv{\theta})$ when $\ell(\XX,\mv{\theta})$ is convex in $\mv{\theta}.$ Except for the knowledge of $\rho,$ the model agnostic nature of the change makes the change of measure induced by $\mv{T}$ to possess low variance at every feasible $(u,{\ttheta})$. Parameterizing the stretch factor $s$ as $s = h\log\log(1/\beta),$ the selection of hyperparameter $h$ at any given feasible $(u,\ttheta)$ can be accomplished by minimizing the second moment as in Step 2 of Algorithm \ref{algo:RA_IS}. 
Algorithm~\ref{algo:CVaR-FULL.} below incorporates this selection  in every stage of Retrospective Approximation of the CVaR objective. We refer the readers to a follow-up work \cite{deo2022combining} for further implementation details. 
\begin{algorithm}[h]
{
 \caption{IS based CVaR Optimisation}\label{algo:CVaR-FULL.}
  \textbf{Input:} Tail probability level $\beta$,  samples 
  $\boldsymbol{X}_1,\ldots,$ from $f_{\XX}(\cdot)$, initializations $u_0,\ttheta_0,h_0$.\\
 
 \noindent\textbf{ For $k\geq 1$, do} the following steps until stopping criterion is met\\
 
  \textbf{Step 1}( IS-Weighted CVaR optimisation): With a sample size of $m_{k}$ and error tolerance $\varepsilon_{k}$:
  
  \noindent \textbf{a) Transform the samples:} For each sample
  $i=1,\ldots,m_k,$ compute the transformation,
  \begin{align*}
    \ZZ_{i} = \mv{T}(\XX_i) := \XX_i [s]^{\mv{\kappa}^{(1)}(\XX_i)},
   \end{align*}
   with $s = h_{k-1}\log\log(1/\beta).$\\
\noindent \textbf{b) Minimize the IS weighted CVaR objective}
   \begin{equation}
     \label{eqn:CVaR-I.S}
        \inf_{u,\ttheta}\left\{u + \frac{1}{m_k\beta}
      \sum_{i=1}^{m_k} \big[\ell({\mv{Z}}_{i},\ttheta) - u \big]^+
      \mathcal{L}_{h,i}\right\}, 
    \end{equation}
      with the initial solution iterate set to $(u_{k-1},\ttheta_{k-1})$ and $h=h_{k-1}.$ Let  $(u_k,\ttheta_k)$ denote the optimiser returned after 
       reaching an error tolerance $\varepsilon_{k} = m_k^{-1/2}.$ \\
\noindent \textbf{2 Update the cross validation parameter: }
With the initial solution iterate set to $h_{k-1},$ minimize the sample second moment estimate as in
\begin{align*}
\inf_{h > 0} \ \frac{1}{m_k}
      \sum_{i=1}^{m_k} \left\{ \mathbf{I}\left(\ell(\mv{T}_h(\XX_i),\ttheta_k) \geq u_{k}\right)
      \mathcal{L}_{h,i} \right\}^2, 
\end{align*}
where $\mv{T}_h(\XX_i) = \XX_i [h\log \log(1/\beta)]^{\mv{\kappa}^{(1)}(\XX_i)}.$ Let $h_k$ denote the solution obtained by solving until reaching error-tolerance $\varepsilon_k = m_k^{-1/2}.$
}
\end{algorithm}

\noindent \textbf{Numerical results: }Consider the constrained minimum CVaR  portfolio optimisation problem: Here $\ell(\xx,\ttheta) = \ttheta^\intercal\xx $ and the set $\Theta = \{\ttheta:\ttheta^{\intercal}\mv{1} = 1, \mv{\theta}^\intercal \mv{\mu} \geq r\}$, where $\mv{\mu}$ denotes the expected returns. 
The marginals of the loss realizations $\XX$ are taken to have the c.d.f.s $F_i(x) = P(X_i \leq x) = 1-\e^{-x^{\alpha_i}}$ where $\alpha_i =0.5 \ \forall i$.  Dependence is modelled through a Gaussian copula whose covariance matrix $\mv{R}$ is designed to capture a realistic degree of correlation among various asset returns. In order to compare the effort required to obtain a desired out of sample accuracy, we give (i) the number of samples required by each method to obtain $1\%$ relative regret (relative error between the optimal CVaR and CVaR computed at the solution proposed by the respective algorithm) and (ii) the out-of-sample regret when the number of loss evaluations used by each algorithm is restricted to $2500$. For the former case, with $\beta=0.037$, for IS, this is $\approx 600$, while SAA requires $\approx 5500$ samples. This difference is even more pronounced when $\beta= 0.003$, where SAA requires roughly $28000$ samples, while IS only requires $1175$. For the latter, at $\beta = 0.037$, IS gives a regret of $2\%$ while SAA gives a regret of $5\%$. We refer the reader to \cite{deo2022combining} for more details on the numerical experiments and the explicit specifications for $\{m_k,\varepsilon_k: k\geq 1\}$.
}

\noindent \textbf{Code availability:}  Python implementations for importance sampling using self-structuring transformations are available at \href{https://github.com/ananddeo161093/BBIS_Source_Codes}{\url{https://github.com/ananddeo161093/BBIS_Source_Codes}}.

\end{document}